%% file: 00000_main_2.tex
\documentclass[preprint,12pt,authoryear]{elsarticle}
\usepackage[top=2cm, bottom=2cm, left=2cm, right=2cm]{geometry}
\usepackage{amsmath,amssymb}
\usepackage{amsthm} 

\usepackage{booktabs}
\usepackage{multirow}
\usepackage{subcaption}

\usepackage{enumitem}
\usepackage{xparse}
\usepackage{xcolor}
\usepackage{pifont}
\usepackage{hyperref}
\hypersetup{colorlinks, linkcolor={red},citecolor={green}, urlcolor={blue}}
\usepackage[nameinlink,noabbrev]{cleveref}
\usepackage{array}

\usepackage{tikz}
\usetikzlibrary{positioning,arrows.meta,decorations.pathreplacing,calc,shapes.geometric,fit,backgrounds, calc, intersections}

\definecolor{rkhs}{RGB}{31,119,180}        
\definecolor{kernel}{RGB}{255,127,14}      
\definecolor{sample}{RGB}{44,160,44}       
\definecolor{layer1}{RGB}{214,39,40}       
\definecolor{layer2}{RGB}{148,103,189}     
\definecolor{output}{RGB}{140,86,75}       
\definecolor{cluster}{RGB}{227,119,194}    
\definecolor{axisgray}{RGB}{100,100,100}   

\definecolor{nBlue}{RGB}{0, 114, 178}       
\definecolor{nOrange}{RGB}{213, 94, 0}      
\definecolor{nGreen}{RGB}{0, 158, 115}      
\definecolor{nPurple}{RGB}{204, 121, 167}   
\definecolor{nYellow}{RGB}{240, 228, 66}    
\definecolor{nGray}{RGB}{80, 80, 80}        
\definecolor{nLightGray}{RGB}{230, 230, 230}

\setcitestyle{numbers,open={[},close={]}}

\definecolor{red}{rgb}{0.7,0.15,0.15}
\definecolor{green}{rgb}{0,0.5,0}
\definecolor{blue}{rgb}{0,0,0.7}
\definecolor{darkcyan}{rgb}{0.0, 0.55, 0.55}
\definecolor{MidnightBlue}{RGB}{25,25,112}
\definecolor{MidnightBlueComplementingGreen}{RGB}{25,112,25}
\definecolor{MidnightBlueComplementingPurple}{RGB}{112,25,112}
\definecolor{MidnightBlueComplementingRed}{RGB}{112,25,69}
\definecolor{WowColor}{rgb}{.75,0,.75}
\definecolor{MildlyAlarming}{rgb}{0.85,0.25,0.1}
\definecolor{SubtleColor}{rgb}{0,0,.50}
\definecolor{antiquefuchsia}{rgb}{0.57, 0.36, 0.51}
\definecolor{fashionfuchsia}{rgb}{0.96, 0.0, 0.63}
\definecolor{jade}{rgb}{0.0, 0.66, 0.42}
\definecolor{caribbeangreen}{rgb}{0.0, 0.8, 0.6}
\definecolor{aquamarine}{rgb}{0.5, 0.8, 0.85}
\definecolor{darkmidnightblue}{rgb}{0.0, 0.2, 0.4}
\definecolor{attentioncolor}{RGB}{152,90,81}
\definecolor{burgred}{RGB}{40,3,22}
\definecolor{AKGreen}{RGB}{17,123,92}
\definecolor{egyptianblue}{rgb}{0.06, 0.2, 0.65}
\definecolor{Turquoise}{RGB}{64,224,208}
\definecolor{darkjade}{RGB}{0,122,84}
\definecolor{Window1}{RGB}{92,150,31}
\definecolor{Window1dark}{RGB}{41,67,13}
\definecolor{Window2}{RGB}{255,168,28}
\definecolor{Window2dark}{RGB}{114,75,12}
\definecolor{Window3}{RGB}{255,96,33}
\definecolor{Window3dark}{RGB}{97,36,12}
\definecolor{InputColor}{RGB}{20,255,177}
\definecolor{InputColorlight}{RGB}{222,237,229}
 \definecolor{richblack}{rgb}{0.0, 0.25, 0.25}

\hypersetup{colorlinks, linkcolor={red},citecolor={green}, urlcolor={blue}}
			
\makeatletter \@addtoreset{equation}{section}

\usepackage[ruled,lined]{algorithm2e}

\SetKwComment{Comment}{//~}{}

\SetCommentSty{mycommfont}

\theoremstyle{definition}
\newtheorem{definition}{Definition}

\theoremstyle{plain}
\newtheorem{theorem}{Theorem}
\newtheorem{proposition}{Proposition}
\newtheorem{lemma}{Lemma}

\theoremstyle{remark}

\newcommand{\eqdef}{\ensuremath{\stackrel{\mbox{\upshape\tiny def.}}{=}}}

\journal{Physica D}
\begin{document}

\begin{frontmatter}



\title{Neural Operators Can Discover Functional Clusters} 


\affiliation[1]{organization={McMaster University and the Vector Institute},
            addressline={Main St.},
            city={Hamilton},
            postcode={L8S 4L8},
            state={Ontario},
            country={Canada}}

\affiliation[2]{organization={Rice University},
            addressline={Main St.},
            city={Houston},
            postcode={77005},
            state={Texas},
            country={United States}}

\author[1]{Yicen Li}
\author[2]{J. Antonio Lara Benitez}
\author[1]{Ruiyang Hong}
\author[1]{Anastasis Kratsios}
\author[1]{Paul David McNicholas}
\author[2]{Maarten V. de Hoop}

\begin{abstract}
Operator learning is reshaping scientific computing by amortizing inference across infinite families of problems. While neural operators (NOs) are increasingly well understood for regression, far less is known for classification and its unsupervised analogue: clustering.
We prove that sample-based neural operators can learn any finite collection of classes in an infinite-dimensional reproducing kernel Hilbert space, even when the classes are neither convex nor connected, under mild kernel sampling assumptions. Our universal clustering theorem shows that any $K$ closed classes can be approximated to arbitrary precision by NO-parameterized classes in the upper Kuratowski topology on closed sets, a notion that can be interpreted as disallowing false-positive misclassifications.

Building on this, we develop an NO-powered clustering pipeline for functional data and apply it to unlabeled families of ordinary differential equation (ODE) trajectories. Discretized trajectories are lifted by a fixed pre-trained encoder into a continuous feature map and mapped to soft assignments by a lightweight trainable head. Experiments on diverse synthetic ODE benchmarks show that the resulting practical SNO recovers latent dynamical structure in regimes where classical methods fail, providing evidence consistent with our universal clustering theory.
\end{abstract}




\end{frontmatter}


\input{Intro_2}
\input{00_Theory_2}

\input{Our_method}
\input{Data_benchmarking}
 \input{Experiments}

\input{Conclusion}


\appendix
\input{Proofs}
\input{Appendix}
\newpage
\bibliography{references}

\end{document}

%% file: Intro_2.tex
\section{Introduction}
\label{s:Intro}
Neural operators have become an essential tool in scientific machine learning by learning maps between function spaces and thereby amortizing inference across entire families of parametric problems (e.g., solution operators for PDEs) \cite{lu2021deeponet,li2020fourier,kovachki2023neuraloperator}.  While this paradigm is now theoretically and empirically mature for regression-style operator approximation, considerably less is understood for \emph{set-valued} tasks such as classification—and, a fortiori, its unsupervised counterpart, clustering—where one must recover decision regions rather than pointwise function values.  This gap is especially salient in functional data analysis, where clustering is classically formulated directly in function space and surveyed extensively in the functional data clustering literature \cite{ramsay2005fda,jacques2014fdcsurvey,zhang2022reviewclusteringmethodsfunctional}.  In parallel, modern deep clustering has shown that representation learning and clustering can be coupled effectively at scale, from deep embedding objectives to clustering-based self-supervision \cite{xie2016dec,caron2018deepcluster,caron2020swav,ji2019iic}, often leveraging strong pre-trained or self-supervised encoders \cite{radford2021learning,caron2021dino,chen2020simclr}; yet such pipelines typically do not provide guarantees that the induced decision regions converge to the \emph{true} cluster sets in an infinite-dimensional setting.

In this paper we provide such guarantees for sampling-based neural operators.  We prove that sample-based NOs can approximate \emph{any} finite collection of closed classes in an infinite-dimensional RKHS—even when the classes are neither convex nor connected—under mild kernel sampling assumptions.  Our universal clustering theorem is explicitly \emph{set-valued}: it controls approximation of closed decision regions in the upper-Kuratowski topology on closed sets, a mode of convergence naturally aligned with excluding false-positive misclassifications.  Building on this, we develop a neural-operator-driven clustering pipeline for functional data and instantiate it on unlabeled families of ODE trajectories: discretized paths are lifted by a fixed pre-trained encoder into a continuous feature map and mapped to soft assignments by a lightweight trainable head.  Across diverse synthetic ODE benchmarks, the resulting practical SNO recovers latent dynamical structure in regimes where classical methods fail, in qualitative agreement with the universal clustering theory developed here.

\subsection{Problem Formulation}
In this work, we address the clustering problem directly in the infinite-dimensional setting. We consider a dataset of functions $\{g_n\}_{n=1}^N \subset \mathcal{H}$, where $\mathcal{H}$ is a reproducing kernel Hilbert space (RKHS) and $N$ is a positive integer. We fix the number of clusters $K \in \mathbb{N}_+$ (with $K \le N$) and formulate the infinite-dimensional $K$-means problem. This problem seeks a set of \textit{cluster centers} $\{f_k\}_{k=1}^K$ within a domain $\mathcal{K} \subseteq \mathcal{H}$ that minimize the objective
\begin{equation}
\label{eq:k_means__problem}
    \mathcal{J}(\{f_k\}_{k=1}^K) 
    \eqdef 
    \sum_{n=1}^N \, \min_{k=1,\dots,K} \, \|g_n - f_k\|_{\mathcal{H}}^2.
\end{equation}

These optimal centers induce a partition of the domain into $K$ disjoint (true) \textit{clusters} $C_1, \dots, C_K \subset \mathcal{H}$. Specifically, the $k$-th cluster is defined as the closed set of functions nearest to $f_k$ among all other centers
\begin{equation}
	\label{eq:the_clusters__boundaryfree}
	C_k
	\eqdef
	\big\{
	h\in \mathcal{K}
	:\,
	\|h-f_k\| \le \min_{j \neq k}\, \|h-f_j\|
	\big\}.
\end{equation}

\begin{figure}[hbp!]
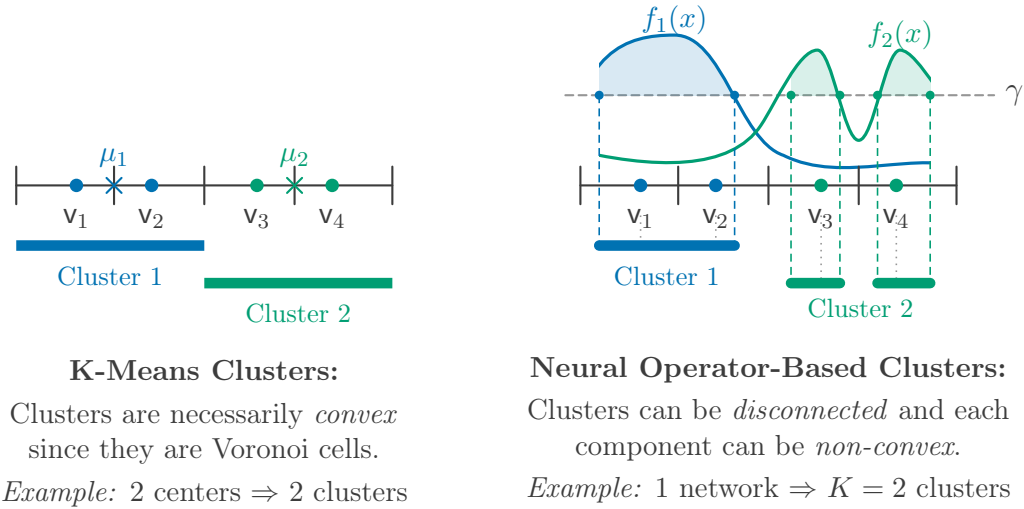

    \centering
    \include{images/Cluster_image_tikz}
\caption{
(\textit{Left: K-Means}) Standard finite-dimensional $K$-means relies on mapping data to discrete cluster centers $\mu_1,\mu_2$ in a reduced feature space, often discarding continuous dynamics.
\\
(\textit{Right: Neural Operator}) In contrast, our proposed method leverages a single neural operator to generate continuous cluster indicator functions.
The blue and green curves represent the two output components $\hat f_1$ and $\hat f_2$ of the neural operator (one for each output dimension), and the corresponding cluster regions are defined via thresholding at $\gamma$.
}
    \label{fig:Kmeans_ours}
\end{figure}

While the formulation above is standard in finite dimensions, it poses two distinct difficulties in the infinite-dimensional setting. First, the optimal cluster centers $f_k$ are functions that cannot be exactly represented or implemented in practice. This introduces a fundamental gap between theory and application: purely analytical guarantees do not yield constructive algorithms, and obtaining a numerical solution often faces the curse of dimensionality, making the optimization process computationally prohibitive.

Second, it is an open question whether approximating these centers using \textit{infinite-dimensional neural networks}—specifically, standard neural operators (NOs)—guarantees that the \textit{learned} clusters converge to the \textit{true} clusters $C_k$.

To resolve this, we shift from explicitly solving for the intractable functional centers $\{f_k\}$ to directly parameterizing the induced partition. To this end, we employ a single neural operator $\hat{f}: \mathcal{H} \to \mathbb{R}^K$ to parameterize the clustering assignment, mapping each input trajectory to a soft cluster probability. We then define the approximate (learned) clusters $\hat{C}_1, \dots, \hat{C}_K$ via a soft-assignment mechanism:
\begin{equation}
\label{eq:clearned_cluster}
    \hat{C}_k \eqdef \left\{ h \in \mathcal{H} :\, \sigma(\hat{f}_k(h)) 
        \ge
    \gamma \right\},
\end{equation}
where $\hat{f}_k(h)$ is the $k$-th component of the NO output, $\sigma: \mathbb{R} \to (0,1)$ is a monotone increasing continuous map (e.g., a sigmoid), and $0 < \gamma < 1$ is a decision threshold.
Figure~\ref{fig:Kmeans_ours} provides a conceptual comparison with standard finite-dimensional $K$-means.

We formalize this set-valued viewpoint in Sections~\ref{s:Theory}--\ref{s:UniversalClustering}. In particular, our main theoretical result (\Cref{thrm:MainKuratowski}) shows that, given a sufficiently expressive neural operator, the induced clusters $\{\hat{C}_k\}_{k=1}^K$ converge to the true clusters $\{C_k\}_{k=1}^K$ in the \emph{upper Kuratowski topology} on relevant subsets of the RKHS $\mathcal{H}$.

We validate our approach experimentally by using a sampling-based neural operator model to cluster diverse trajectories generated by Ordinary Differential Equations (ODEs).
Specifically, we bridge the gap between continuous function spaces and discrete inputs by treating the discretization process as a sampling projection from the infinite-dimensional function space to a finite-dimensional coordinate representation.
Leveraging recent theoretical advances in sampling recovery \cite{Krieg_2026}, we propose that a finite collection of point evaluations provides an information-preserving representation of the underlying continuous dynamics within a bounded observation window. In practice, these sampled values are arranged on a regular grid, yielding a two-dimensional array that can be viewed as an image.
These inputs are then processed by the visual encoder of the pre-trained model (e.g., CLIP \cite{radford2021learning}), which is used as a fixed nonlinear feature map from the sampling coordinate space into a high-dimensional latent space.
Finally, a lightweight MLP is trained on these lifted features via contrastive learning to produce soft cluster assignments. This pipeline constitutes a concrete instantiation, where sampling provides the discretization operator and the MLP head implements the learnable decision mapping in finite dimensions.
Across multiple settings, we show that our method captures latent dynamical properties more effectively than baselines relying solely on visual or finite-dimensional projections.

\subsection{Outline}
The remainder of this paper is organized as follows. Sections~\ref{s:Theory} and~\ref{s:UniversalClustering} introduce the theoretical background of universal clustering and our convergence proofs are provided in~\ref{s:Proofs}. Section~\ref{our method} details the proposed functional data clustering framework, including the construction of positive pairs and the training objective. Section~\ref{sec:synthetic} describes the generation of synthetic ODE datasets used for validation. Finally, Section~\ref{sec:experiments} reports experimental results, comparing our NO-based approach against K-means baselines and analyzing the robustness of the method.

%% file: images/Cluster_image_tikz.tex
\definecolor{nBlue}{RGB}{0, 114, 178}
\definecolor{nGreen}{RGB}{0, 158, 115}
\definecolor{nGray}{RGB}{60, 60, 60}
\definecolor{axisgray}{RGB}{150, 150, 150}

    \begin{tikzpicture}[
        font=\sffamily, 
        every node/.style={color=nGray},
        thick_axis/.style={thick, color=nGray, line cap=round},
        point/.style={circle, fill=nGray, inner sep=1.8pt},
        point_blue/.style={circle, fill=nBlue, inner sep=1.8pt},
        point_green/.style={circle, fill=nGreen, inner sep=1.8pt},
        curve1/.style={line width=1.2pt, color=nBlue},
        curve2/.style={line width=1.2pt, color=nGreen},
        bar1/.style={line width=4pt, color=nBlue},
        bar2/.style={line width=4pt, color=nGreen},
        dropline/.style={densely dashed, color=axisgray, line width=0.6pt},
        pointdrop/.style={dotted, color=axisgray, line width=0.6pt}
    ]

        \def\barYone{-0.8}
        \def\barYtwo{-1.3}

        \begin{scope}[shift={(0,0)}]
            \node[font=\Large\bfseries, text=black] at (-0.5, 2.5) {};

            \draw[thick_axis] (0,0) -- (5,0);
            
            \draw[thick_axis] (0, 0.25) -- (0, -0.25);
            \draw[thick_axis] (1.3, 0.25) -- (1.3, -0.25);
            \draw[thick_axis] (2.5, 0.25) -- (2.5, -0.25); 
            \draw[thick_axis] (3.7, 0.25) -- (3.7, -0.25);
            \draw[thick_axis] (5, 0.25) -- (5, -0.25);

            \draw[thick, nBlue, line cap=round] (1.2, -0.1) -- (1.4, 0.1);
            \draw[thick, nBlue, line cap=round] (1.2, 0.1) -- (1.4, -0.1);
            \node[above, text=nBlue, font=\small] at (1.3, 0.1) {$\mu_1$};
            
            \draw[thick, nGreen, line cap=round] (3.6, -0.1) -- (3.8, 0.1);
            \draw[thick, nGreen, line cap=round] (3.6, 0.1) -- (3.8, -0.1);
            \node[above, text=nGreen, font=\small] at (3.7, 0.1) {$\mu_2$};

            \node[point_blue, label={[yshift=-2pt]below:{$\mathsf{v}_1$}}] at (0.8, 0) {};
            \node[point_blue, label={[yshift=-2pt]below:{$\mathsf{v}_2$}}] at (1.8, 0) {};
            \node[point_green, label={[yshift=-2pt]below:{$\mathsf{v}_3$}}] at (3.2, 0) {};
            \node[point_green, label={[yshift=-2pt]below:{$\mathsf{v}_4$}}] at (4.2, 0) {};

            \draw[bar1, line cap=butt] (0.0, \barYone) -- (2.5, \barYone) 
                node[midway, below=2pt, text=nBlue, font=\footnotesize] {Cluster 1};
                
            \draw[bar2, line cap=butt] (2.5, \barYtwo) -- (5.0, \barYtwo) 
                node[midway, below=2pt, text=nGreen, font=\footnotesize] {Cluster 2};

            \node[align=center, yshift=-3.5cm, font=\small] at (2.5, 0) {
            \textbf{K-Means Clusters:}\\[0.5ex]
            Clusters are necessarily \textit{convex} \\since they are Voronoi cells.\\[0.5ex]
                \textit{Example:} $2$ centers $\Rightarrow$ $2$ clusters\\ 
                
            };
        \end{scope}

        \begin{scope}[shift={(7.5,0)}]
            \node[font=\Large\bfseries, text=black] at (-0.5, 2.5) {};

            \def\thresh{1.2}
            
            \def\bstart{0.25}  
            \def\bend{2.05}    
            \def\gstartA{2.80} 
            \def\gendA{3.45}   
            \def\gstartB{3.95} 
            \def\gendB{4.65}   
            
            \draw[thick_axis, axisgray, densely dashed] 
                (-0.2, \thresh) -- (5.5, \thresh) 
                node[right, text=nGray, font=\normalsize] {$\gamma$};

            \begin{scope}
                \clip (\bstart, \thresh) rectangle (\bend, 2.5);
                \fill[nBlue, opacity=0.15] (-0.2, 0.3)
                    .. controls (0.0, 1.8) and (0.6, 2.0) .. (1.25, 2.0)
                    .. controls (1.9, 2.0) and (2.1, 0.6) .. (2.7, 0.4)
                    .. controls (3.5, 0.0) and (4.5, 0.5) .. (5.2, 0.2)
                    -- (5.2, 0) -- (-0.2, 0) -- cycle;
            \end{scope}
            
            \begin{scope}
                \clip (\gstartA, \thresh) rectangle (\gendA, 2.5);
                \fill[nGreen, opacity=0.15] (-0.2, 0.5)
                    .. controls (1.0, 0.2) and (2.0, 0.2) .. (2.4, 0.8)
                    .. controls (2.6, 1.1) and (2.9, 1.8) .. (3.15, 1.8)
                    .. controls (3.35, 1.8) and (3.5, 0.6) .. (3.7, 0.6)
                    .. controls (3.9, 0.6) and (4.05, 1.8) .. (4.25, 1.8)
                    .. controls (4.5, 1.8) and (4.9, 0.9) .. (5.2, 0.8)
                    -- (5.2, 0) -- (-0.2, 0) -- cycle;
            \end{scope}

            \begin{scope}
                \clip (\gstartB, \thresh) rectangle (\gendB, 2.5);
                \fill[nGreen, opacity=0.15] (-0.2, 0.5)
                    .. controls (1.0, 0.2) and (2.0, 0.2) .. (2.4, 0.8)
                    .. controls (2.6, 1.1) and (2.9, 1.8) .. (3.15, 1.8)
                    .. controls (3.35, 1.8) and (3.5, 0.6) .. (3.7, 0.6)
                    .. controls (3.9, 0.6) and (4.05, 1.8) .. (4.25, 1.8)
                    .. controls (4.5, 1.8) and (4.9, 0.9) .. (5.2, 0.8)
                    -- (5.2, 0) -- (-0.2, 0) -- cycle;
            \end{scope}

            \begin{scope}
                \clip (\bstart, -0.5) rectangle (\gendB, 3.5);
                
                \draw[curve1] (-0.2, 0.3)
                    .. controls (0.0, 1.8) and (0.6, 2.0) .. (1.25, 2.0) 
                    .. controls (1.9, 2.0) and (2.1, 0.6) .. (2.7, 0.4)
                    .. controls (3.5, 0.0) and (4.5, 0.5) .. (5.2, 0.2);

                \draw[curve2] (-0.2, 0.5)
                    .. controls (1.0, 0.2) and (2.0, 0.2) .. (2.4, 0.8)
                    .. controls (2.6, 1.1) and (2.9, 1.8) .. (3.15, 1.8) 
                    .. controls (3.35, 1.8) and (3.5, 0.6) .. (3.7, 0.6) 
                    .. controls (3.9, 0.6) and (4.05, 1.8) .. (4.25, 1.8) 
                    .. controls (4.5, 1.8) and (4.9, 0.9) .. (5.2, 0.8);
            \end{scope}

            \node[text=nBlue, font=\small] at (1.25, 2.2) {$f_1(x)$};
            \node[text=nGreen, font=\small] at (4.25, 2.0) {$f_2(x)$};
                
            \node[circle, fill=nBlue, inner sep=1.2pt] at (\bstart, \thresh) {};
            \node[circle, fill=nBlue, inner sep=1.2pt] at (\bend, \thresh) {};
            \node[circle, fill=nGreen, inner sep=1.2pt] at (\gstartA, \thresh) {};
            \node[circle, fill=nGreen, inner sep=1.2pt] at (\gendA, \thresh) {};
            \node[circle, fill=nGreen, inner sep=1.2pt] at (\gstartB, \thresh) {};
            \node[circle, fill=nGreen, inner sep=1.2pt] at (\gendB, \thresh) {};

            \draw[thick_axis] (0,0) -- (5,0);
            
            \draw[thick_axis] (0, 0.25) -- (0, -0.25);
            \draw[thick_axis] (1.3, 0.25) -- (1.3, -0.25);
            \draw[thick_axis] (2.5, 0.25) -- (2.5, -0.25);
            \draw[thick_axis] (3.7, 0.25) -- (3.7, -0.25);
            \draw[thick_axis] (5, 0.25) -- (5, -0.25);

            \node[point_blue, label={[yshift=-2pt]below:{$\mathsf{v}_1$}}] at (0.8, 0) {};
            \node[point_blue, label={[yshift=-2pt]below:{$\mathsf{v}_2$}}] at (1.8, 0) {};
            \node[point_green, label={[yshift=-2pt]below:{$\mathsf{v}_3$}}] at (3.2, 0) {};
            \node[point_green, label={[yshift=-2pt]below:{$\mathsf{v}_4$}}] at (4.2, 0) {};

            \draw[dropline, nBlue] (\bstart, \thresh) -- (\bstart, \barYone);
            \draw[dropline, nBlue] (\bend, \thresh) -- (\bend, \barYone);
            
            \draw[dropline, nGreen] (\gstartA, \thresh) -- (\gstartA, \barYtwo);
            \draw[dropline, nGreen] (\gendA, \thresh) -- (\gendA, \barYtwo);
            \draw[dropline, nGreen] (\gstartB, \thresh) -- (\gstartB, \barYtwo);
            \draw[dropline, nGreen] (\gendB, \thresh) -- (\gendB, \barYtwo);

            \draw[pointdrop] (0.8, -0.4) -- (0.8, \barYone);
            \draw[pointdrop] (1.8, -0.4) -- (1.8, \barYone);
            \draw[pointdrop] (3.2, -0.4) -- (3.2, \barYtwo);
            \draw[pointdrop] (4.2, -0.4) -- (4.2, \barYtwo);

            \draw[bar1, line cap=round] (\bstart, \barYone) -- (\bend, \barYone) 
                node[midway, below=2pt, text=nBlue, font=\footnotesize] {Cluster 1};
                
            \draw[bar2, line cap=round] (\gstartA, \barYtwo) -- (\gendA, \barYtwo);
            \draw[bar2, line cap=round] (\gstartB, \barYtwo) -- (\gendB, \barYtwo);
            \node[below=2pt, text=nGreen, font=\footnotesize] at (3.725, \barYtwo) {Cluster 2};

            \node[align=center, yshift=-3.5cm, font=\small] at (2.5, 0) {
            \textbf{Neural Operator-Based Clusters:}\\[0.5ex]
            Clusters can be \textit{disconnected} and each \\
            component can be \textit{non-convex}.
            \\[0.5ex]
                \textit{Example:} $1$ network $\Rightarrow$ $K=2$ clusters\\[0.5ex]
                
            };
        \end{scope}

    \end{tikzpicture}

%% file: 00_Theory_2.tex
\section{Learning Cluster Assignments in Hilbert Spaces}
\label{s:Theory}

In this section, we formally establish the mathematical foundations for learning cluster assignments in infinite-dimensional Hilbert spaces. We begin by formalizing the structure of the $K$-means clustering problem in function spaces and characterizing the geometry of the resulting partitions. Next, we introduce upper Kuratowski convergence as a rigorous framework to quantify the approximation of these cluster sets, ensuring topological consistency beyond simple norm-based error bounds. Finally, we define Sampling-Based Neural Operators (SNO) to provide the constructive framework for our main result on universal clustering.


\subsection{Structure of the $K$-Means Clustering Problem}

Let $K\in \mathbb{N}_+$ and let $f_1,\dots,f_K$ be distinct elements of the Hilbert space $\mathcal{H}$, arising, for example, as cluster centers from a $K$-means problem of the form~\eqref{eq:k_means__problem}--\eqref{eq:the_clusters__boundaried}. Then, the following holds for the induced clusters in $\mathcal{K}$.

\begin{proposition}[Representation of Pure Clusters]
	\label{prop:classifier_representation__prelimVersion}
	Let $\mathcal{K}$ be a non-empty subset of a Hilbert space $\mathcal{H}$,
	let $K\in \mathbb{N}_+$, and let $f_1,\dots,f_K\in \mathcal{K}$ be distinct.
	For every $k=1,\dots,K$ define the (cluster) set
	\begin{equation}
		\label{eq:the_clusters__boundaried}
		C_k
		\eqdef
		\Big\{
			f\in \mathcal{K}
			:\,
			\|f-f_k\| \le \min_{i=1,\dots,K;\,i\neq k}\, \|f-f_i\|
		\Big\}
		.
	\end{equation}
	Then, for every $k=1,\dots,K$, the set $C_k$ is a non-empty closed subset of $\mathcal{K}$.
\end{proposition}

Consequently, the distinct centers $f_1,\dots,f_K$ induce a family of $K$ closed clusters in $\mathcal{K}$.

This raises a fundamental question: \textit{how do we quantify the approximation of closed sets in an infinite-dimensional space?} The answer lies in set-valued analysis.


\subsection{Approximating Cluster Sets via Upper Kuratowski Convergence }
Unlike classical operator learning, where the primary objective is minimizing the norm error
of a point estimator, clustering is intrinsically a set-valued problem. Our goal shifts from approximating individual functions to recovering subsets of the Hilbert space. In infinite-dimensional settings, standard metrics like the Hausdorff distance are often ill-defined or overly restrictive for non-compact sets~\cite{rockafellar1998variational, beer1993topologies}. Consequently, we require a specialized notion of convergence that captures the geometry of these clusters without relying on uniform boundedness. Thus, we introduce theoretic tools tailored specifically to this challenge as follows.


The \textit{upper Kuratowski convergence} provides a conservative notion of approximation tailored to prevent \emph{Type~I (false positive)} errors~\cite{rockafellar1998variational}. 
In approximation problems, false positives correspond to spurious limit points or infeasible behaviors that persist despite not belonging to the target object. 
Upper Kuratowski convergence rules this out by ensuring that any limit point of a sequence drawn from the approximating sets is contained within the target set.
In other words, if our approximation sequence suggests that points are converging to a location, that location must truly belong to the target set.
In this sense, it enforces \emph{safety} rather than \emph{completeness}: approximation error may lead to loss (Type~II error~\cite{rockafellar1998variational}), but never to the creation of invalid limit behavior, as shown in Figure~\ref{fig:blob}.

\begin{figure}
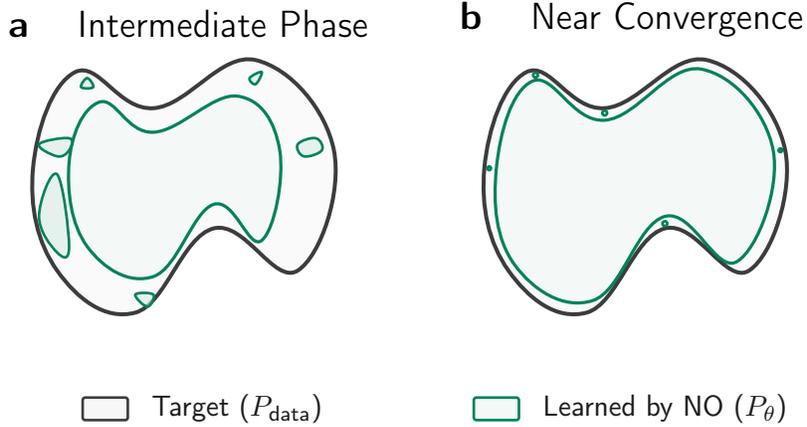

    \centering
    \include{images/blob_tiKZ}
    \caption{Geometric intuition of set-valued clustering via upper Kuratowski convergence. Unlike point-wise estimation, the learned cluster support $P_{\theta}$ (green) approximates the ground truth $P_{\text{data}}$ (black) strictly from the interior. Whether in the intermediate phase (a) or near convergence (b), the approximation is constrained to stay within the target boundaries, thereby maximizing cluster purity. This ensures no false positives, meaning the model prevents infeasible behaviors lying outside the target set.}
    \label{fig:blob}
\end{figure}

\medskip\medskip

\noindent

Crucially, this guarantee is achieved without imposing compactness or uniform boundedness assumptions, 
in contrast to classical notions such as uniform or Hausdorff convergence,  
since the latter is often ill-defined or infinite for unbounded sets in Hilbert spaces ~\cite{rockafellar1998variational, beer1993topologies}. 
Upper Kuratowski convergence therefore remains meaningful in large, non-compact, or infinite-dimensional settings, where stability against false positives is more important than exact two-sided convergence.

The most transparent formulation of the upper Kuratowski 
convergence can be expressed as the convergence of the epigraphs of the \textit{convex-analytic} indicator functions of the sets $C_n$ to that of $C$. We recall that the convex-analytic indicator function of a subset $A\subseteq \mathcal{H}$, denoted by $\delta_A:\mathcal{H}\to [0,\infty]$, yields value $0$ on any $x\in A$ and $\infty$ otherwise,
\[
\delta_A(x) \eqdef 
\begin{cases} 
0 & \text{if } x \in A, \\
\infty & \text{if } x \notin A.
\end{cases}
\]
We highlight that this is the opposite direction that the usual probabilistic indicator function which yields the ``larger value'', namely $1$ if $x\in A$, and returns the ``smaller value'' of $0$ otherwise. 

Following Lemma~\ref{lem:setvaluedconvergenceequivalences} in ~\ref{s:Proofs}, the sequence $C_n$ upper Kuratowski converges to $C$ provided that, for every sequence $f_n \to f$ in $\mathcal{H}$, the following inequality holds:
\[
    \delta_{C}(f) \le \liminf_{n \to \infty} \delta_{C_n}(f_n).
\]

\subsection{Sampling-Based Neural Operators in RKHS}
Let $\mathcal{H}$ be a reproducing kernel Hilbert space (RKHS) on a non-empty set $\mathcal{X}$ with reproducing kernel $\kappa:\mathcal{X}^2\to \mathbb{R}$~\cite{aronszajn1950theory}. 
A countable sequence $\{x_{\lambda}\}_{\lambda\in \Lambda} \subset \mathcal{X}$, indexed by $\Lambda \subseteq \mathbb{N}$, is called \textit{sampling} if there exists a constant $C \ge 1$ such that, for all $f \in \mathcal{H}$: 
\begin{equation}
	\label{eq:sampling_def}
	\frac{\|f\|}{C}
	\le
	\|f\|_{\Lambda}
	\eqdef
	\sqrt{\sum_{\lambda\in \Lambda}\, |f(x_{\lambda})|^2}
	\le
	C\,
	\|f\|
	.
\end{equation}
The sequence is said to be \textit{interpolating} if for every sequence $v_{\cdot} \eqdef (v_{\lambda})_{\lambda \in \Lambda}$ satisfying $\|v_{\cdot}\|_2<\infty$, there exists a function $f\in \mathcal{H}$ such that
\begin{equation}
	\label{eq:interpolating_def}
	f(x_{\lambda})
	=
	v_{\lambda}
	\qquad
	\mbox{ for every } \lambda\in \Lambda
	.
\end{equation}
If the sequence $\{x_{\lambda}\}_{\lambda\in \Lambda}$ is both sampling and interpolating, we say that it is \textit{complete interpolating sampling} (CIS).  We henceforth assume that $\mathcal{H}$ admits a CIS.

The CIS property established above provides a rigorous mechanism to represent functions in $\mathcal{H}$ via discrete samples. We leverage this structure to define the \textit{Sampling-Based Neural Operator} (SNO).
The SNO maps functions from the Hilbert space into a finite-dimensional feature space using the sampling points $\{x_\lambda\}$, which are then processed by a standard neural network architecture.
\begin{figure}[ht!]
    \centering
    \input{images/Architecture_TiKz}
    \caption{Architecture of the Sampling-Based Neural Operator (SNO). \textbf{(a)} The discretization process samples the function $f \in \mathcal{H}$ via inner products with reproducing kernels $\kappa(\cdot, x_s)$ centered at points $x_s$ to form the input vector $v_0 \in \mathbb{R}^S$. \textbf{(b)} The deep neural network maps the discretized input $v_0$ through multiple layers of affine maps $A_l$ and nonlinearities (ReLU) to produce output logits, defining the clusters $C_k$.
    \\
    Note that, only the first layers is a non-local kernel-based operator acting on the function space, and the remaining layers are lightweight standard deep learning models; in our analysis an MLP.
    }
    \label{fig:SamplingNO}
\end{figure}

\begin{definition}[Sampling-Based Neural Operators]
	\label{defn:SNO}
	Let $K\in \mathbb{N}_+$ and fix an RKHS $\mathcal{H}$ over a non-empty set $\mathcal{X}$ with reproducing kernel $\kappa$.
	A map $\hat{f}:\mathcal{H}\to \mathbb{R}^K$ is called a \textit{sampling-based neural operator (SNO)} if there exists some $S,L\in \mathbb{N}_+$, points $x_1,\dots,x_S\in \mathcal{X}$, and for each $l=0,\dots,L$ affine maps
	$A_l:\mathbb{R}^{d_l}\to \mathbb{R}^{d_{l+1}}$ where $d_0,\dots,d_{L+1}\in \mathbb{N}_+$, $S=d_0$, and $K=d_{L+1}$, such that for every $f\in \mathcal{H}$,
	\begin{equation}
		\label{eq:NO_maps}
		\hat{f}(f)
		=
		A_L
		\circ
		(\operatorname{ReLU}\bullet A_{L-1})
		\circ
		\dots
		\circ
		(\operatorname{ReLU}\bullet A_1)
		\Big(
		\big(
		\langle f, \kappa(\cdot,x_s) \rangle_{\mathcal{H}}
		\big)_{s=1}^S
		\Big).
	\end{equation}
	Here $\bullet$ denotes componentwise composition and $\operatorname{ReLU}(\cdot) \eqdef \max\{\cdot,0\}$.
	The set of all SNOs is denoted by $\mathcal{F}_{SNO}$.
\end{definition}

To obtain discrete cluster assignments from the operators in $\mathcal{F}_{SNO}$, we apply a thresholding mechanism.
Fix a continuous, monotonically increasing surjection $\sigma: \mathbb{R} \to (0, 1)$. In this context, $\sigma$ serves as an activation function that maps the unconstrained raw model outputs in $\mathbb{R}$ to normalized soft cluster assignments in $(0, 1)$. Additionally, fix a hard clustering threshold $0 < \gamma < 1$.

Given this data, every continuous map $\hat{f}: \mathcal{H} \to \mathbb{R}^K$ induces a family of $K$-tuples of subsets $(C^{\hat{f}}_k)_{k=1}^K$ of $\mathcal{H}$, defined for each $k=1,\dots,K$ by
\begin{equation}
\label{eq:induced_clusters}
C_k^{\hat{f}}
\eqdef
\hat{f}_k^{-1}\big[[\sigma^{-1}(\gamma),\infty)\big]
=
\Big\{
    f\in \mathcal{H}
    :\,
    \sigma\big((\hat{f}(f))_k\big)\ge \gamma
\Big\}
,
\end{equation}
where $\hat{f}_k(f)\eqdef(\hat{f}(f))_k$ denotes the $k^{th}$ coordinate map.
Since $[\gamma,1]$ is closed in $[0,1]$ and $\sigma\circ \hat{f}_k:\mathcal{H}\to (0,1)$ is continuous, it follows that each induced cluster $C_k^{\hat{f}}\subseteq \mathcal{H}$ is closed (in the topology of $\mathcal{H}$).
Consequently, the family $\mathcal{F}_{SNO}$ induces a family of $K$-tuples of \textit{closed} subsets of $\mathcal{H}$
\begin{equation}
\label{eq:SNO_subsets}
        \mathcal{CF}_{SNO}
    \eqdef
        \Big\{
            (C_k^{\hat{f}})_{k=1}^K
            :\,
                f \in \mathcal{F}_{SNO}
        \Big\}
.
\end{equation}
See Fig.\ref{fig:SamplingNO} for a depiction of SNO. 
Unlike DeepONet \cite{lu2021deeponet}, which is designed for operator regression via dual branch-trunk structures, our SNO formulation adapts the framework specifically for universal clustering. While sharing the sampling-based input discretization, Definition~\ref{defn:SNO} explicitly constrains the codomain to $\mathbb{R}^K$. This replaces the function-valued output of operator regression with a finite-dimensional descriptor, eliminating the need for the trunk network.


\section{Universal Clustering}
\label{s:UniversalClustering}

Our main result shows that, under very mild conditions on our RKHS, we may learn how to cluster within the decision region, i.e.\ away from the decision boundary where clustering becomes ill-posed.  Importantly, we quantify convergence in a natural topology for closed sets, namely the upper-Kuratowski topology, (and not in Kuratowski convergence, which is best behaved on compact sets).

\begin{theorem}[Universal Clustering]
	\label{thrm:MainKuratowski}
Let $K\in \mathbb{N}_+$, and let $\mathcal{K}$ be a non-empty, locally-compact, subset of an RKHS $\mathcal{H}$ over a set $\mathcal{X}$ which admits a CIS sequence.
Fix an increasing surjective ``logit feature'' map $\sigma\in C(\mathbb{R},(0,1))$, and a cluster threshold $0<\gamma<1$.
	For any \textit{distinct} (true) cluster centers $f_1,\dots,f_K\in \mathcal{K}$, define $C_k$ by~\eqref{eq:the_clusters__boundaried}. Then there exists a sequence of sampling-based NOs $\{\hat{f}_n\}_{n=1}^{\infty}\subseteq \mathcal{F}_{SNO}$ such that: 
    
	\begin{equation}
	\label{eq:thmMain_limit1}
		\lim\limits_{n\uparrow \infty}
        C_k^{\hat{f}_n}
        =
        C_k
    \mbox{ for every } k=1,\dots,K
	\end{equation}
	where the limit is in the upper Kuratowski sense.\\[0.25cm]
    Equivalently, for every convergent sequence $(g_n)_{n=1}^{\infty}$ in $\mathcal{K}$ we have
    \[
        \delta_{C}\Big(
            \lim\limits_{n\uparrow \infty}
            \, 
            g_n
        \Big)
    \le 
        \liminf_{n \to \infty} \delta_{C_n}(g_n)
    .
    \]
\end{theorem}
\begin{proof}
For the proofs see~\ref{s:Proofs}.
\end{proof}





%% file: images/blob_tikZ.tex
    \definecolor{targetcolor}{RGB}{60, 60, 60}    
    \definecolor{learncolor}{RGB}{0, 130, 90}      

    \newcommand{\targetpath}{
        (-2.2, 0.2)   
        (-1.5, 1.8)   
        (-0.2, 1.2)   
        (1.5, 2.0)    
        (2.8, 0.5)    
        (2.2, -1.5)   
        (0.8, -0.8)   
        (-0.5, -2.2)  
        (-1.8, -1.5)  
    }

    \begin{tikzpicture}[
        scale=0.8, 
        thick,
        target/.style={
            draw=targetcolor, 
            line width=1.5pt, 
            fill=targetcolor!5, 
            fill opacity=0.5,
            smooth cycle, 
            tension=0.7
        },
        learned/.style={
            draw=learncolor, 
            line width=1.2pt, 
            fill=learncolor!8, 
            fill opacity=0.4,
            smooth cycle, 
            tension=0.7
        },
        island/.style={
            draw=learncolor, 
            line width=1.0pt, 
            fill=learncolor!15, 
            fill opacity=0.6,
            smooth cycle,
            tension=0.7
        },
        labeltext/.style={text=black, font=\sffamily, align=left}
    ]

    \begin{scope}[shift={(0, -3.8)}] 
        \node (t_line) [draw=targetcolor, fill=targetcolor!5, fill opacity=0.8, line width=1pt, minimum width=0.6cm, minimum height=0.3cm, rounded corners=1pt] at (-1.0, 0) {};
        \node [right=0.15cm of t_line, text=black!90, font=\small\sffamily] {Target ($P_{\text{data}}$)};
        
        \node (l_line) [draw=learncolor, fill=learncolor!8, fill opacity=0.6, line width=1pt, minimum width=0.6cm, minimum height=0.3cm, rounded corners=1pt] at (5.5, 0) {};
        \node [right=0.15cm of l_line, text=black!90, font=\small\sffamily] {Learned by NO ($P_{\theta}$)};
    \end{scope}

    \begin{scope}[local bounding box=leftfig]
        
        \node[labeltext, anchor=south west] at (-2.8, 2.2) {\large \textbf{a} \quad Intermediate Phase};

        \draw[target] plot coordinates { \targetpath };

        \draw[learned] plot coordinates {
            (-1.6, 0.3)   
            (-1.1, 1.3)   
            (-0.2, 0.8)   
            (1.2, 1.4)    
            (1.9, 0.5)    
            (1.6, -1.0)   
            (0.8, -0.4)   
            (-0.2, -1.6)  
            (-1.3, -1.1)
        };
        
        \draw[island] plot coordinates {(-1.8, 0.4) (-2.1, 0.6) (-1.6, 0.7) (-1.6, 0.5)}; 
        \draw[island] plot coordinates {(-1.8,.1) (-2.1, -.6) (-1.8, -1.2) (-1.6, -1.1)}; 
        \draw[island] plot coordinates {(2.3, 0.4) (2.6, 0.5) (2.5, 0.7) (2.2, 0.6)}; 
        \draw[island] plot coordinates {(-1.4, 1.55) (-1.3, 1.7) (-1.2, 1.55)}; 
        \draw[island] plot coordinates {(1.4, 1.7) (1.6, 1.8) (1.5, 1.6)};
        \draw[island] plot coordinates {(-0.5, -1.9) (-0.3, -2.1) (-0.2, -1.9)}; 
        
    \end{scope}

    \begin{scope}[xshift=7.5cm, local bounding box=rightfig]
        
        \node[labeltext, anchor=south west] at (-2.8, 2.2) {\large \textbf{b} \quad Near Convergence};

        \draw[target] plot coordinates { \targetpath };

        \draw[learned] plot coordinates {
            (-2.0, 0.2)
            (-1.4, 1.65)
            (-0.2, 1.0)   
            (1.4, 1.85)
            (2.6, 0.5)
            (2.0, -1.35)
            (0.8, -0.6)   
            (-0.4, -2.0)
            (-1.7, -1.4)
        };
        
        \draw[island] (-0.2, 1.12) circle (0.04); 
        \draw[island] (0.8, -0.72) circle (0.04); 
        \draw[island] (2.72, 0.5) circle (0.03);  
        \draw[island] (-2.12, 0.2) circle (0.03);
        \draw[island] (-1.35, 1.75) circle (0.04);
        
    \end{scope}

    \end{tikzpicture}

%% file: images/Architecture_TiKz.tex
\begin{tikzpicture}[
    >=Stealth,
    node distance=1.2cm,
    panel_label/.style={font=\bfseries\large\sffamily, text=black, anchor=north west},
    axis_style/.style={draw=nGray, ->, thick, line cap=round},
    func_curve/.style={draw=nBlue, very thick, smooth, line cap=round},
    kernel_fill/.style={fill=nOrange, opacity=0.15},
    kernel_line/.style={draw=nOrange, thick, smooth},
    sample_dot/.style={circle, fill=nBlue!80!black, inner sep=0pt, minimum size=3.5pt},
    sample_line/.style={draw=nGray!50, dashed, thin},
    layer_box/.style={rounded corners=4pt, draw=#1!40, fill=#1!5, inner sep=6pt, line width=0.8pt},
    neuron/.style={circle, draw=#1!80, fill=white, line width=0.8pt, inner sep=0pt, minimum size=0.55cm}, 
    connection/.style={draw=nGray!20, line width=0.3pt}
]

    \begin{scope}[local bounding box=panelA]
        \node[panel_label] at (-0.8, 4.2) {a Sampling-Based Discretization};

        \draw[axis_style] (-0.5, 0) -- (6.0, 0) node[right, font=\footnotesize, text=nGray] {$\mathcal{X}$};
        \draw[axis_style] (0, -0.2) -- (0, 3.0);

        \path[name path=function_curve] plot[smooth, tension=0.6] coordinates {
            (0.0, 0.8) (0.7, 1.8) (1.8, 1.0) (3.5, 2.5) (4.5, 1.1) (5.3, 1.8) (6.0, 0.5)
        };
        
        \draw[func_curve] plot[smooth, tension=0.6] coordinates {
            (0.0, 0.8) (0.7, 1.8) (1.8, 1.0) (3.5, 2.5) (4.5, 1.1) (5.3, 1.8) (6.0, 0.5)
        };
        \node[text=nBlue, font=\large\bfseries, anchor=west] at (3.6, 2.8) {$f \in \mathcal{H}$};

        \node[font=\scriptsize, text=nOrange!80!black, align=center] at (3.0, -0.65) {Non-uniform sampling};
        \foreach \pos/\h/\w/\lab in {0.7/1.0/20/1, 1.9/1.3/35/2, 4.0/0.9/18/s, 5.3/1.2/28/S} {
            
            \draw[kernel_line] plot[domain=\pos-0.7:\pos+0.7, samples=40] 
                (\x, {2.5 * \h * exp(-\w * (\x-\pos)^2)});
            \fill[kernel_fill] plot[domain=\pos-0.7:\pos+0.7, samples=40] 
                (\x, {2.5 * \h * exp(-\w * (\x-\pos)^2)}) -- cycle;
                
            \draw[nGray] (\pos, 0) -- (\pos, -0.1) node[below, font=\scriptsize, text=nGray] {$x_{\lab}$};
            
            \path[name path=vert_\lab] (\pos, 0) -- (\pos, 3.0);
            
            \path [name intersections={of=function_curve and vert_\lab, by=intersect_\lab}];
            
            \draw[sample_line] (\pos, 0) -- (intersect_\lab);
            \node[sample_dot] at (intersect_\lab) {};
        }
        
        \node[text=nOrange, font=\small\bfseries] at (3., 0.4) {$\kappa(\cdot, x_s)$};
        \node[text=nOrange!70!black, font=\tiny, align=center, fill=white, inner sep=2pt] at (1., 2.8) {Reproducing\\kernel};

        \node[right=1.0cm of {6.2, 1.4}, font=\small, align=center, fill=nBlue!5, draw=nBlue!30, rounded corners, inner sep=6pt] (vector) {
            $\begin{bmatrix}
            \langle f, \kappa_1 \rangle_{\mathcal{H}} \\
            \langle f, \kappa_2 \rangle_{\mathcal{H}} \\
            \vdots \\
            \langle f, \kappa_S \rangle_{\mathcal{H}}
            \end{bmatrix}$ \\[0.5em]
            \footnotesize $v_0 \in \mathbb{R}^{S}$
        };

        \draw[->, nBlue!60, thick, dashed] (intersect_S) to[out=0, in=180] (vector.west);
        \node[above, font=\scriptsize\bfseries, text=nBlue!80] at ($(intersect_S)!0.5!(vector.west)$) {Projection};

    \end{scope}

    \begin{scope}[yshift=-5.5cm, local bounding box=panelB]
        \node[panel_label] at (-0.8, 3.5) {b Deep Neural Operator Architecture};

        \coordinate (InPos) at (0,0);
        \coordinate (H1Pos) at (2.5,0);
        \coordinate (H2Pos) at (5.0,0);
        \coordinate (OutPos) at (7.5,0);

        
        \foreach \i in {1,2,3,4} {
            \node[neuron=nBlue] (I\i) at ($(InPos) + (0, {1.5 - 0.7*\i})$) {};
        }
        \node[font=\bfseries, text=nBlue] at ($(InPos) + (0, -1.8)$){};
        
        \foreach \i in {1,2,3,4} {
            \node[neuron=nGreen] (H1\i) at ($(H1Pos) + (0, {1.5 - 0.7*\i})$) {};
        }
        \node[font=\bfseries, text=nGreen] at ($(H1Pos) + (0, -1.8)$){};
        
        \foreach \i in {1,2,3,4} {
            \node[neuron=nGreen] (H2\i) at ($(H2Pos) + (0, {1.5 - 0.7*\i})$) {};
        }
        \node[font=\bfseries, text=nGreen] at ($(H2Pos) + (0, -1.8)$){};
        
        \node[font=\huge, text=nGray!50] (ellipsis) at ($(H2Pos)!0.5!(OutPos)$) {$\cdots$};
        \node[below=0.25cm of ellipsis, font=\tiny, text=nGray, align=center] {$L-2$ \\hidden\\ layers};

        \foreach \i in {1,2,3} {
            \node[neuron=nPurple] (O\i) at ($(OutPos) + (0, {1.15 - 0.75*\i})$) {};
        }
        \node[font=\bfseries, text=nPurple] at ($(OutPos) + (0, -1.8)$){};

        \begin{scope}[on background layer]
            \foreach \i in {1,2,3,4} {
                \foreach \j in {1,2,3,4} {
                    \draw[connection] (I\i) -- (H1\j);
                }
            }
            \foreach \i in {1,2,3,4} {
                \foreach \j in {1,2,3,4} {
                    \draw[connection] (H1\i) -- (H2\j);
                }
            }
        \end{scope}

        \node[layer_box=nBlue, fit=(I1) (I4), label={[text=nBlue, font=\small\bfseries]above:Input}] (B_In) {};
        \node[below, font=\scriptsize, text=nGray] at (B_In.south) {$v_0 \in \mathbb{R}^{d_0}$};

        \node[layer_box=nGreen, fit=(H11) (H14), label={[text=nGreen, font=\small\bfseries]above:Hidden}] (B_H1) {};
        \node[below, font=\scriptsize, text=nGray] at (B_H1.south) {$A_1: \mathbb{R}^{d_0} \to \mathbb{R}^{d_1}$}; 

        \node[layer_box=nGreen, fit=(H21) (H24), label={[text=nGreen, font=\small\bfseries]above:Hidden}] (B_H2) {};
        \node[below, font=\scriptsize, text=nGray] at (B_H2.south){$A_2: \mathbb{R}^{d_1} \to \mathbb{R}^{d_2}$};

        \node[layer_box=nPurple, fit=(O1) (O3), label={[text=nPurple, font=\small\bfseries]above:Output}] (B_Out) {};
        \node[below, font=\scriptsize, text=nGray] at (B_Out.south) {$v_L \in \mathbb{R}^{K}$};
        \node[below=0.48cm of B_Out.south, font=\scriptsize, text=nGray] {$d_{L+1} = K$};

        \node[right=1.2cm of B_Out, align=left, font=\scriptsize, text=nGray!20!black, fill=nPurple!5, draw=nPurple!30, rounded corners, inner sep=6pt] (Clusters) {
            \textbf{Cluster Sets:}\\[0.3em]
            $C_1^{\hat{f}} = \{f \in \mathcal{H} : \sigma(\hat{f}(f)_1) \ge \gamma\}$ \\
            $C_2^{\hat{f}} = \{f \in \mathcal{H} : \sigma(\hat{f}(f)_2) \ge \gamma\}$ \\
            \hspace{1.5cm}$\vdots$ \\
            $C_K^{\hat{f}} = \{f \in \mathcal{H} : \sigma(\hat{f}(f)_K) \ge \gamma\}$
        };
        
        \node[above=0.08cm of Clusters, font=\small\bfseries, text=nPurple] {Open Subsets of $\mathcal{H}$};
        \node[below=0.08cm of Clusters, font=\scriptsize, text=nGray] {(Equation~\ref{eq:induced_clusters})};
        
        \draw[->, very thick, nPurple] (B_Out.east) -- (Clusters.west) node[midway, above, font=\tiny, text=nPurple] {$\hat{f}(f)$};

    \end{scope}

    \node[anchor=north, align=center, font=\small, draw=nGray!30, fill=white, rounded corners, inner sep=8pt] at ($(panelB.south) + (0, -0.2)$) {
        \textbf{Definition~\ref{defn:SNO}:} $\hat{f}(f) = A_L \circ (\text{ReLU} \bullet A_{L-1}) \circ \cdots \circ (\text{ReLU} \bullet A_1)(v_0)$
    };

\end{tikzpicture}

%% file: Our_method.tex
\section{An Illustration: Clustering Trajectories Associated with Distinct ODEs}

\label{our method}
We present a concrete instantiation of our sampling-based neural operator framework for clustering functional observations.
As a case study, we cluster trajectories generated by distinct ODE systems by learning an operator model $\hat f:\mathcal H\to\mathbb R^K$.
Formally, we assume that the underlying trajectories $\{x_i\}_{i=1}^N$ reside in a Hilbert space $\mathcal{H}$ defined on a time domain $\mathcal{T}$. 
The observed dataset is $\{(x_i(t_i),\ell_i)\}_{i=1}^N$, where each trajectory $x_i$ is sampled at solution points $t_i=(t_{i1},\ldots,t_{in})$ to yield $x_i(t_i)=(x_i(t_{i1}),\ldots,x_i(t_{in}))^\top$. Ground-truth labels $\ell_i\in\{1,\ldots,C\}$ are used only for evaluation.

Classical operator learning's 
architectures such as DeepONet \cite{lu2021deeponet} or FNO \cite{li2020fourier} are primarily designed for supervised regression tasks, where the objective is to approximate function values or operators in pointwise or operator-norm senses. In contrast, our work adapts this operator-theoretic framework to the unsupervised setting. 
Instead of mapping functions to functions, we leverage the universal approximation capabilities of neural operators to map infinite-dimensional trajectories directly to finite-dimensional cluster assignments, effectively extending the scope of operator learning from regression to pattern discovery.

To empirically verify the asymptotic convergence guarantees established in Theorem~\ref{thrm:MainKuratowski}, we require a constructive realization of the sampling-based neural operator $\mathcal{F}_{SNO}$. We specifically construct the operator $\hat{f}$ as a composition of a discretized sampling projection and a high-capacity continuous feature map. This design is not merely a heuristic for scalability, but a numerical necessity to satisfy the universality assumptions (Proposition~\ref{prop:universality}) regarding the density of the operator class in the function space.

Thus, we instantiate the operator mapping  $\hat f:\mathcal H\to\mathbb R^K$ using a practical architecture as a composition of three stages:
(i) sampling and registration of functional inputs,
(ii) a fixed feature map induced by a pre-trained encoder,
and (iii) a trainable MLP producing cluster logits, as illustrated in Figure~\ref{fig:SamplingNO}.

\subsection{Trajectory Registration}

To satisfy the SNO definition (Definition~\ref{defn:SNO}), we instantiate the dimension-reduction operator $P_{\mathcal{H}:m}$~\cite{CNO} as a finite-dimensional discretization via trajectory rendering.
Projecting a trajectory to a $224 \times 224$ grid defines a finite-dimensional measurement operator $P_{\mathcal{H}:m}$, producing a vector $v \in \mathbb{R}^{S}$ with $S = 50{,}176$. The resulting pixel vector $v \in \mathbb{R}^S$ constitutes the finite-dimensional input $v_0$, effectively bridging the infinite-dimensional Hilbert space $\mathcal H$ with the encoder input space as required by Proposition~\ref{prop:universality}.

Prior to rendering, trajectories are linearly normalized to $[-1,1]$. This ensures numerical stability and focuses the learning objective on intrinsic geometric shape rather than absolute amplitude. While trajectory images primarily capture geometric information, they may overlook spectral signatures in high-frequency or quasi-periodic systems. To address this, we additionally employ Short-Time Fourier Transforms (STFT), providing a complementary spectral view~\cite{salazar2020deep, perez2020new}.

\subsection{Pair Construction Backbone}

Inspired by BYOL, we construct positive pairs to probe the local shape stability of cluster assignments.
Since Theorem~\ref{thrm:MainKuratowski} establishes that each true cluster $C_k$ is a closed set in $\mathcal{H}$, any admissible perturbation within a functional neighborhood should not alter the trajectory's cluster identity.
Given a registered image $x_i$, two stochastic transformations $T^a, T^b\in\mathcal{T}$ produce $x_i^{a} = T^a(x_i)$ and $x_i^{b} = T^b(x_i)$.
We adopt commonly-used \texttt{RandomCrop} and \texttt{GaussianBlur} to simulate these neighborhood perturbations in the discretized domain. In this context, spatial cropping simulates temporal translations or phase shifts in the underlying ODE, while Gaussian blurring models the smoothing of high-frequency observation noise.

The frozen encoder $\phi$ can be interpreted as a fixed, nonlinear feature map applied to these sampled evaluations, producing feature embeddings
$h_i^{a} = \phi(x_i^a)$ and $h_i^{b} = \phi(x_i^b)$.
Enforcing consistency between these paired embeddings trains the operator to be invariant under admissible sampling noise, effectively smoothing the decision boundary in the RKHS.

\subsection{Cluster Head}

Let $g:\mathbb R^d\to\mathbb R^K$ denote the trainable MLP cluster mapping.
Given encoder features, the head outputs \emph{logits}
\[
z_i^a = g(h_i^a),
\qquad
z_i^b = g(h_i^b),
\qquad
z_i^a, z_i^b \in \mathbb R^K,
\]
and the corresponding \emph{soft assignments} are obtained via the softmax map:
\[
y_i^a = \operatorname{softmax}(z_i^a),
\qquad
y_i^b = \operatorname{softmax}(z_i^b),
\qquad
y_i^a, y_i^b \in \Delta^{K-1}.
\]
Thus $y_i^a$ and $y_i^b$ serve as soft pseudo-labels. At inference time, the encoder feature is passed through $g(\cdot)$ to obtain a stable cluster assignment, which can be viewed as thresholding a monotone transformation of the logits, as in~\eqref{eq:clearned_cluster}.

\subsection{Clustering Objective Function}

The clustering objective function is designed to steer the SNO toward the ideal clustering partition defined in Section~\ref{s:UniversalClustering}. Specifically, the loss terms are constructed to satisfy the conditions for universal clustering: (1) operator continuity, (2) indicator function convergence, and (3) non-degenerate partitioning.

First, to ensure continuity across functional views, we maximize the similarity between soft assignments using a row-wise cross-entropy:
\[
L_e
=
- \sum_{i=1}^{N} \sum_{k=1}^{K}
y^{a}_{i,k}\, \log y^{b}_{i,k}.
\]
Consistency alone may allow overly flat distributions. To promote the convergence of soft assignments toward sharp indicator functions (as required by the Sierpiński space formulation in Lemma~\ref{lem:classification_open__generalapproximation___hardclassificatoin}), we introduce a confidence term:
\[
L_{\mathrm{con}}
=
- \log \frac{1}{N} \sum_{i=1}^{N}
(y_i^{a})^{\top} y_i^{b}.
\]
This term drives the soft assignments to concentrate, facilitating convergence toward indicator functions.
Finally, Lemma~\ref{lem:Kclusters_are_open} ensures that the induced clusters are closed and non-empty; to avoid degenerate collapse in practice, we maximize the marginal entropy.
\[
H(Y)
=
- \sum_{k=1}^{K}
\left[
P^a_k \log P^a_k
+
P^b_k \log P^b_k
\right],
\]
where
\[
P^a_k = \frac{1}{N} \sum_{i=1}^{N} y^{a}_{i,k},
\qquad
P^b_k = \frac{1}{N} \sum_{i=1}^{N} y^{b}_{i,k}.
\]
The final joint objective is given by
\[
L_{\mathrm{clu}}
=
L_e
+
L_{\mathrm{con}}
-
\alpha\, H(Y),
\]
with $\alpha>0$ controlling the regularization strength. This total objective function yields clusters that are simultaneously stable, sharp, and well-balanced, providing a numerical surrogate for the theoretical convergence guarantees. The overall iterative procedure is detailed in Algorithm~\ref{alg:CAC} in the Appendix.

%% file: Data_benchmarking.tex
\section{Our benchmark: Synthetic Data Generation}
\label{sec:synthetic}

To empirically validate the sampling-based operator learning framework, we construct two synthetic functional datasets that control the geometry of solution manifolds and the severity of decision-boundary ambiguity in the underlying function space.
In both datasets, the underlying dynamical systems generate trajectories in a multi-dimensional phase space (typically coupled states $(u(t), v(t))$). 
For the clustering task, we project these trajectories onto their first coordinate component $u(t)$, represented on a common time grid, and focus on the observable dynamics $f_i = u_i(\cdot)\in\mathcal{H}$.
This design aligns with our theory in Section~\ref{s:UniversalClustering}: the true clusters are modeled as closed subsets of $\mathcal{H}$, and the hardest instances arise near ambiguous decision boundaries, where different systems can produce visually similar sampled trajectories after sampling and registration. Importantly, ODEs serve here only as a controlled generator of functional observations; our universal clustering framework applies to general functional data in $\mathcal{H}$.

\paragraph{Two regimes: structured vs.\ high-variability}
The ODE-6 dataset provides a structured regime in which classes correspond to qualitatively distinct system families (linear vs.\ nonlinear, homogeneous vs.\ non-homogeneous, and IVP vs.\ BVP), yielding well-separated solution patterns under moderate sampling noise.
Additionally, we prepare the ODE-4 dataset, which provides a high-variability regime generated by neural vector fields with randomized parameters, where intra-class diversity is substantially larger and inter-class overlap is more likely after discretization.
Together, these datasets form a controlled benchmark for assessing when a sampling-based operator can reliably approximate the cluster regions in $\mathcal{H}$ and when the task approaches the decision boundary.

\subsection{ODE-6 Dataset: Distinct Dynamical Families}

ODE-6 contains six major categories of dynamical systems, spanning distinct classes including boundary value problems (BVP) and initial value problems (IVP). For each of the six categories, we define three distinct subclasses by varying initial/boundary conditions or system parameters. Within each subclass, we simulate 500 trajectories on uniform grids, yielding 9,000 diverse functional samples across all subclasses. For the ODE-6 dataset, we adopt the notation $u(t)$ for the system state, adhering to standard conventions in the physical sciences literature.
The six defining systems are:

\subsubsection*{1. Linear Boundary Value Problem (BVP)}
A second-order linear boundary system defined on $x\in[0,1]$:
\[
u''(x) = -k \sin(k\pi x), \qquad u(0)=0,\; u(1)=0.
\]
The scalar parameter $k$ controls the oscillation frequency and is sampled as $k \sim \mathcal{U}(0.5,\,5.5)$. Each trajectory is represented in phase space as $(u(x), u'(x))$, where 
$x$ denotes the spatial coordinate. This class exhibits smooth harmonic-like curves with varying spatial frequencies.

\subsubsection*{2. Nonlinear Boundary Value Problem (Bratu Type)}
A nonlinear BVP governed by the Bratu equation:
\[
u''(x) + \lambda e^{u(x)} = 0, \qquad u(0)=0,\; u(1)=0,
\]
where the nonlinearity parameter $\lambda$ is drawn from $\mathcal{U}(0,\,3.5)$, remaining below the classical bifurcation limit $\lambda_c \approx 3.51$. The solution displays smooth, monotonic, or symmetric profiles depending on $\lambda$.

\subsubsection*{3. Linear Homogeneous System}
A two-dimensional linear autonomous system:
\[
\begin{cases}
\dot{u}_1 = a\,u_1 + b\,u_2, \\
\dot{u}_2 = c\,u_1 + d\,u_2,
\end{cases}
\qquad
(a,b,c,d) \sim \mathcal{N}(0,1),
\]
with initial condition $u(0) = [1,1]^\top$ and time domain $t\in[0,10]$. Depending on the eigenvalues of the coefficient matrix, the system exhibits exponential growth, decay, or oscillatory behavior, forming distinct linear flows in phase space.

\subsubsection*{4. Linear Non-Homogeneous System}
A damped-driven linear system with sinusoidal forcing:
\[
\begin{cases}
\dot{u}_1 = a\,u_1 + b\,u_2 + \sin(\omega t), \\
\dot{u}_2 = c\,u_1 + d\,u_2 + \cos(\omega t),
\end{cases}
\]
where parameters are sampled independently from
$a,d \sim \mathcal{U}(-0.5,\,-0.1)$, $b,c \sim \mathcal{U}(-1,\,1)$, and $\omega \sim \mathcal{U}(0.1,\,2.1)$.
The system evolves on $t\in[0,10]$ with $u(0)=[1,1]^\top$. This class introduces oscillatory non-stationarity due to the external forcing.

\subsubsection*{5. Nonlinear Homogeneous System (Lotka–Volterra)}
A classical predator–prey nonlinear system:
\[
\begin{cases}
\dot{u}_1 = \alpha u_1 - \beta u_1 u_2, \\
\dot{u}_2 = \delta u_1 u_2 - \gamma u_2,
\end{cases}
\]
with baseline parameters $(\alpha,\beta,\delta,\gamma) = (1.5,\,1.0,\,3.0,\,1.0)$. Each parameter is perturbed by a random factor $\epsilon_j \sim \mathcal{U}(-0.15,\,0.15)$. The system is simulated on $t\in[0,25]$ with $u(0)=[10,5]^\top$. Trajectories exhibit periodic oscillations (closed orbits around a center) that are topologically distinct from linear oscillations.

\subsubsection*{6. Nonlinear Non-Homogeneous System (Forced Van der Pol)}
A forced non-autonomous nonlinear oscillator:
\[
\begin{cases}
\dot{u}_1 = u_2, \\
\dot{u}_2 = \mu(1-u_1^2)u_2 - u_1 + A\cos(\omega t),
\end{cases}
\]
where $\mu \sim \mathcal{U}(0.1,\,2.1)$, $A \sim \mathcal{U}(0.1,\,1.1)$, and $\omega \sim \mathcal{U}(0.5,\,2.5)$.
The initial condition is $u(0) = [1,0]^\top$ and $t\in[0,50]$. This system produces stiff nonlinear oscillatory dynamics ranging from regular limit cycles to complex transient patterns.

\subsection{ODE-4 Dataset: Randomized Neural ODEs}

The ODE-4 dataset serves as a benchmark for clustering functional data exhibiting non-canonical dynamics and high intra-class variability. To achieve this, we model the data generation process as a Randomized Neural ODE, where the governing physical laws are replaced by parameterized neural vector fields.

Formally, the time-evolution of the system is driven by a vector field $V: \mathbb{R}^d \to \mathbb{R}^d$, parameterized as a shallow neural network:
\begin{equation}
    V(u) = A\,\sigma_r(Bu + c),
\end{equation}
where $A \in \mathbb{R}^{d \times W}$ and $B \in \mathbb{R}^{W \times d}$ are weight matrices, and $c \in \mathbb{R}^W$ is a bias vector. The network weights are sampled from a centered uniform distribution $\mathcal{U}[-s/\sqrt{W}, s/\sqrt{W}]$, where the scaling factor $s$ is modulated by a complexity parameter $r$.

To introduce controllable algebraic complexity, we replace the standard ReLU with a parametric power activation function $\sigma_r(\cdot)$. For an input scalar $z$, the activation is defined as:
\begin{equation}
    \sigma_r(z) = \left( \max\left\{0, \frac{z}{r}\right\} \right)^r.
\end{equation}
Here, the integer $r \in \{1, \dots, R\}$ acts as a smoothness and complexity order parameter. As $r$ increases, the Neural ODE transitions from modeling piecewise linear dynamics ($r=1$) to higher-order polynomial interactions, fundamentally altering the spectral properties of the trajectories. For numerical stability at high orders ($r > 8$), a smoothed saturation profile is applied.

The dataset comprises four distinct classes of Neural ODEs. For each trajectory, the network parameters $(A, B, c)$ are sampled independently:

\begin{enumerate}
    \item \textit{First-order Homogeneous}: Defined by the autonomous Neural ODE $\dot{u}(t) = V(u(t))$.
    \item \textit{First-order Non-homogeneous}: $\dot{u}(t) = V(u(t)) + f(t) - \lambda u(t)$. Here, $\lambda=0.05$ is a constant damping coefficient ensuring stability under forcing.
    \item \textit{Second-order Homogeneous}: $\ddot{u}(t) = -V(u(t)) - \gamma(u, \dot{u})\,\dot{u}(t)$. This formulation employs a state-dependent adaptive damping term $\gamma(u, \dot{u})$ that scales with $|V(u)|$ to stabilize high-velocity excursions.
    \item \textit{Second-order Non-homogeneous}: $\ddot{u}(t) = -V(u(t)) - \gamma(u, \dot{u})\,\dot{u}(t) + F(t)$.
\end{enumerate}

The time-varying forcing terms $f(t)$ and $F(t)$ are constructed as random mixtures of canonical basis functions, including sinusoidal, polynomial, exponential decay, and Gaussian kernels. We generate 2,000 trajectories for each class (aggregating 100 samples for each complexity level $r \in \{1, \dots, 20\}$). Simulations are recorded on a temporal grid of size $T=101$ (spanning $t \in [0, 50]$), yielding a total of 8,000 functional samples.

%% file: Experiments.tex
\section{Experiments} 
\label{sec:experiments}

We evaluate the proposed clustering framework on two simulated functional datasets, ODE-6 and ODE-4, which represent distinct regimes of dynamical complexity. 
These experiments are designed not only to assess quantitative metrics but also to empirically verify the theoretical claims regarding clustering stability and discretization consistency established in Section~\ref{s:UniversalClustering}.

\paragraph{Evaluation Metrics}
We report clustering performance using three standard metrics: clustering accuracy (ACC), adjusted rand index (ARI), and normalized mutual information (NMI). Higher values indicate a closer approximation to the true partition of the function space.

\paragraph{Experimental Setup}
For each dataset, we extract functional representations via the discretized sampling operator.
To ensure a rigorous benchmark, we compare our method against classical functional data analysis approaches: Functional PCA (FPCA) and B-Spline representations.
For these baselines, trajectories are pre-processed with a standard scaler to normalize mean and variance. We configure FPCA with 30 components and B-Spline with 40 basis functions to ensure accurate signal reconstruction.
Additionally, we include Dynamic Time Warping (DTW) combined with time-series K-means, a strong baseline specifically designed for aligning temporal sequences.
For our proposed method, we employ the original ViT-B/32 OpenAI CLIP model~\cite{radford2021learning} as the frozen feature map $\phi$. This choice serves as a practical instantiation of the lifting stage, and the theoretical framework applies equally to any sufficiently expressive fixed feature map. We optimize the sampling-based neural operator $\hat{f}$ by training its learnable projection head $g$ for 50 epochs. The projection head $g$ is instantiated as a multilayer perceptron with layer dimensions $D \to 1024 \to 768 \to 512 \to 1024 \to K$, where $D$ corresponds to the output dimension of the pre-trained encoder, and $K$ is the number of clusters.
We explicitly set $K$ to match the number of distinct governing systems used to generate the data. This setup treats systems with distinct underlying mechanics as separate clusters, challenging the model to discriminate between them even when their behaviors are similar.
Training uses the Adam optimizer with a cosine annealing learning rate schedule (initial rate $1 \times 10^{-3}$) and a batch size of 512.
Consistent with our sensitivity analysis, we set the regularization weight $\alpha = 1$ for both datasets.
Training on ODE-4 takes approximately three minutes on a single NVIDIA A100 GPU.

\subsection{Main Results: Validating Empirical Consistency}
\begin{table}[ht]
\centering
\footnotesize
\setlength{\tabcolsep}{4pt}
\begin{tabular}{l ccc ccc}
\toprule
\multirow{2}{*}{Method}
& \multicolumn{3}{c}{ODE-6 }
& \multicolumn{3}{c}{ODE-4 } \\
\cmidrule(r){2-4} \cmidrule(r){5-7}
& ACC & ARI & NMI & ACC & ARI & NMI \\
\midrule
FPCA + K-Means
  & 0.314 & 0.175 & 0.357
  & 0.291 & 0.005 & 0.021 \\
B-Spline + K-Means
  & 0.422 & 0.185 & 0.380
  & 0.293 & 0.006 & 0.022 \\
DTW + K-Means
  & 0.790 & 0.651 & 0.705
  & 0.393 & 0.088 & 0.212 \\
\midrule
CLIP + K-Means
  & 0.624 & 0.583 & 0.715
  & 0.491 & 0.243 & 0.441 \\
CLIP + Spectrogram + K-Means
  & 0.684 & 0.528 & 0.647
  & 0.500 & 0.284 & 0.436 \\
\midrule
Ours (SNO)
  & 0.933 & 0.868 & 0.913
  & 0.613 & \textbf{0.357} & 0.418 \\
Ours (SNO) + Spectrogram
  & \textbf{0.945} & \textbf{0.890} & \textbf{0.917}
  & \textbf{0.652} & 0.346 & \textbf{0.476} \\
\bottomrule
\end{tabular}
\caption{Comparison of clustering performance on test data. Best results are highlighted in bold. Both FPCA and B-spline struggle to capture shape distinctions. DTW achieves improved performance in the structured ODE-6 regime but fails to generalize to the high-variability ODE-4 setting. The proposed SNO outperforms baselines in the structured regime  and in the high-variability regime, demonstrating the benefit of learning decision regions from sampled observations in both regimes.}
\label{tab:ode_clustering_results}
\end{table}

Table~\ref{tab:ode_clustering_results} presents the quantitative comparison of clustering performance. In the structured ODE-6 regime, we observe that classical reconstruction-based methods, specifically FPCA and B-Spline, fail to capture the underlying dynamical distinctions. Their low accuracy suggests that representations prioritizing signal variance or smoothness are insufficient for disentangling distinct families. The DTW baseline yields a marked improvement by explicitly correcting for temporal misalignment; however, it still trails the proposed Sampling-based Neural Operator by a significant margin. Our method achieves an accuracy exceeding 93\% and improves the ARI by more than 0.28 compared to directly applying K-means to CLIP features. This shows that when dynamical systems form well-separated closed sets in the Hilbert space, the learned operator effectively approximates the induced cluster sets, surpassing both functional approximation and temporal alignment techniques.

As for the high-variability ODE-4 regime, we observe a general performance decline across all methods due to the complexity of the randomized neural vector fields. A crucial finding here is the degradation of the DTW baseline. Unlike in the structured regime, DTW fails to maintain robustness in the presence of high intra-class variability, with performance dropping below even the naive CLIP K-means baseline. This indicates that rigid geometric alignment is brittle when trajectories exhibit substantial stochastic fluctuations. Conversely, the SNO framework maintains the highest separability among all methods. This empirical resilience supports our theoretical analysis regarding the ill-posedness of clustering near decision boundaries: even when distinctions are numerically ambiguous, the operator learning approach successfully recovers latent dynamical signatures that are inaccessible to classical FDA or alignment-based metrics.

To empirically verify the theoretical claim, we also evaluate the clustering performance across varying sampling resolutions to approximate the convergence trend.

\begin{figure*}[!htbp]
    \centering
    \setkeys{Gin}{width=0.24\linewidth, height=2cm, keepaspectratio}

    \makebox[0.24\linewidth]{\hspace{-1.3cm} \textbf{$S=4$}} \hfill
    \makebox[0.24\linewidth]{\hspace{-0.5cm} \textbf{$S=16$}} \hfill
    \makebox[0.24\linewidth]{ \textbf{$S=64$}} \hfill
    \makebox[0.24\linewidth]{ \hspace{0.6cm} \textbf{$S=224$}}
    \par\medskip

    \begin{subfigure}{\linewidth}
        \centering
        \includegraphics{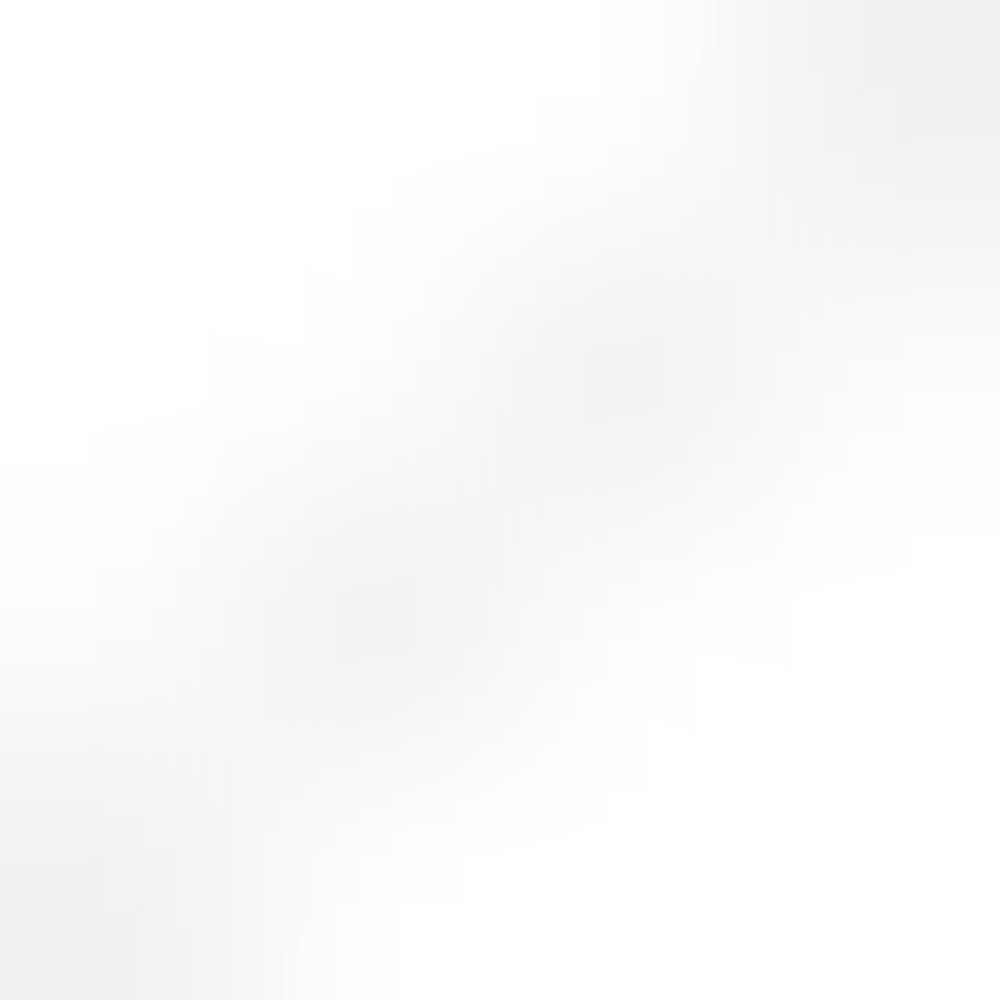}\hfill
        \includegraphics{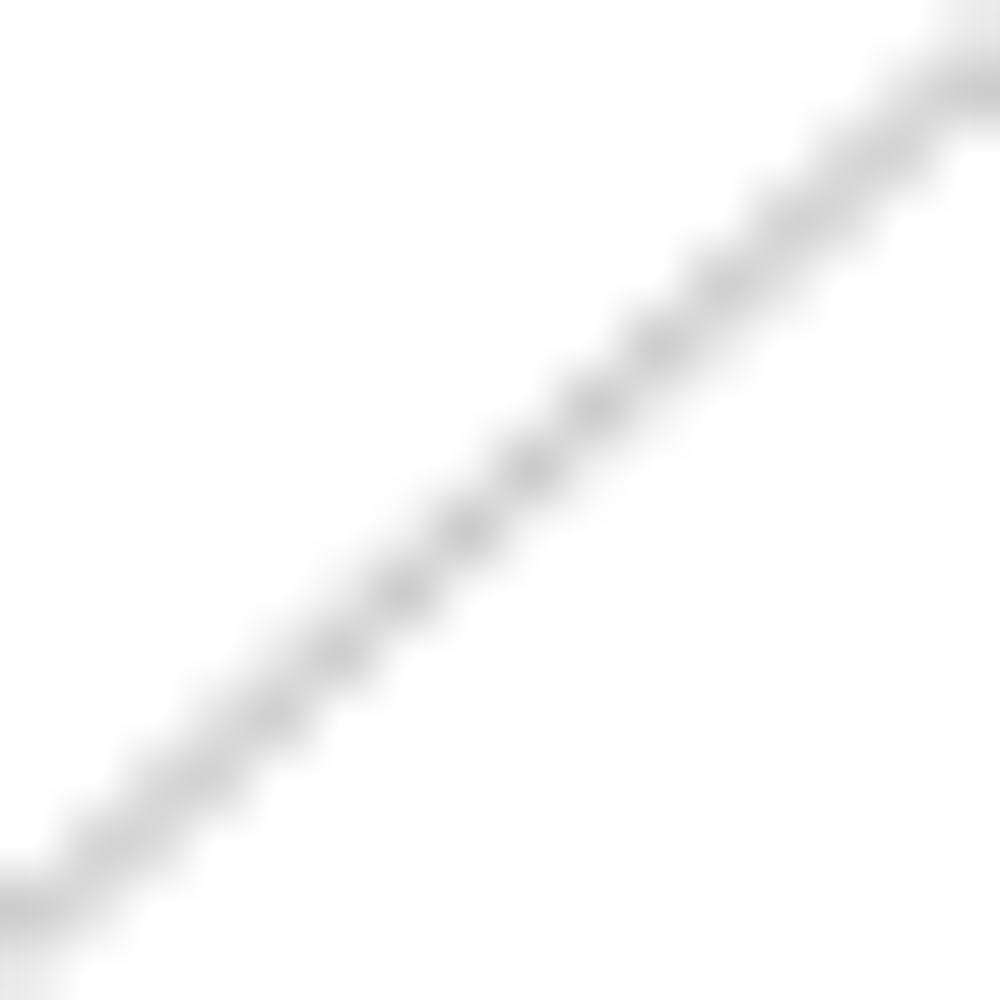}\hfill
        \includegraphics{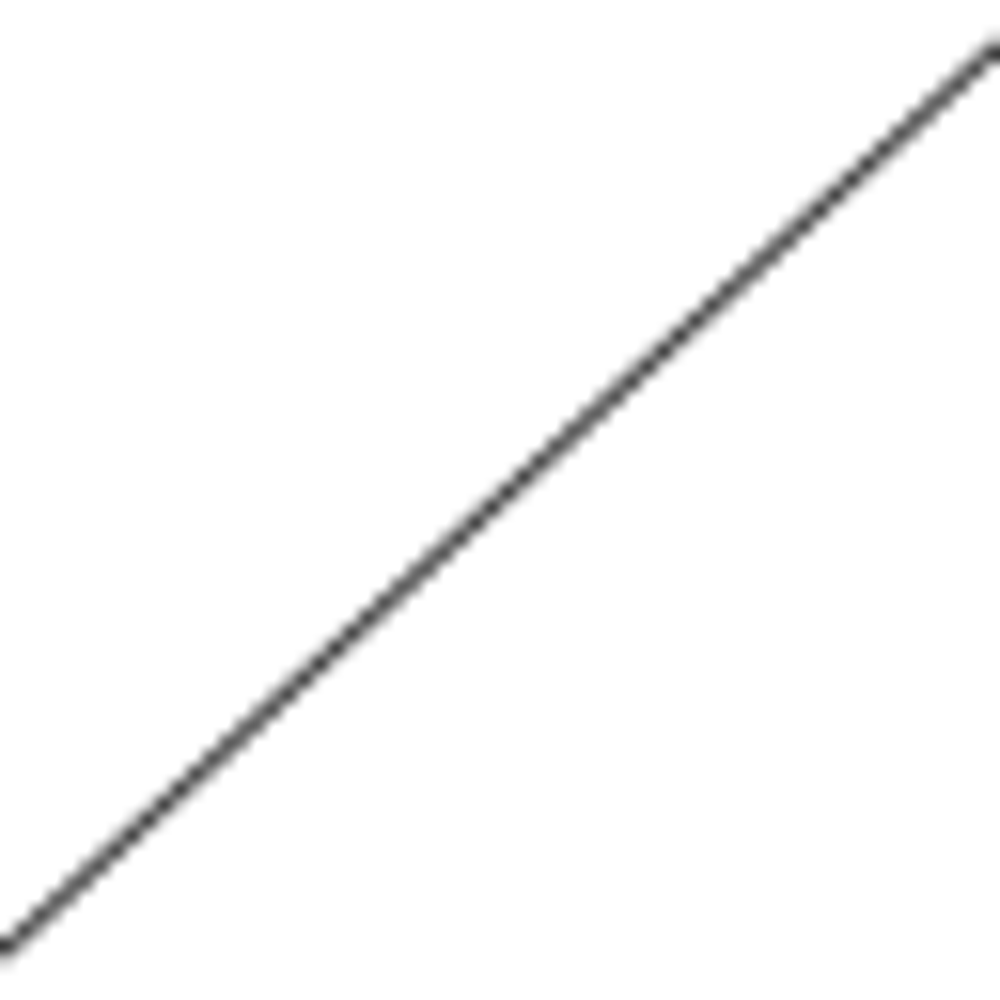}\hfill
        \includegraphics{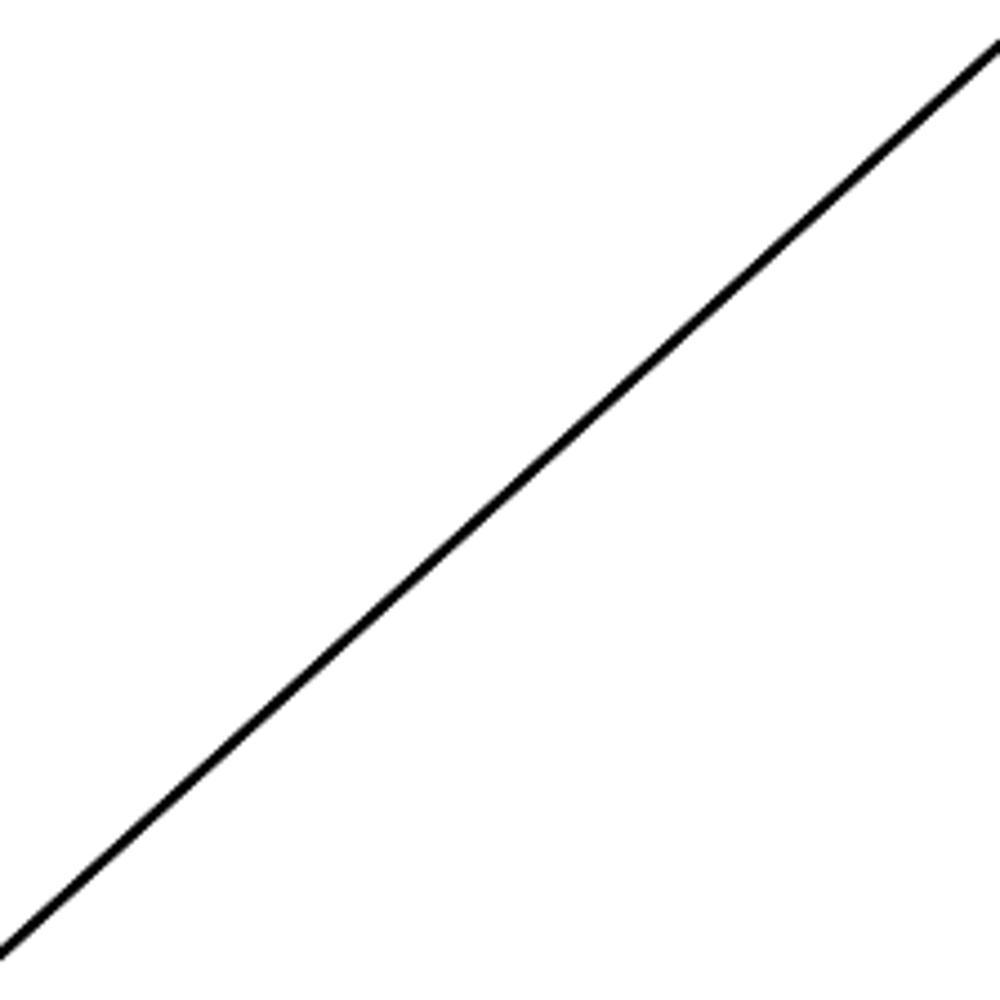}
        \caption{Linear Boundary Value Problem}
    \end{subfigure}
    \vspace{0.1cm}

    \begin{subfigure}{\linewidth}
        \centering
        \includegraphics{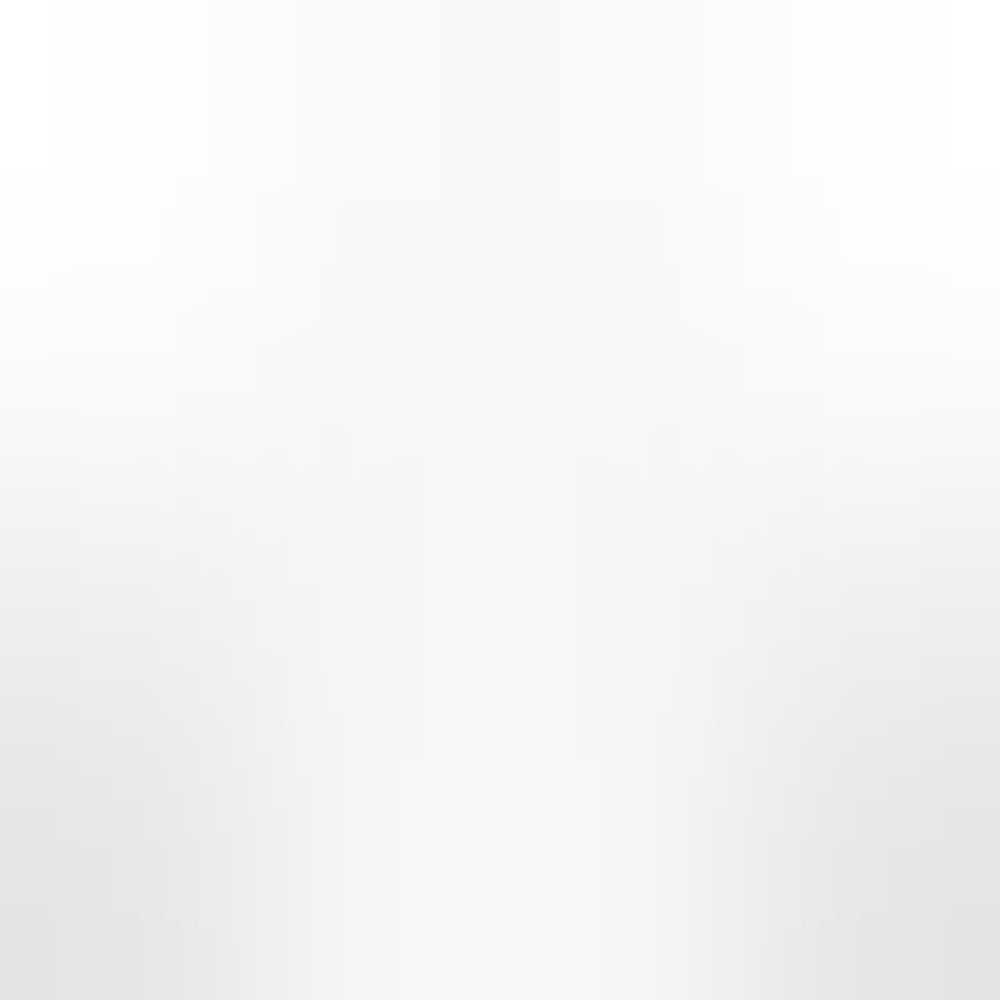}\hfill
        \includegraphics{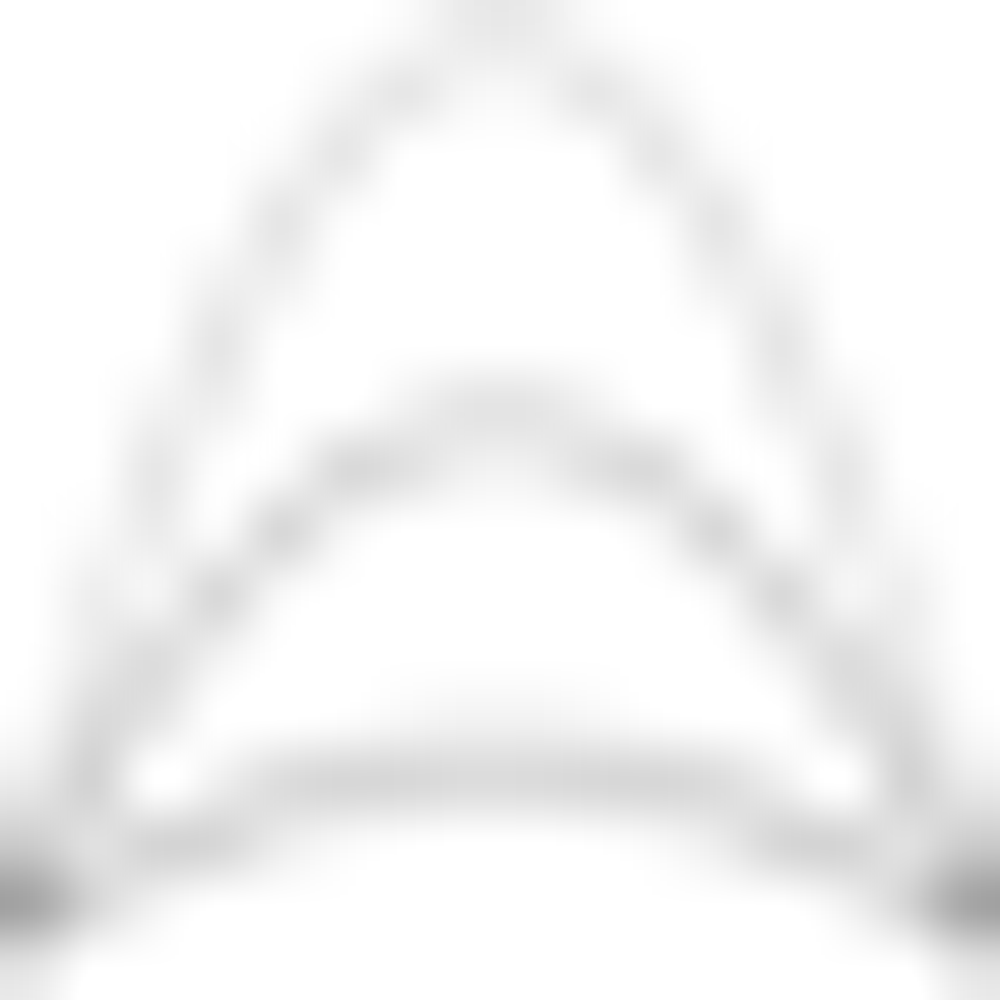}\hfill
        \includegraphics{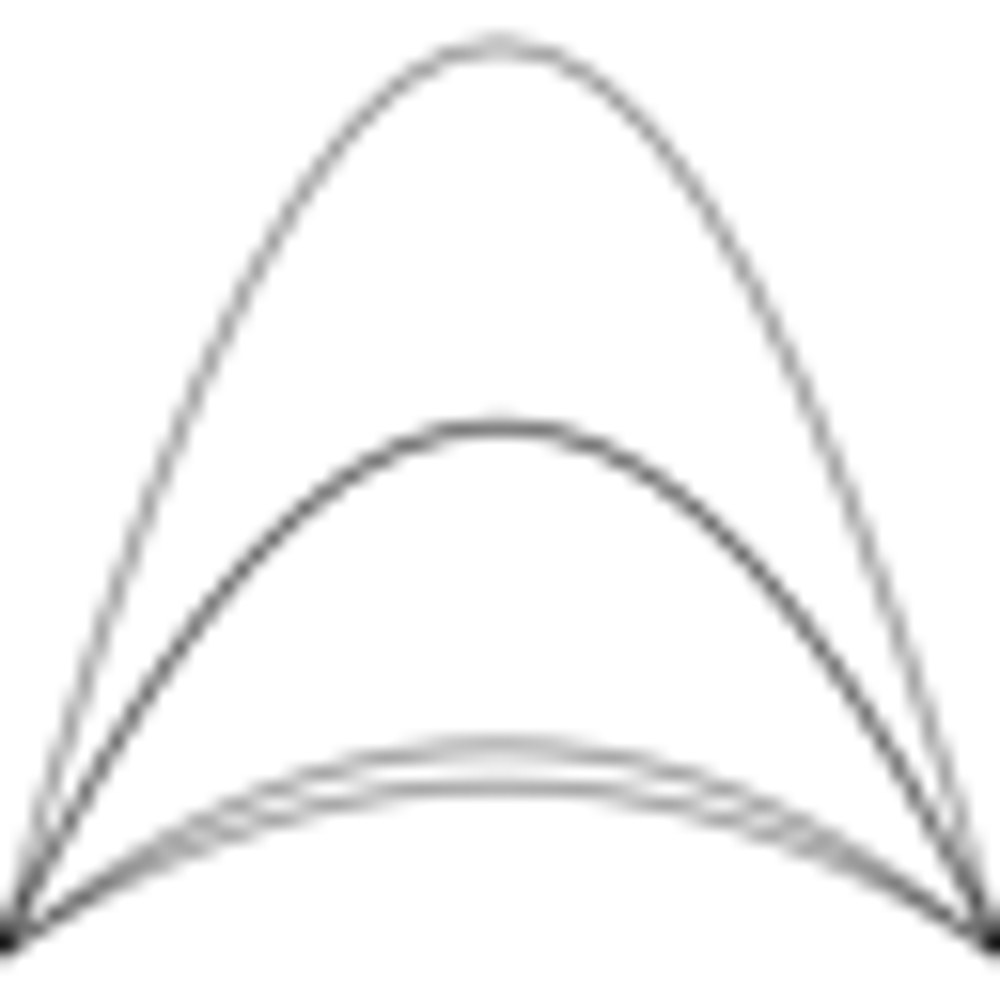}\hfill
        \includegraphics{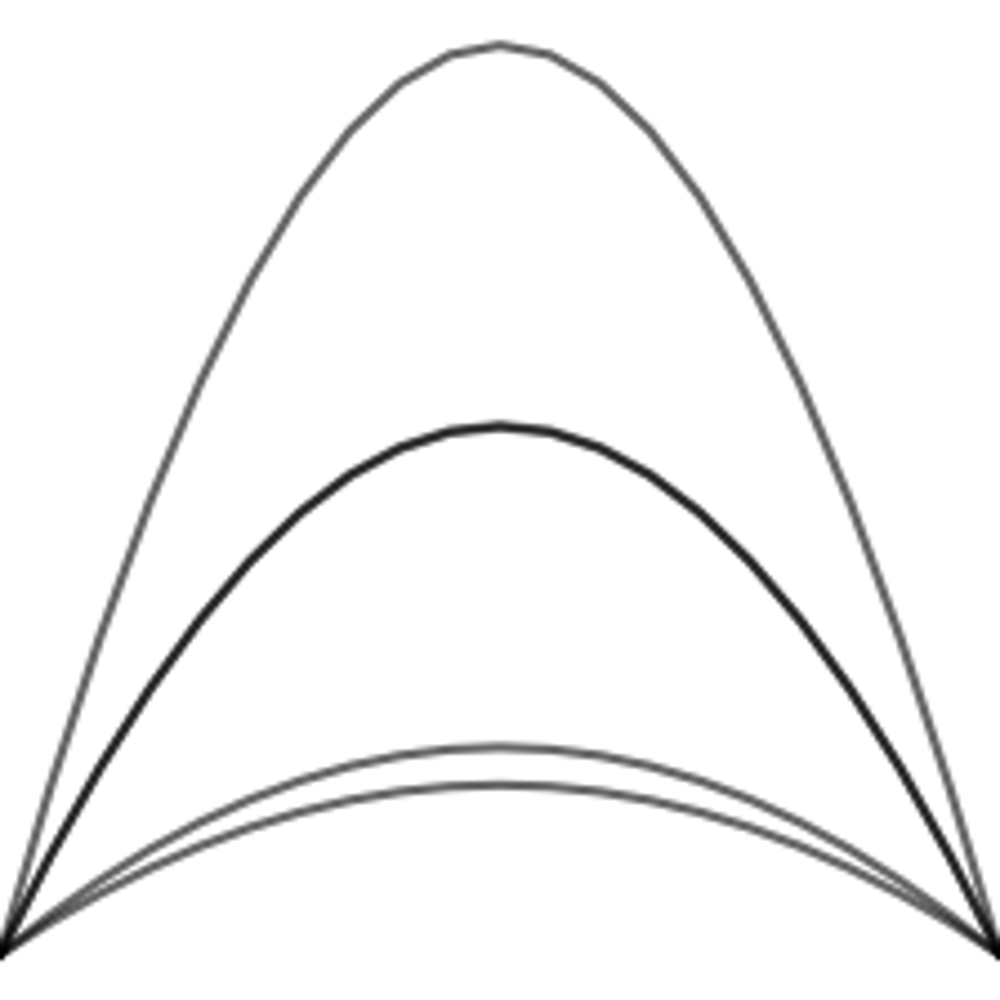}
        \caption{Linear Non-Homogeneous System}
    \end{subfigure}
    \vspace{0.1cm}

    \begin{subfigure}{\linewidth}
        \centering
        \includegraphics{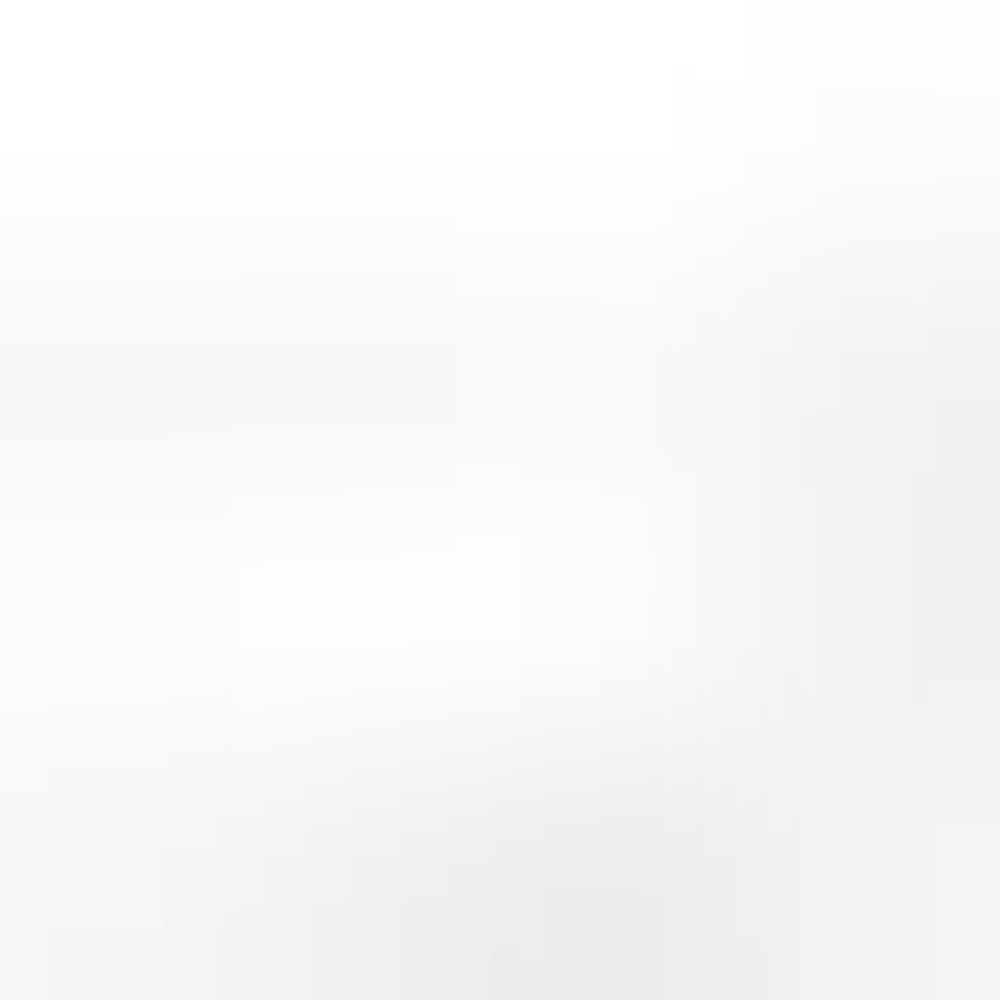}\hfill
        \includegraphics{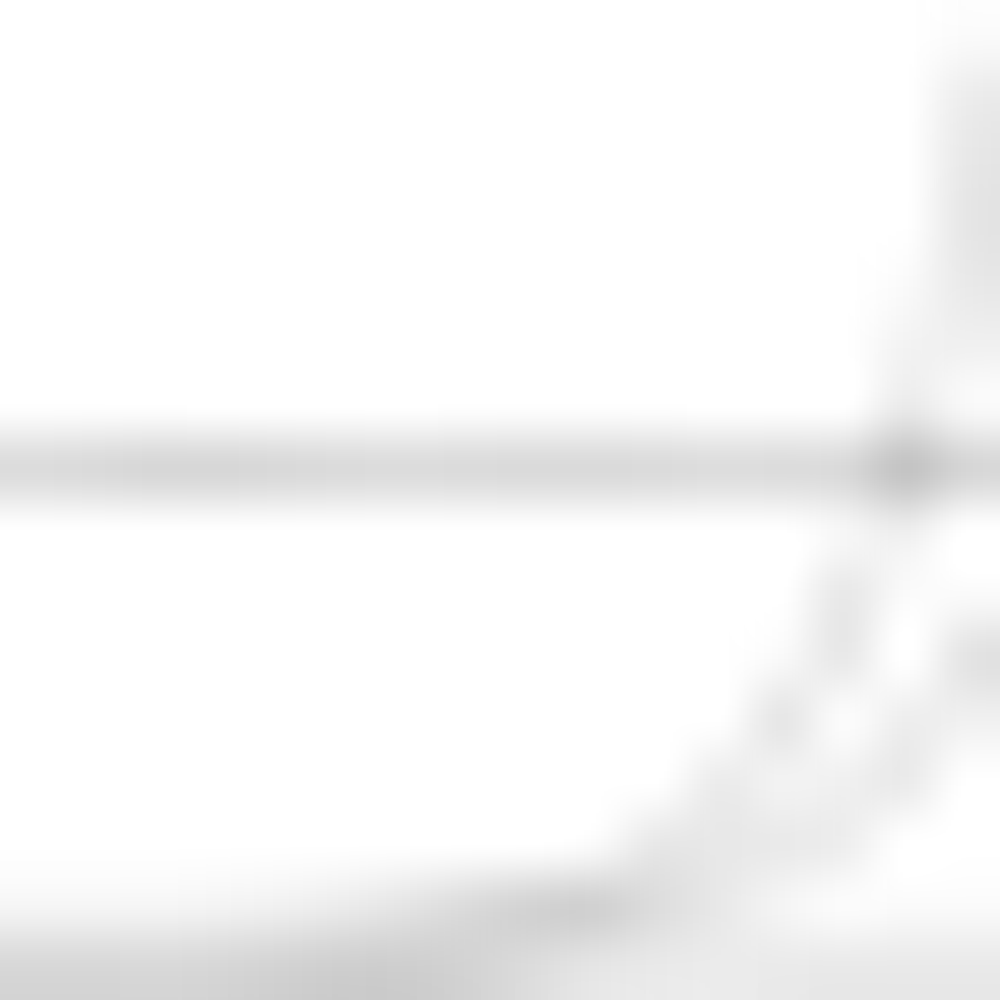}\hfill
        \includegraphics{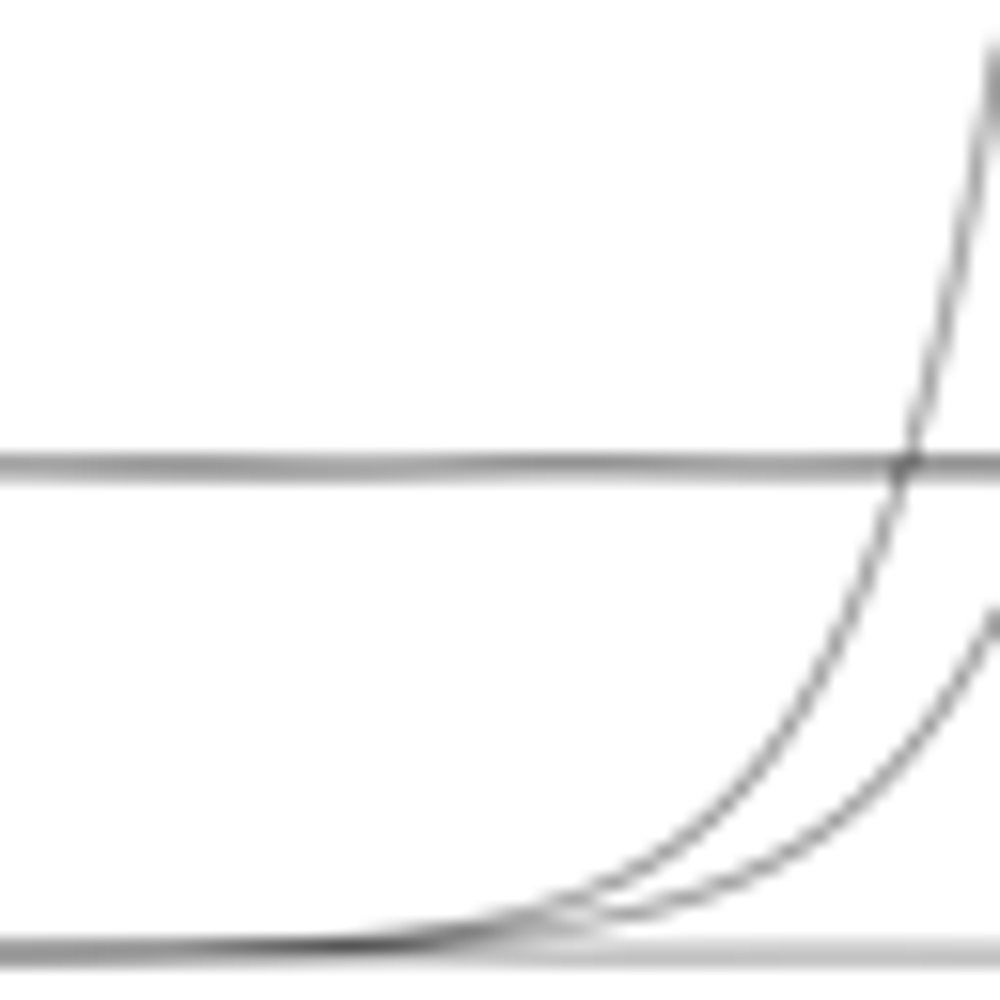}\hfill
        \includegraphics{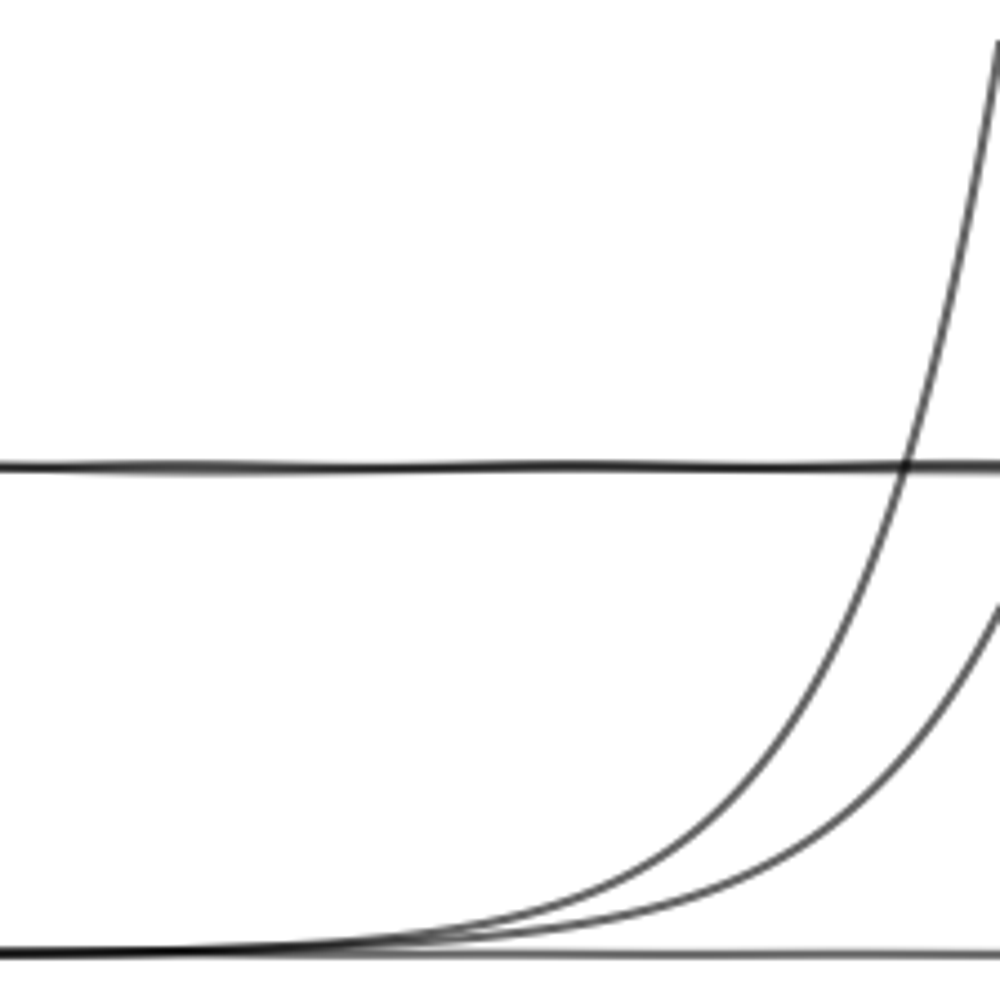}
        \caption{Linear Homogeneous System}
    \end{subfigure}
    \vspace{0.1cm}

    \begin{subfigure}{\linewidth}
        \centering
        \includegraphics{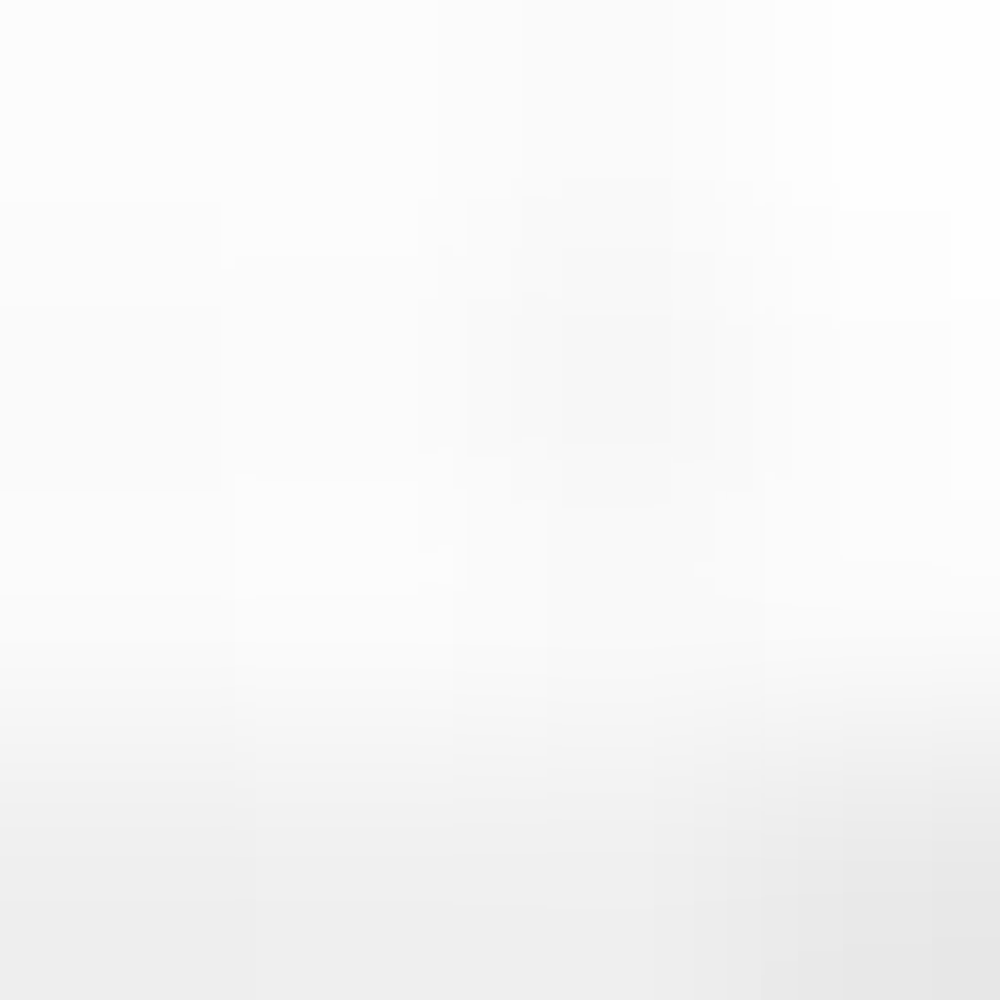}\hfill
        \includegraphics{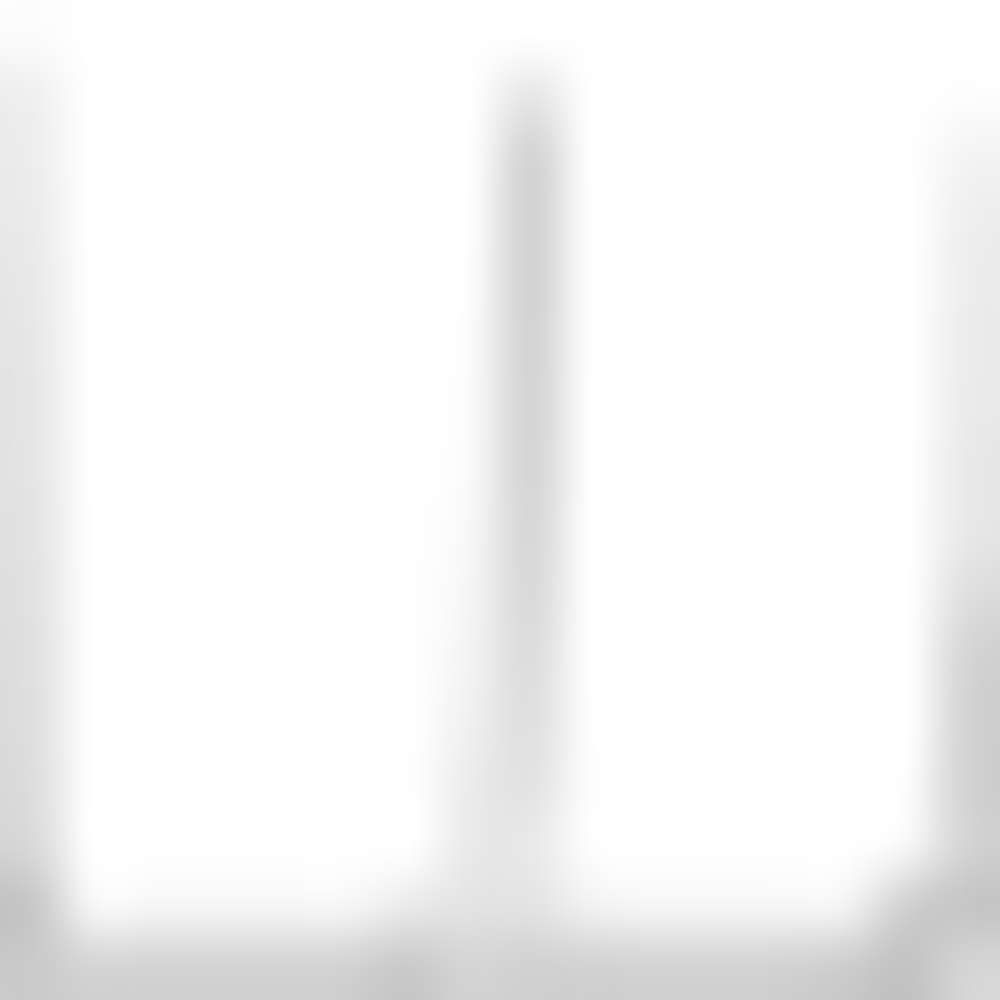}\hfill
        \includegraphics{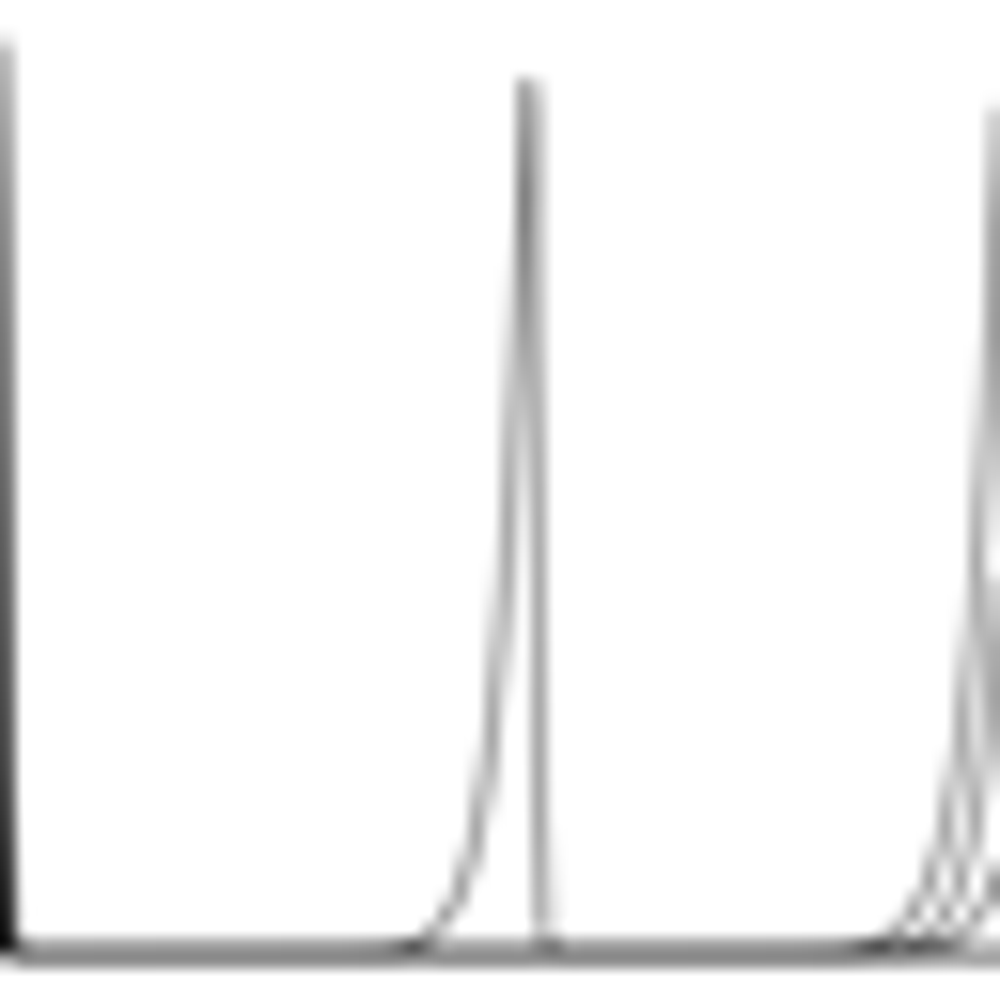}\hfill
        \includegraphics{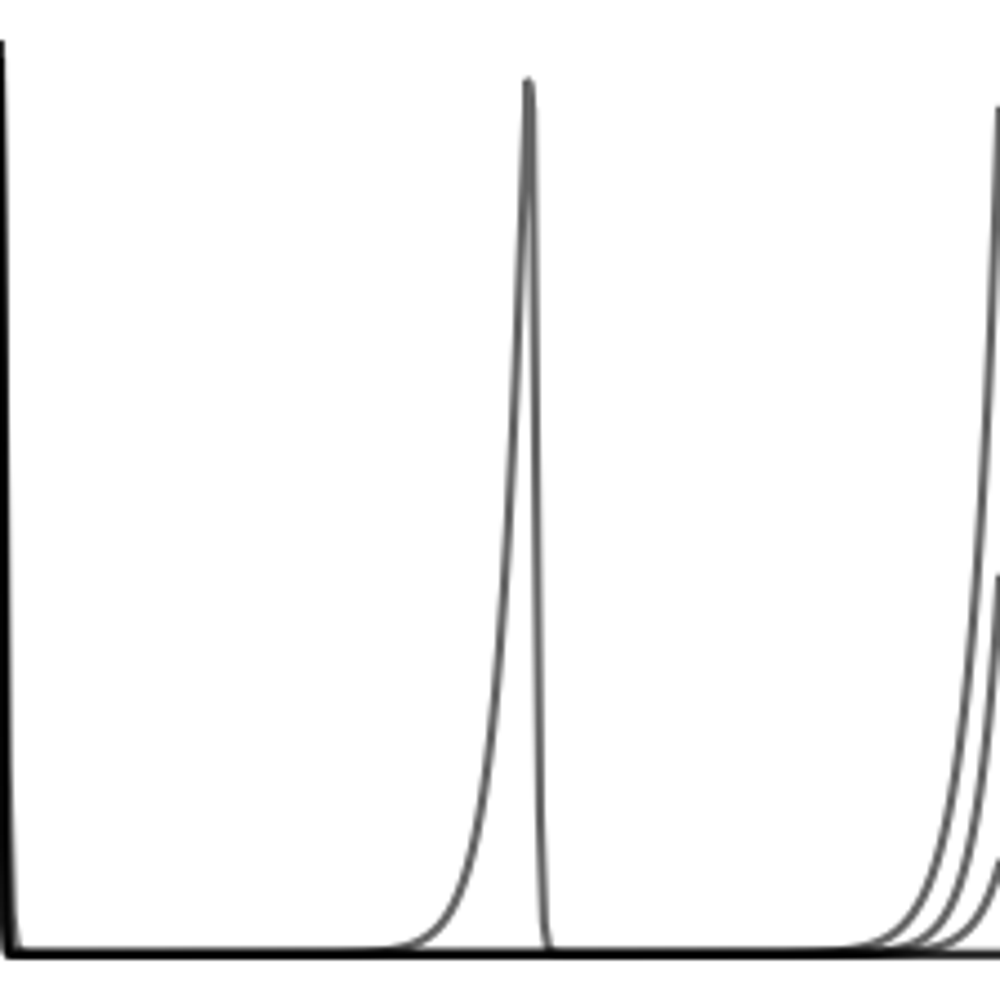}
        \caption{Nonlinear Homogeneous System }
    \end{subfigure}
    
    \vspace{0.1cm}
    \begin{subfigure}{\linewidth}
        \centering
        \includegraphics{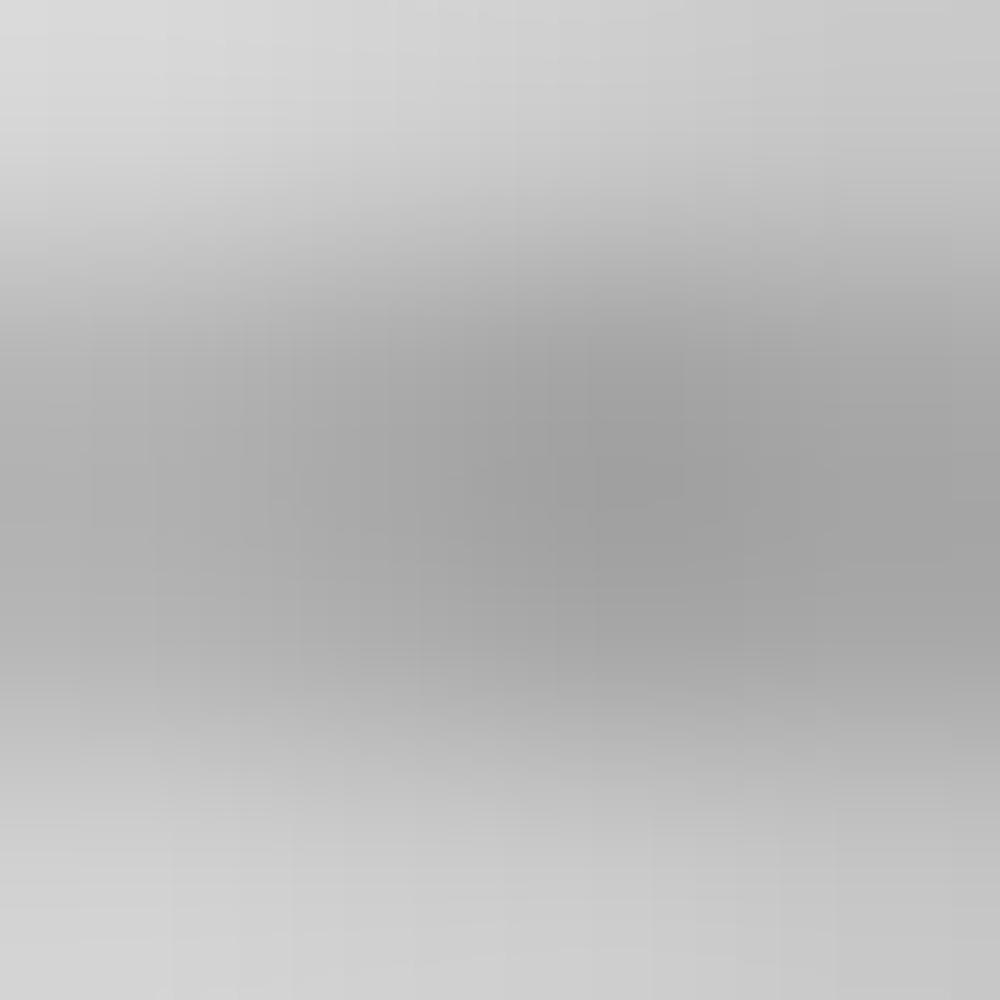}\hfill
        \includegraphics{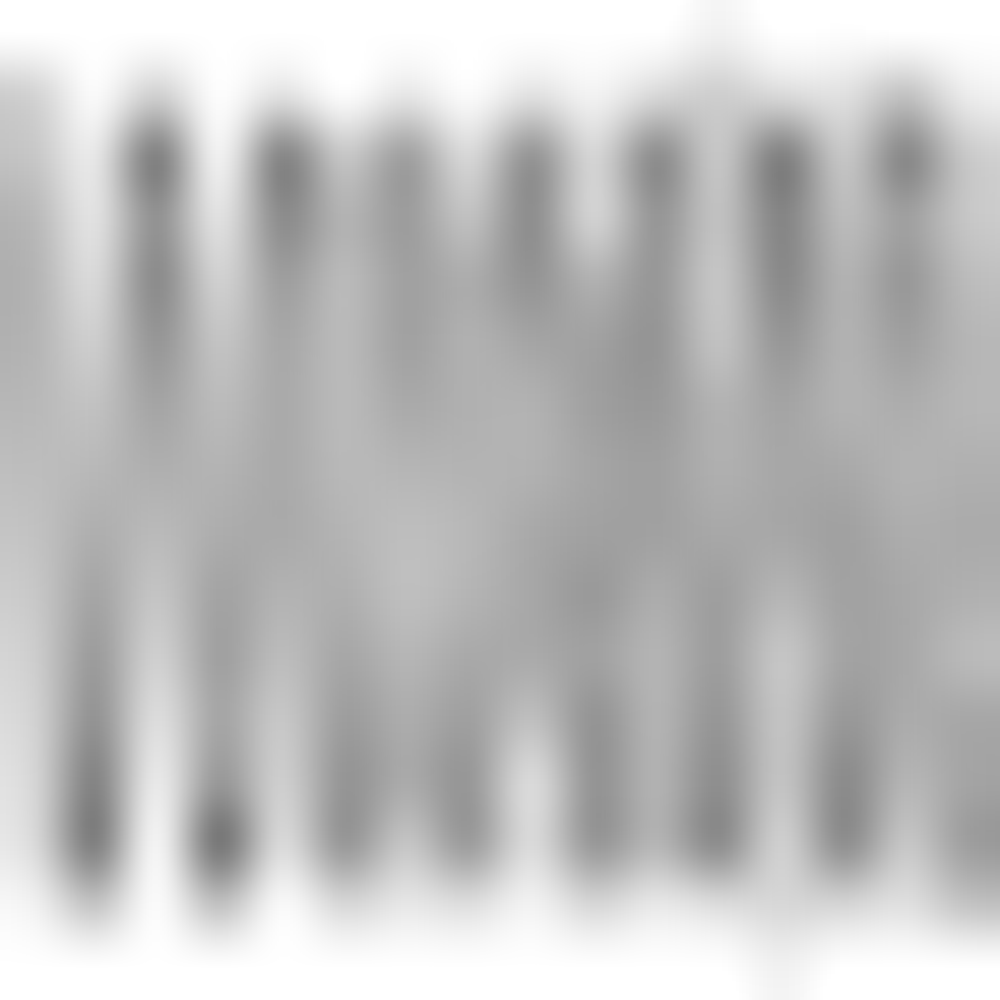}\hfill
        \includegraphics{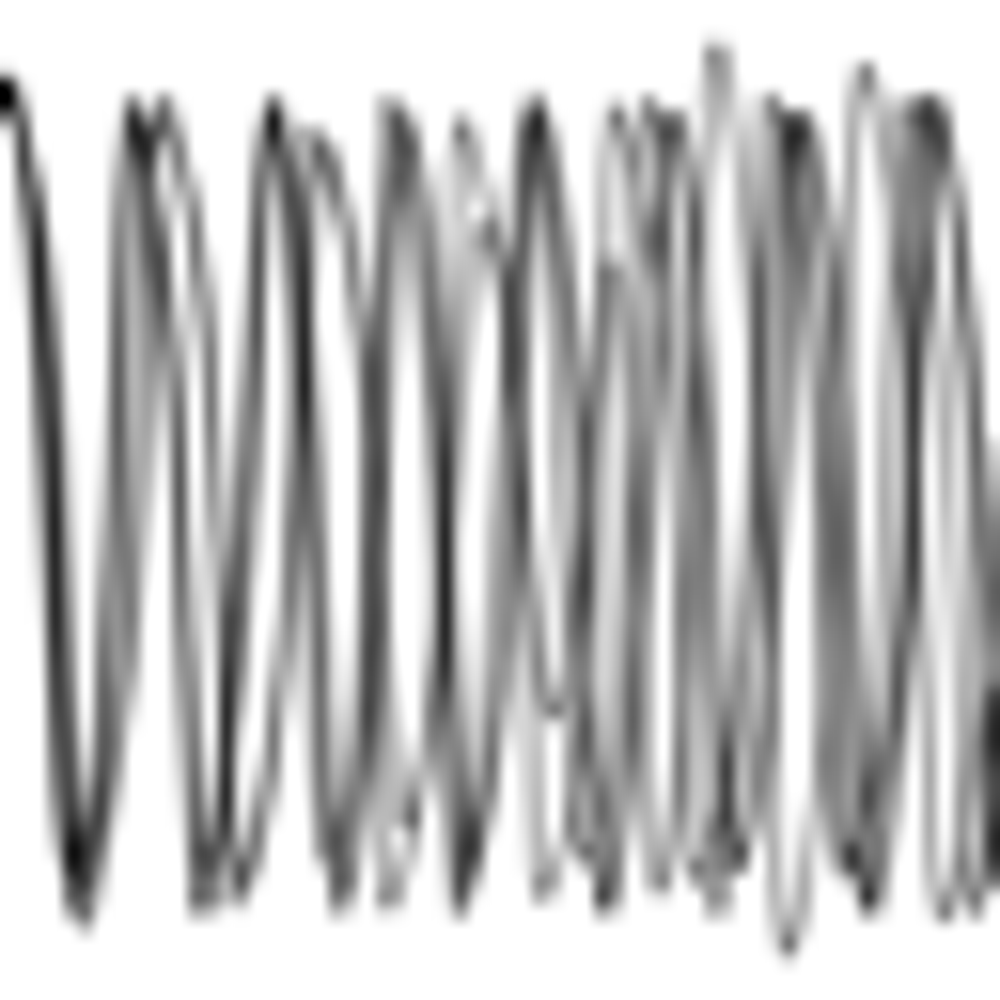}\hfill
        \includegraphics{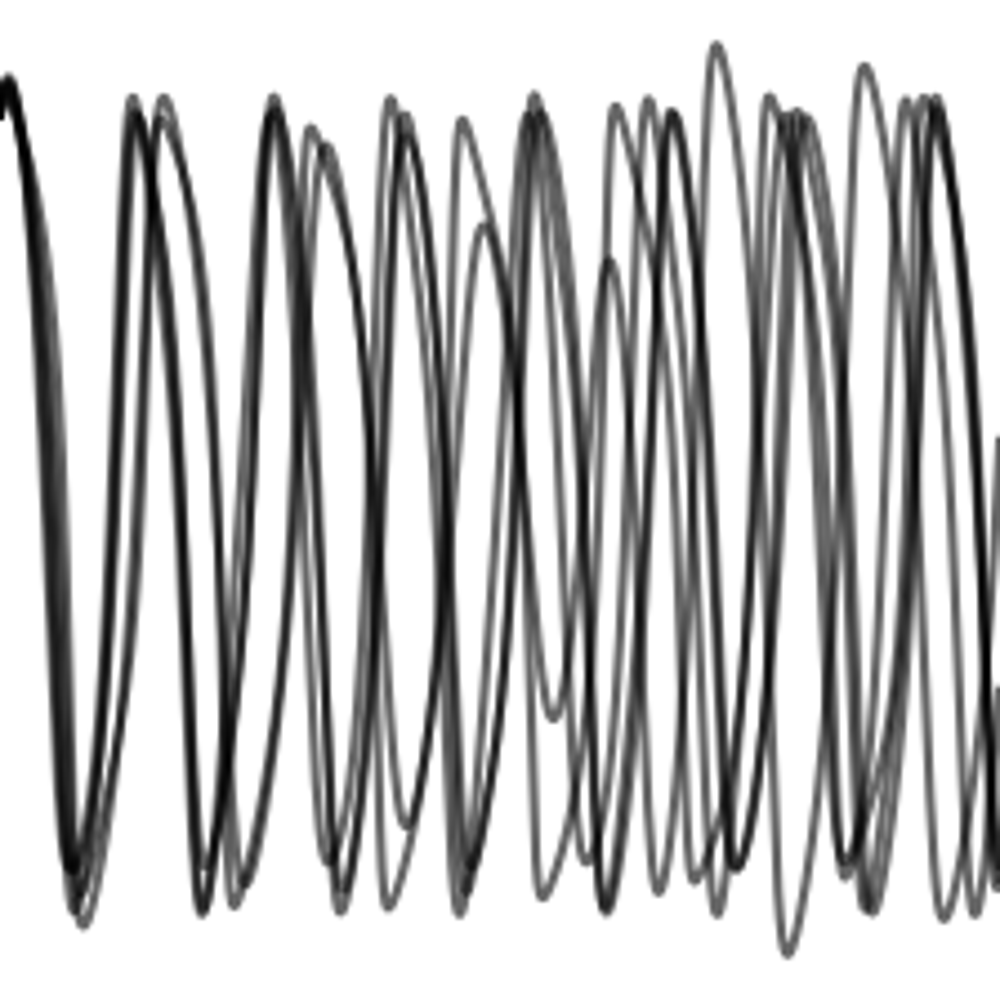}
        \caption{Nonlinear Non-Homogeneous System }
    \end{subfigure}
    
    \vspace{0.1cm}
    
    \begin{subfigure}{\linewidth}
        \centering
        \includegraphics{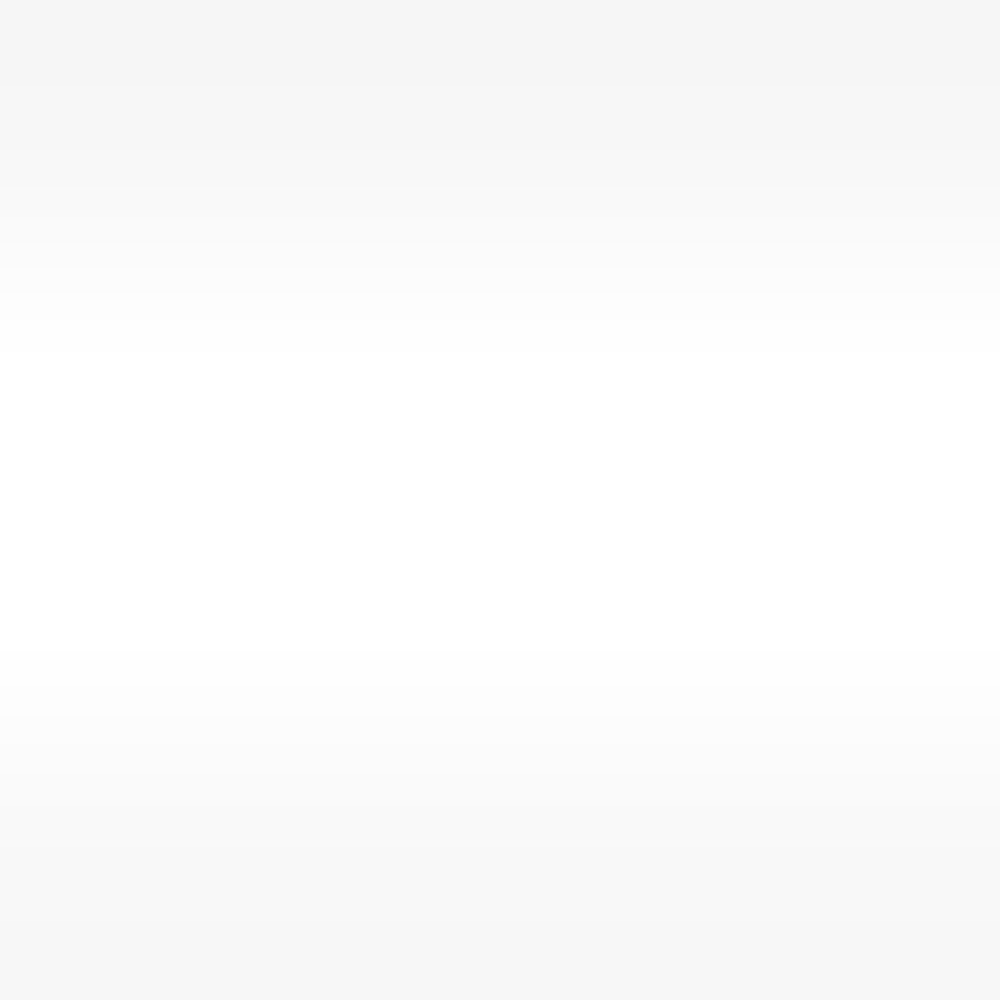}\hfill
        \includegraphics{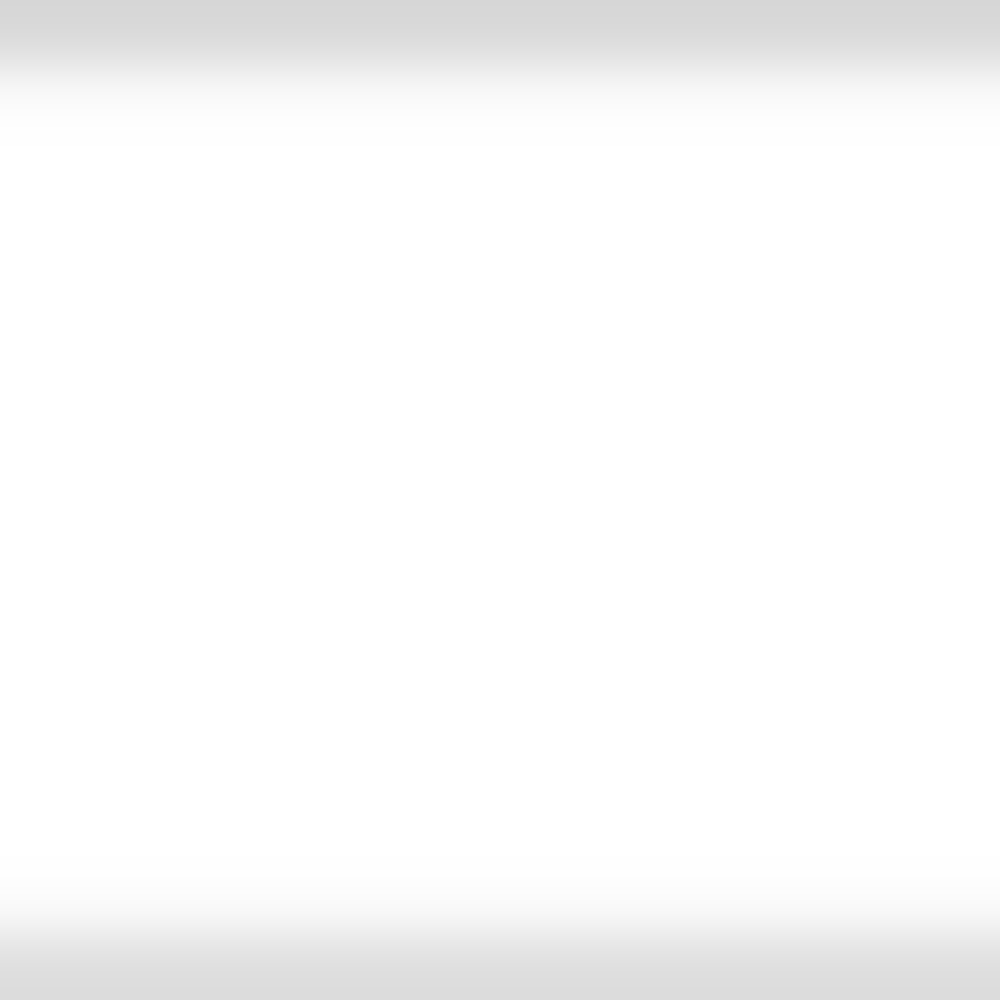}\hfill
        \includegraphics{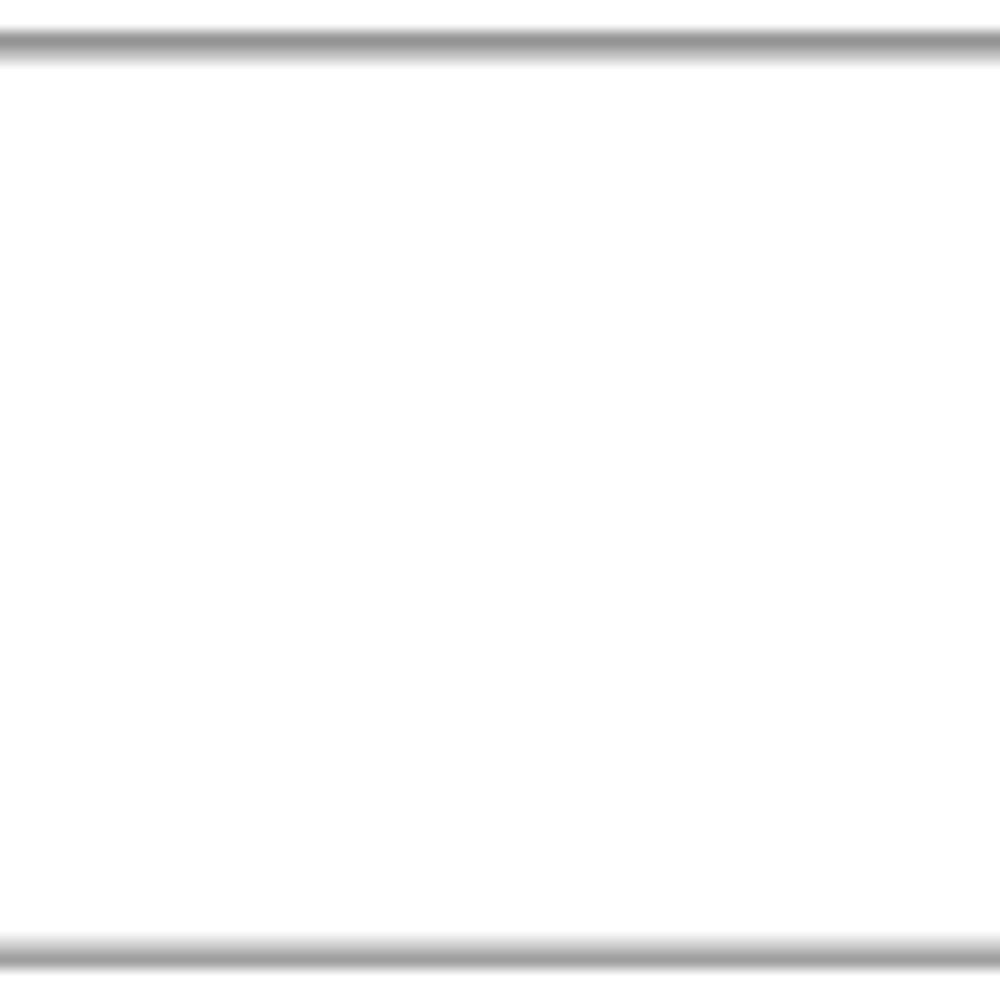}\hfill
        \includegraphics{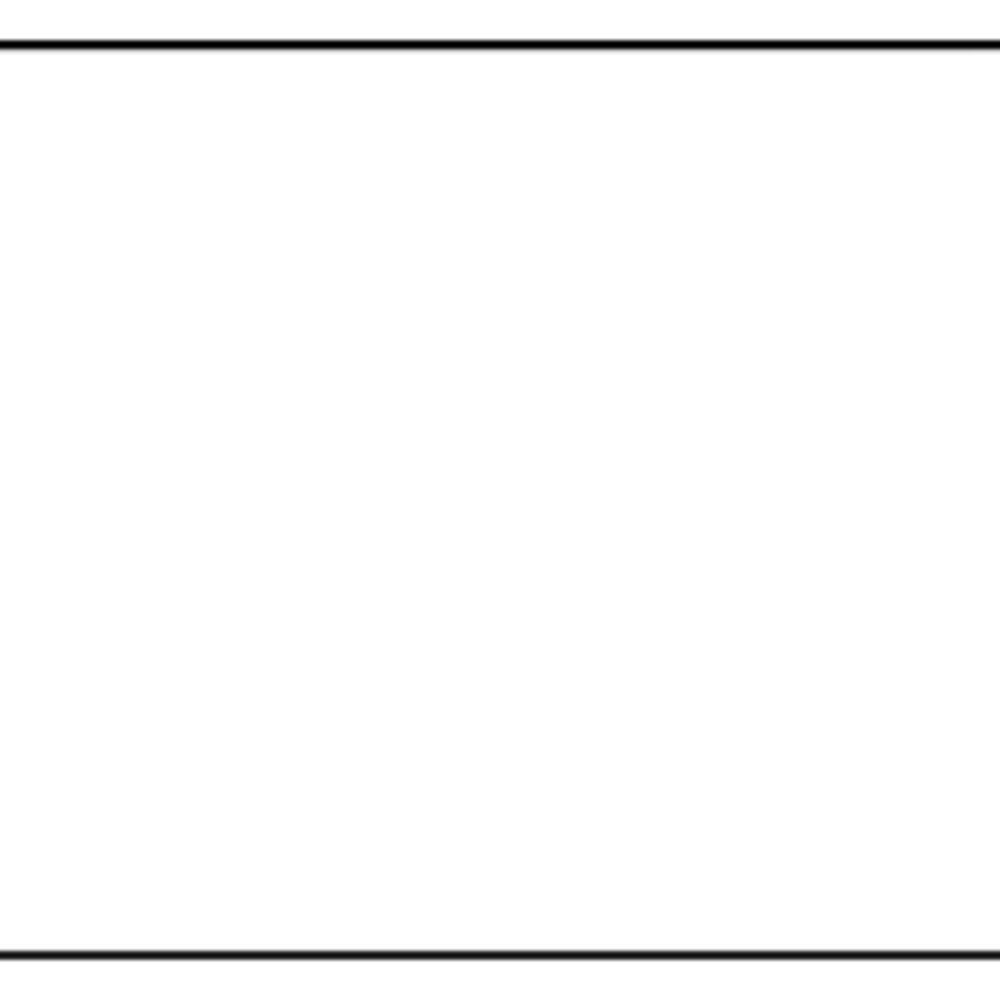}
        \caption{Nonlinear Boundary Value Problem}
    \end{subfigure}
    \caption{Visual comparison of trajectory structures across all systems in ODE-6. Each plot overlays 5 randomly selected trajectories. Due to path overlaps, fewer than 5 distinct lines may be visible in some panels. Rows represent different dynamical systems, and columns illustrate the effect of increasing resolution $S$ from 4 (left) to 224 (right).}
    \label{fig:class_comparison_all}
\end{figure*}

\begin{figure}[h]
    \centering
    \begin{subfigure}[b]{0.32\linewidth}
        \centering
        \includegraphics[width=\linewidth, height=0.2\textheight]{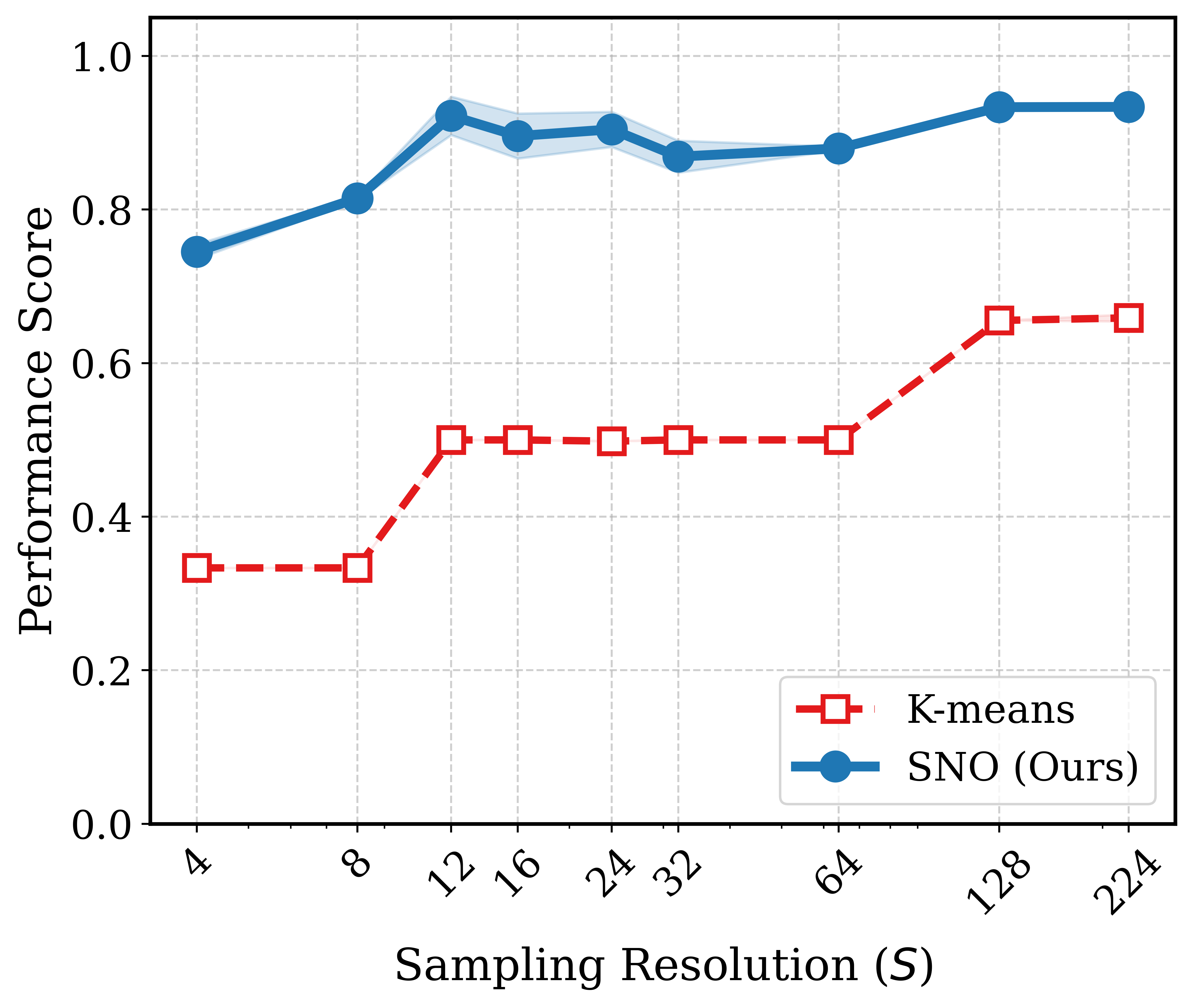}
        \caption{Clustering Accuracy}
        \label{fig:c-acc}
    \end{subfigure}
    \hfill
    \begin{subfigure}[b]{0.32\linewidth}
        \centering
        \includegraphics[width=\linewidth, height=0.2\textheight]{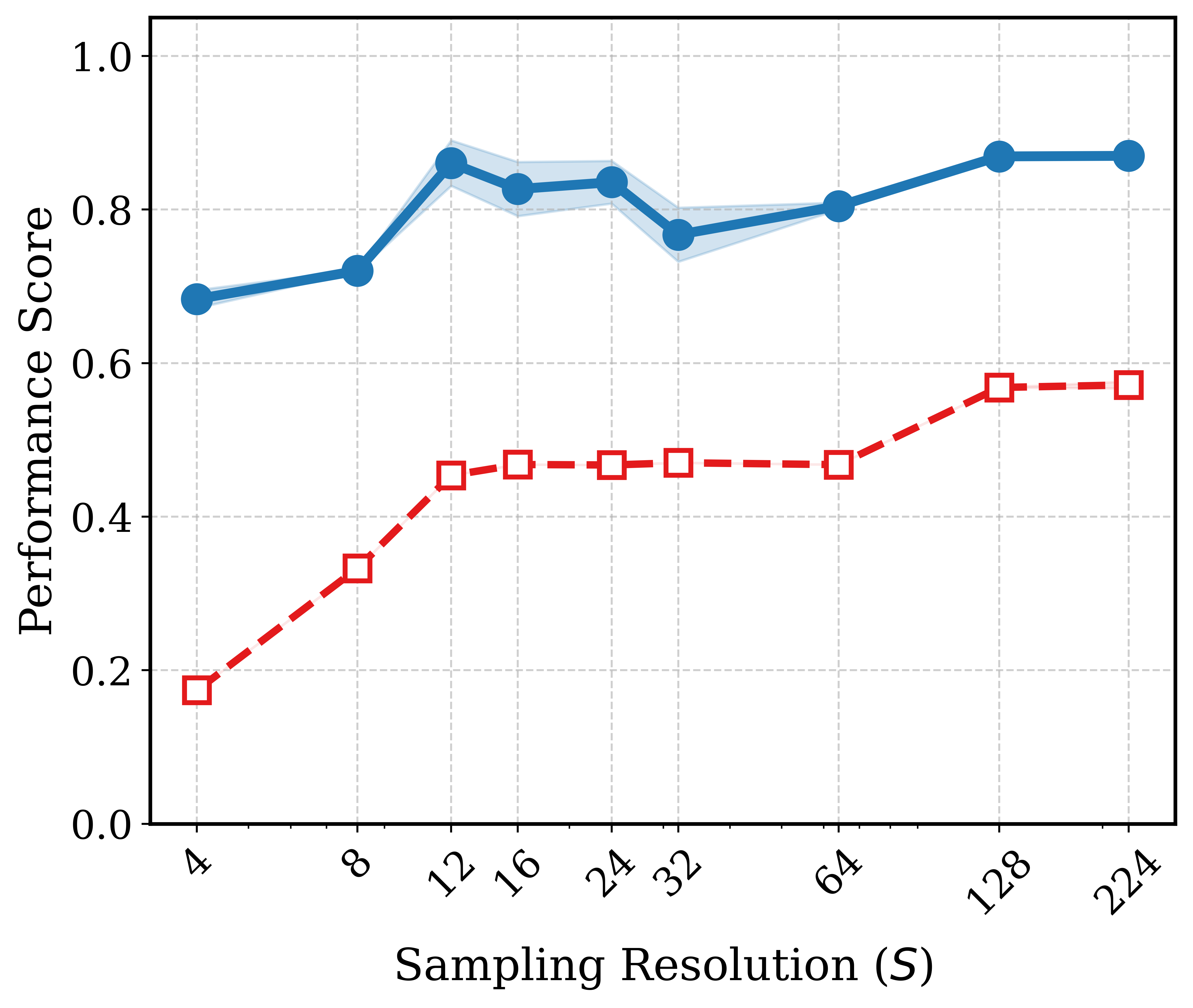}
        \caption{Clustering ARI}
        \label{fig:c-ari}
    \end{subfigure}
     \hfill
    \begin{subfigure}[b]{0.32\linewidth}
        \centering
        \includegraphics[width=\linewidth, height=0.2\textheight]{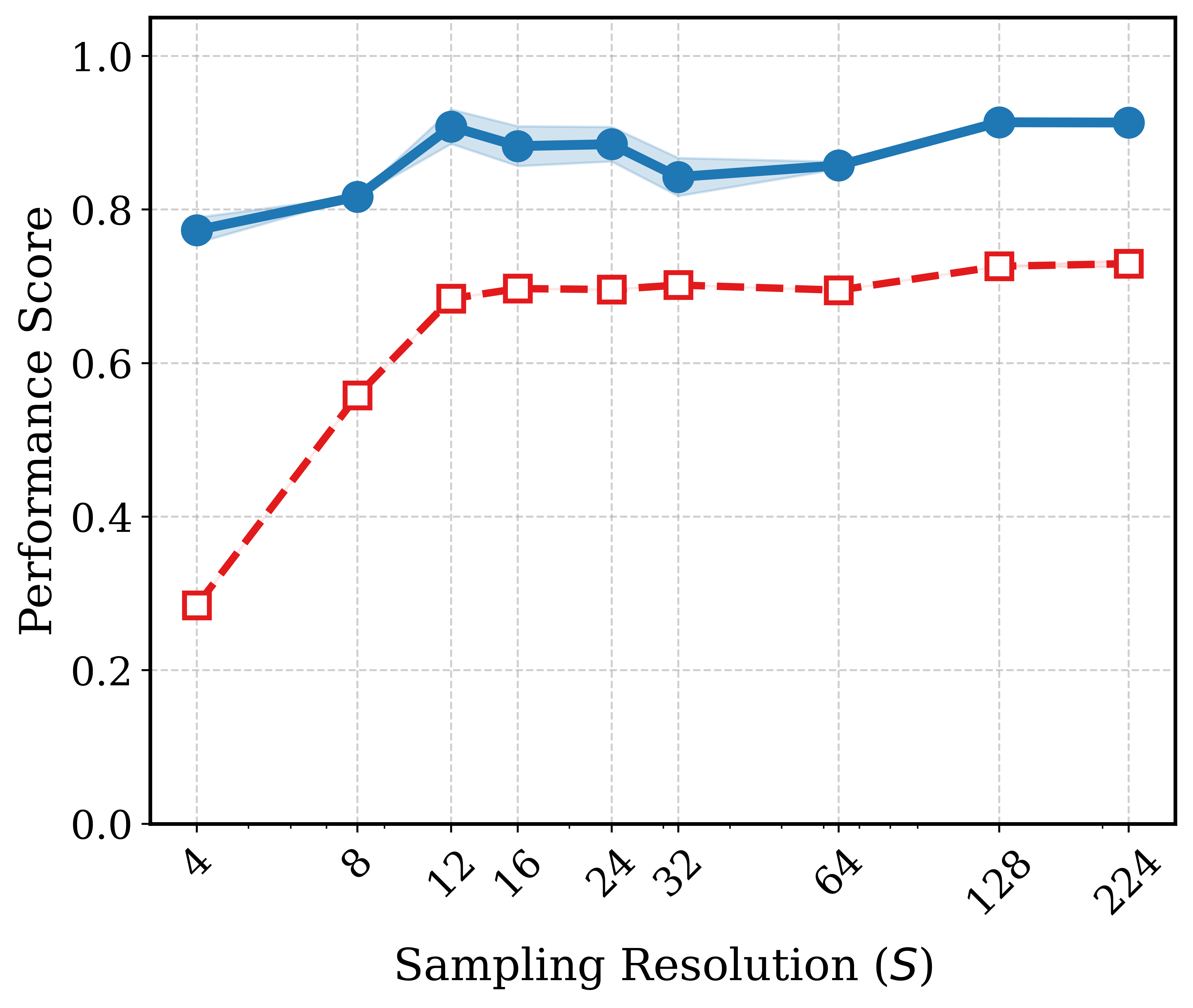}
        \caption{Clustering NMI}
        \label{fig:c-nmi}
    \end{subfigure}

    \caption{
    We evaluate the SNO model with 5 random seeds across varying sampling resolutions, ranging from a coarse $2 \times 2$ to $224 \times 224$ on the ODE-6 test set. The improvement of clustering metrics empirically demonstrates that the sequence of discretized operators converges toward a stable limit as the sampling density increases in this structured regime. This corroborates our claim that SNO approximates the true cluster partition.
}
    \label{fig:convergence}
\end{figure}

As illustrated in Figure~\ref{fig:class_comparison_all} and Figure~\ref{fig:convergence}, we test the model's performance on the ODE-6 test set with varying sampling resolutions to validate the consistency in the structured regime. We observe an increasing trend across all metrics as the resolution increases. Even at coarse resolutions, the model achieves non-trivial performance, indicating that the operator learns to capture global structures rather than relying solely on high-frequency pixel details. As the sampling rate increases, performance improves and eventually saturates, empirically approximating the consistency of the learned sampling-based operator for smooth dynamical flows.

\subsection{Visualization}
We provide visualizations to support our experiments.
\begin{figure}[!htbp]
    \centering
    \begin{subfigure}[b]{0.48\linewidth}
        \centering
        \includegraphics[width=\linewidth]{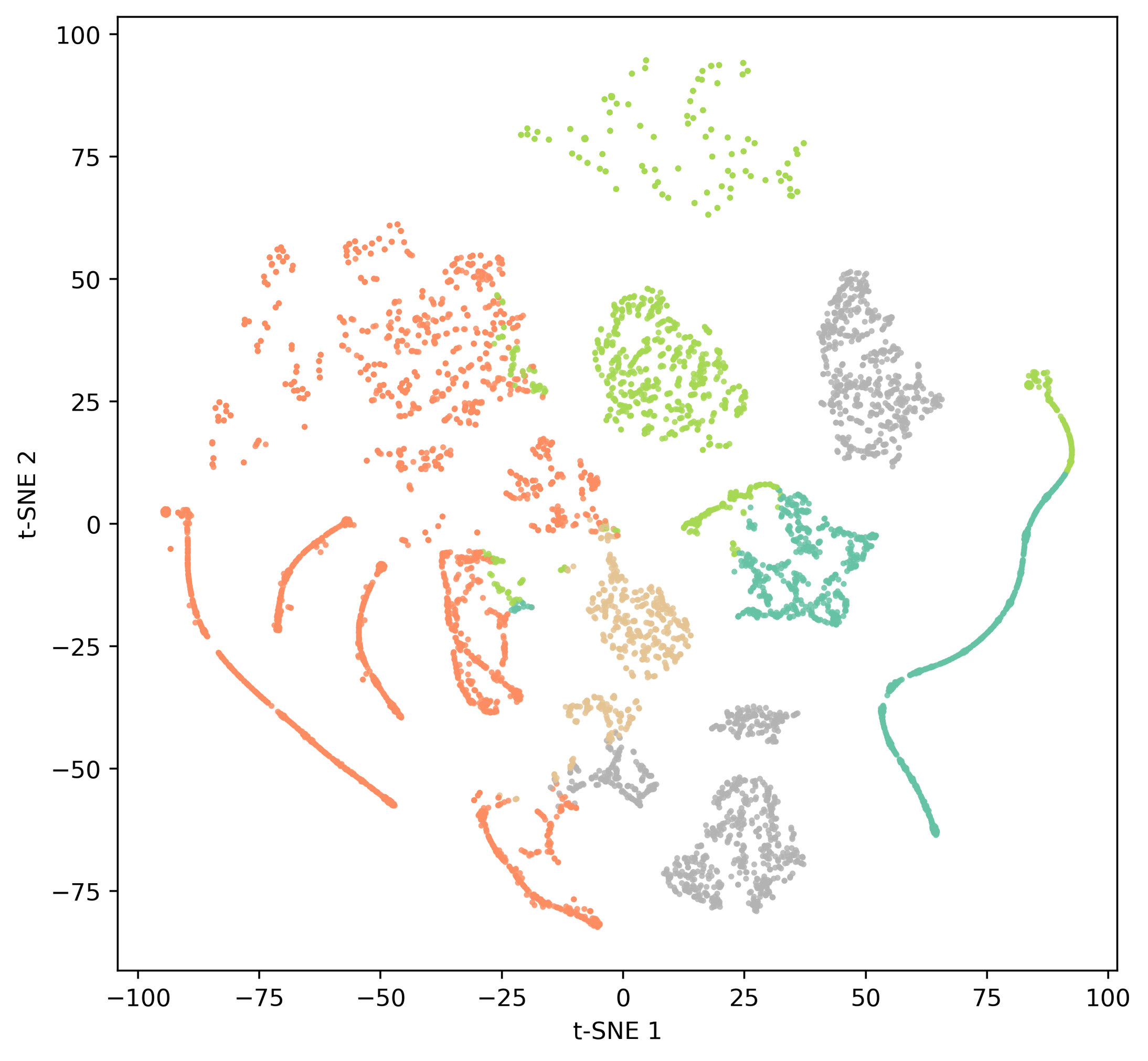}
        \caption{Baseline (CLIP + K-means)}
        \label{fig:fdc_a}
    \end{subfigure}
    \hfill
    \begin{subfigure}[b]{0.48\linewidth}
        \centering
        \includegraphics[width=\linewidth]{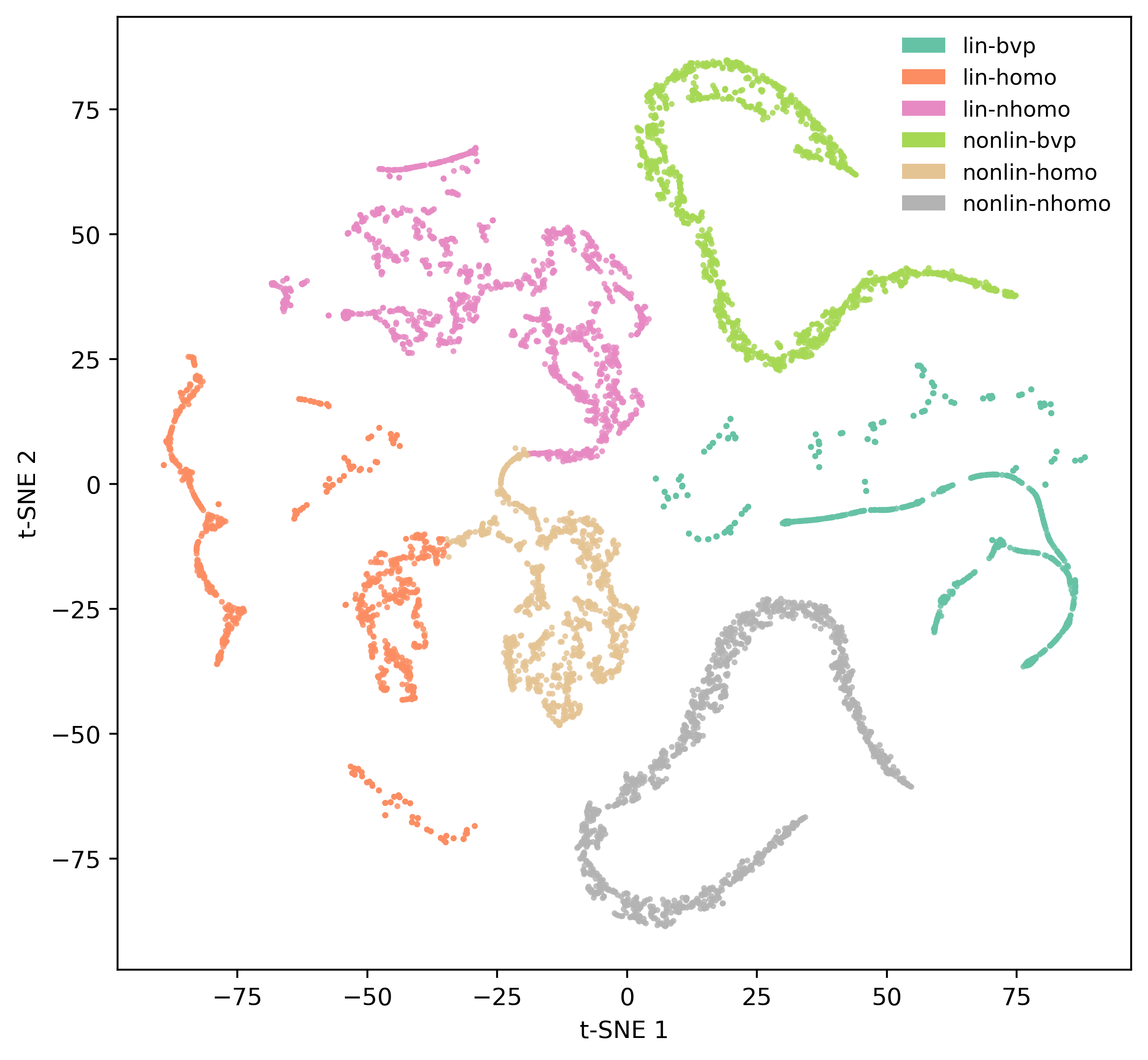}
        \caption{SNO (Learned Logits)}
        \label{fig:fdc_b}
    \end{subfigure}
    \caption{t-SNE visualization of the representation space for ODE-6. Both embeddings exhibit curve-like structures due to continuous dependence on initial conditions, but in (a) different dynamical systems remain entangled, leading to poor separability. (b) Our method successfully disentangles these manifolds, preserving the intrinsic continuous geometry within clusters while pushing distinct systems apart.
    \hfill\\
    Note that the clusters can consist of multiple non-convex components as highlighted in Figure~\ref{fig:Kmeans_ours}.
    This is a key feature of the NO-based clustering pipeline which cannot be performed by classical K-means-type clustering procedures.
    }
    \label{fig:fdc_tsne}
\end{figure}

The sub-optimality of the K-means baseline (Fig.~\ref{fig:fdc_tsne}a) suggests a theoretical insight: while the frozen encoder $\phi$ preserves necessary topological information from the sampling grid, the intrinsic functional similarity in $\mathcal{H}$ is not naturally aligned with the Euclidean metric in the pre-trained embedding space. 
Consequently, relying solely on a fixed metric is insufficient for separating complex dynamical flows.
In contrast, our proposed method (Fig.~\ref{fig:fdc_tsne}b) succeeds by explicitly parameterizing the mapping via the learnable projection head $g$ with clustering objectives. 
This empirically validates the existence of  an SNO that allows for the approximation of the true partition. 
It demonstrates that by composing the fixed feature map with a trainable layer, the SNO effectively disentangles the latent dynamical manifold into coherent clusters, which empirically underscores the premise of Theorem~\ref {thrm:MainKuratowski}.

The qualitative overlays in Figure~\ref{fig:fdc_overlay_ODE6} further confirm that the learned clusters correspond to distinct geometric and spectral patterns. The ODE-6 dataset consists of structurally distinct classes, which induce stable patterns effectively captured by the proposed operator-based clustering.

For the ODE-4 dataset (Figure~\ref{fig:fdc_overlay_ode4}), the visual representations exhibit substantially higher intra-class variability due to the randomized neural parameterization. Despite this difficulty, SNO yields more coherent groupings than K-means and classical methods, indicating that meaningful dynamical signatures are partially recovered, in line with the universality guarantees established in Section~\ref{s:UniversalClustering}. Thus, the application on clustering ODE trajectories empirically supports our theory. 

\begin{figure}[!htbp]
    \centering

    \begin{subfigure}[b]{0.495\linewidth}
        \centering
        \includegraphics[width=\linewidth]{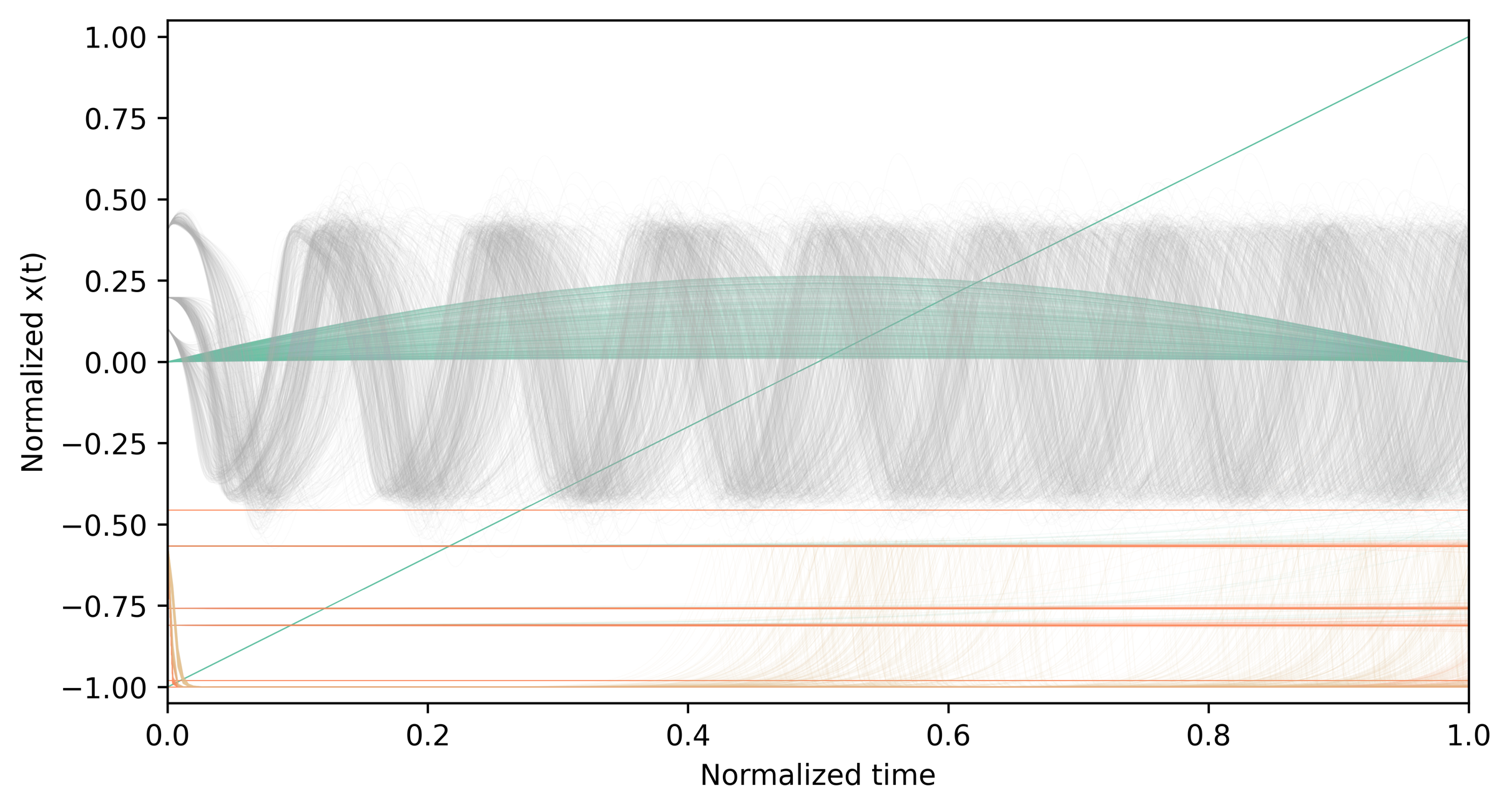}
        \caption{K-means (ACC = 62.4\%)}
        \label{fig:fdc_kmeans}
    \end{subfigure}
    \hfill
    \begin{subfigure}[b]{0.495\linewidth}
        \centering
        \includegraphics[width=\linewidth]{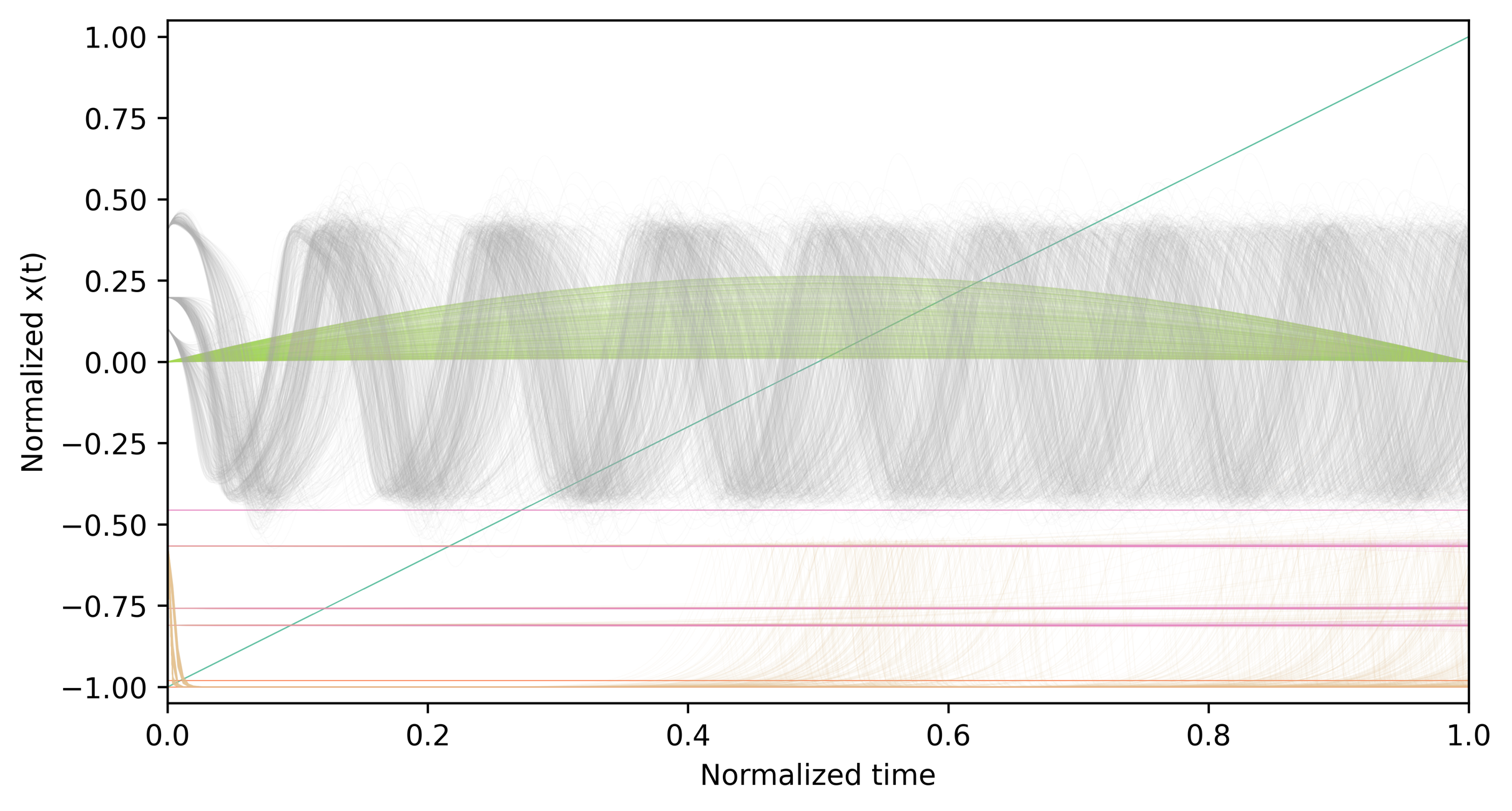}
        \caption{SNO (ACC = 93.3\%)}
        \label{fig:fdc_scp}
    \end{subfigure}
    \vspace{0.5em} 
    \begin{subfigure}[b]{0.495\linewidth}
        \centering
        \includegraphics[width=\linewidth]{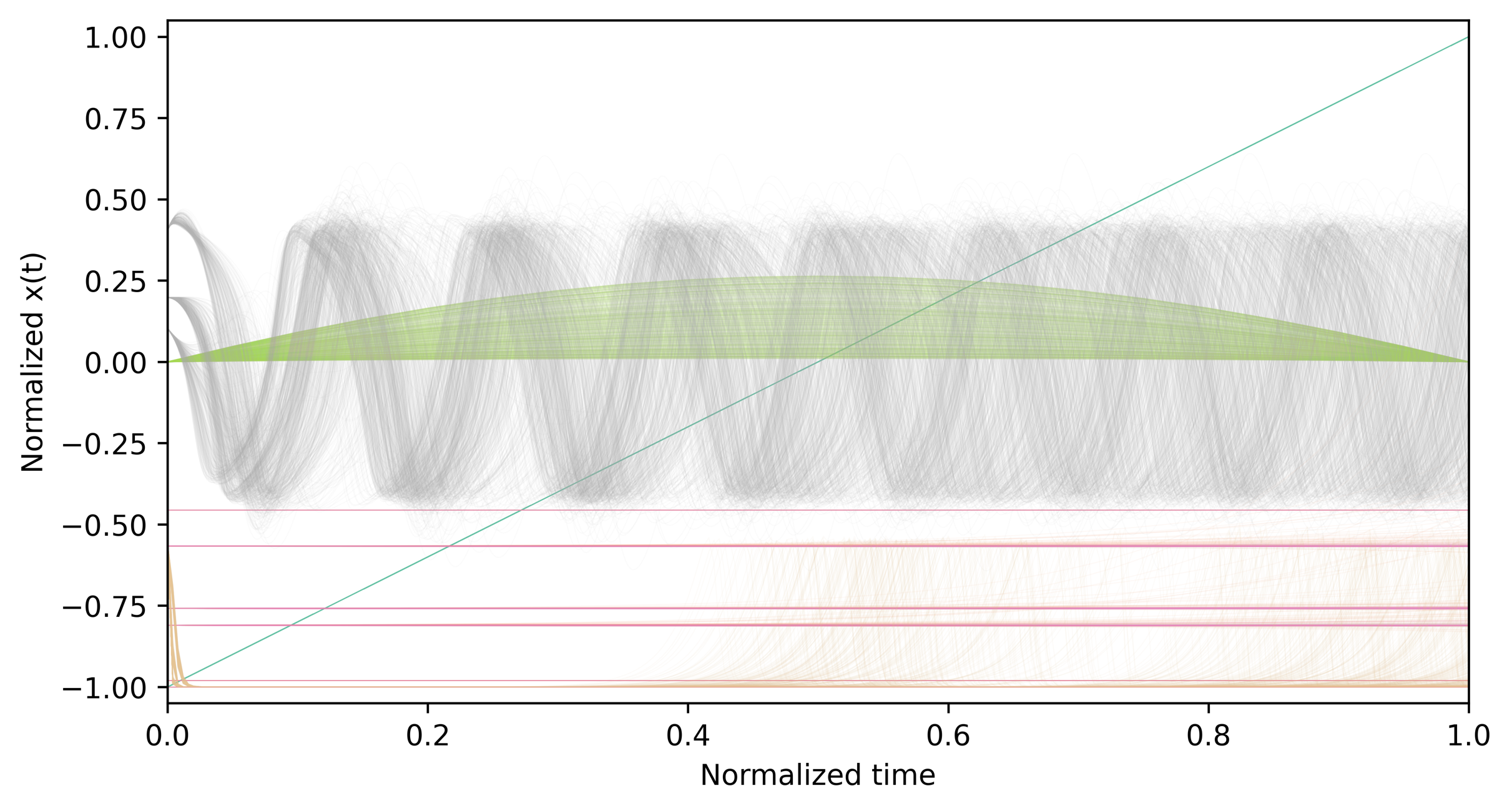}
        \caption{SNO + spectrogram (ACC = 94.5\%)}
        \label{fig:fdc_scp_spec}
    \end{subfigure}
    \hfill
    \begin{subfigure}[b]{0.495\linewidth}
        \centering
        \includegraphics[width=\linewidth]{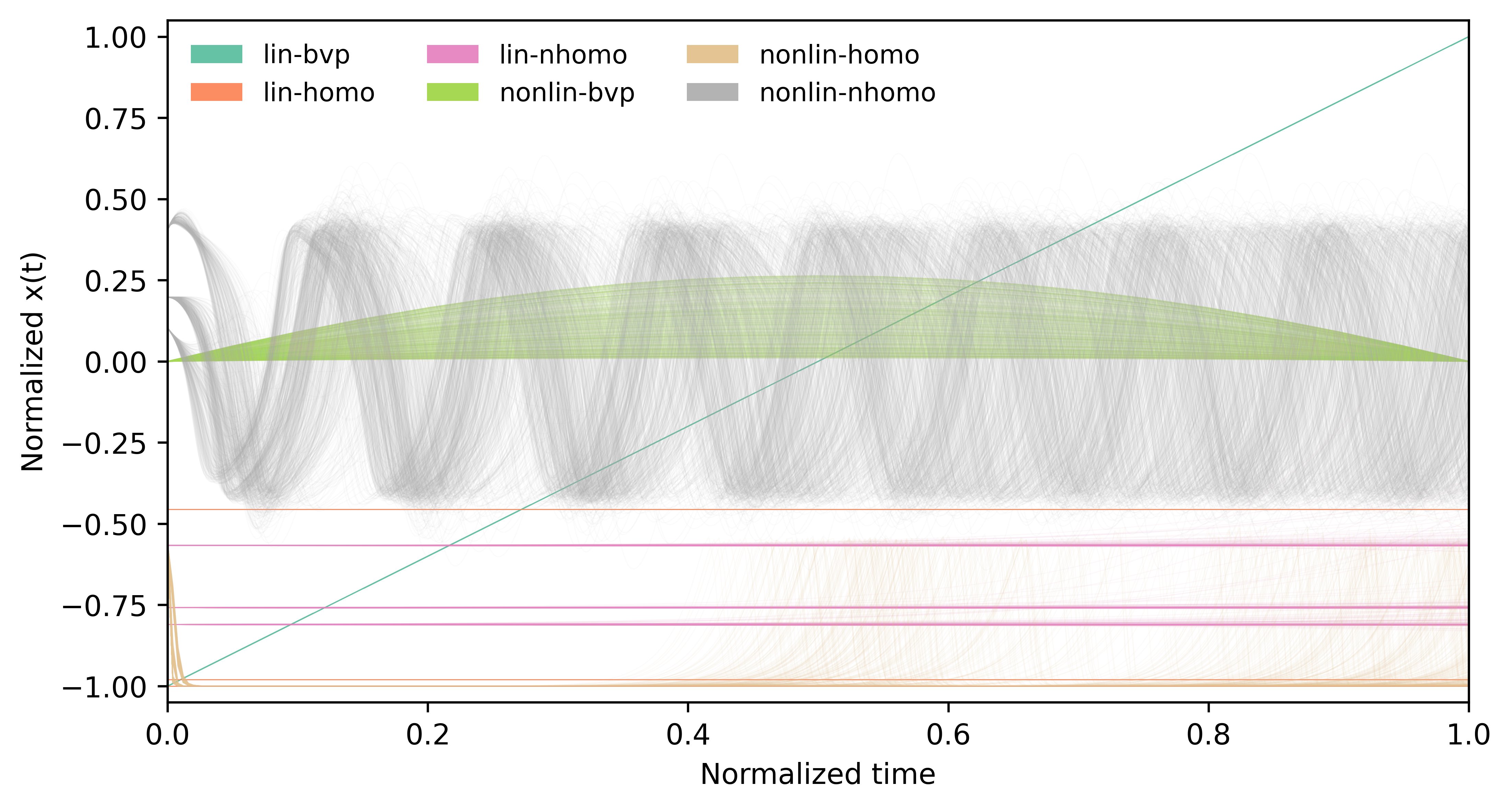}
        \caption{Ground-truth clusters}
        \label{fig:fdc_true}
    \end{subfigure}

    \caption{Qualitative comparison of clustering on the ODE-6 test set. 
    (a) K-means (ACC=62.4\%) fails to distinguish between structurally similar trajectories, resulting in significant overlap across different dynamical families. 
    (b-c) In contrast, the proposed SNO (ACC=93.3\%) and SNO+Spectrogram (ACC=94.5\%) successfully disentangle these latent structures, yielding clean partitions that closely align with the 
    (d) Ground-truth class labels. 
    The color corresponds to the six distinct dynamical systems described in Section~\ref{sec:synthetic}.
    }
    \label{fig:fdc_overlay_ODE6}
\end{figure}

\begin{figure}[!htbp]
    \centering

    \begin{subfigure}[b]{0.495\linewidth}
        \centering
        \includegraphics[width=\linewidth]{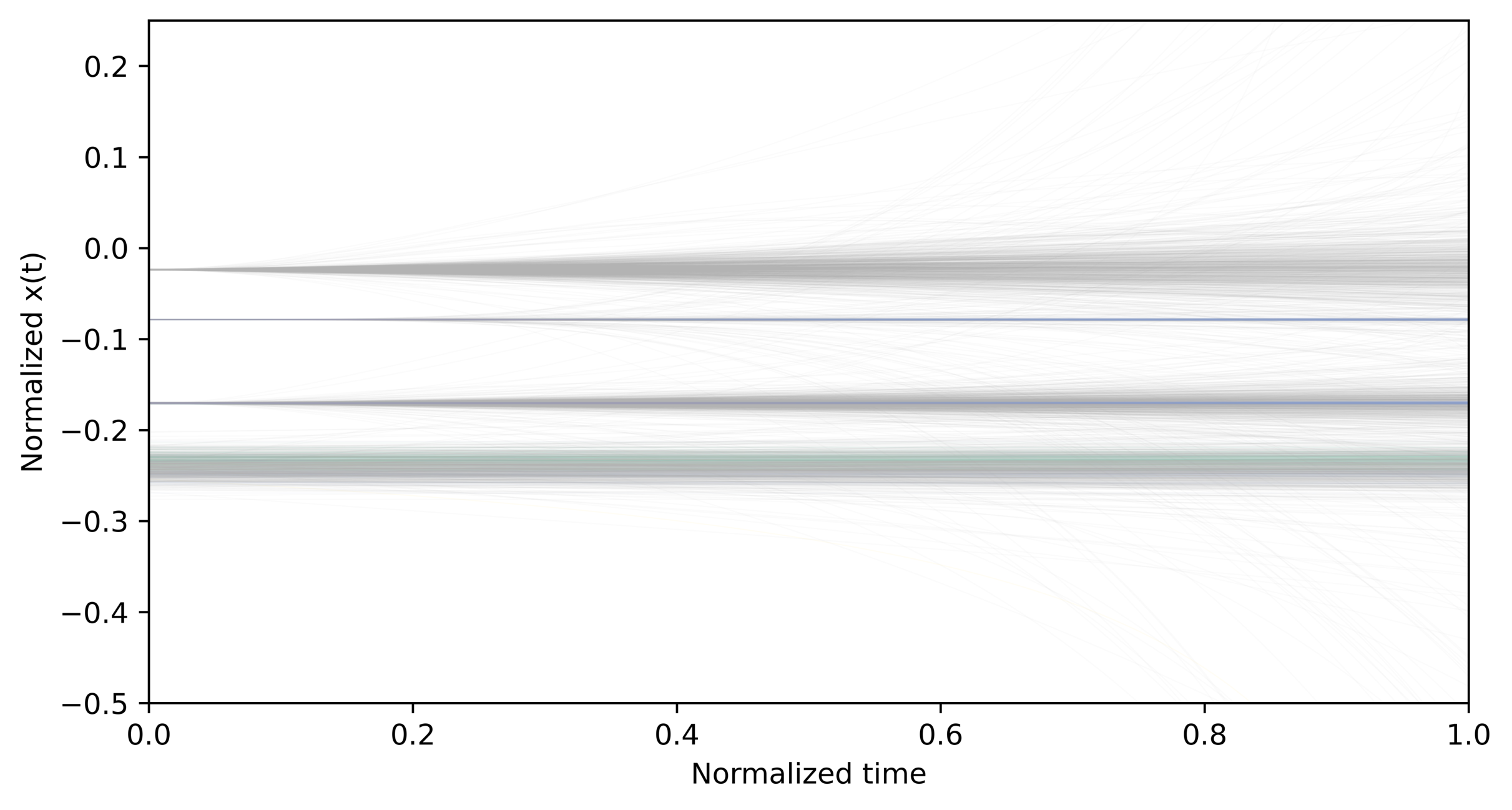}
        \caption{K-means (ACC = 49.1\%)}
        \label{fig:fdc_kmeans_4}
    \end{subfigure}
    \hfill
    \begin{subfigure}[b]{0.495\linewidth}
        \centering
        \includegraphics[width=\linewidth]{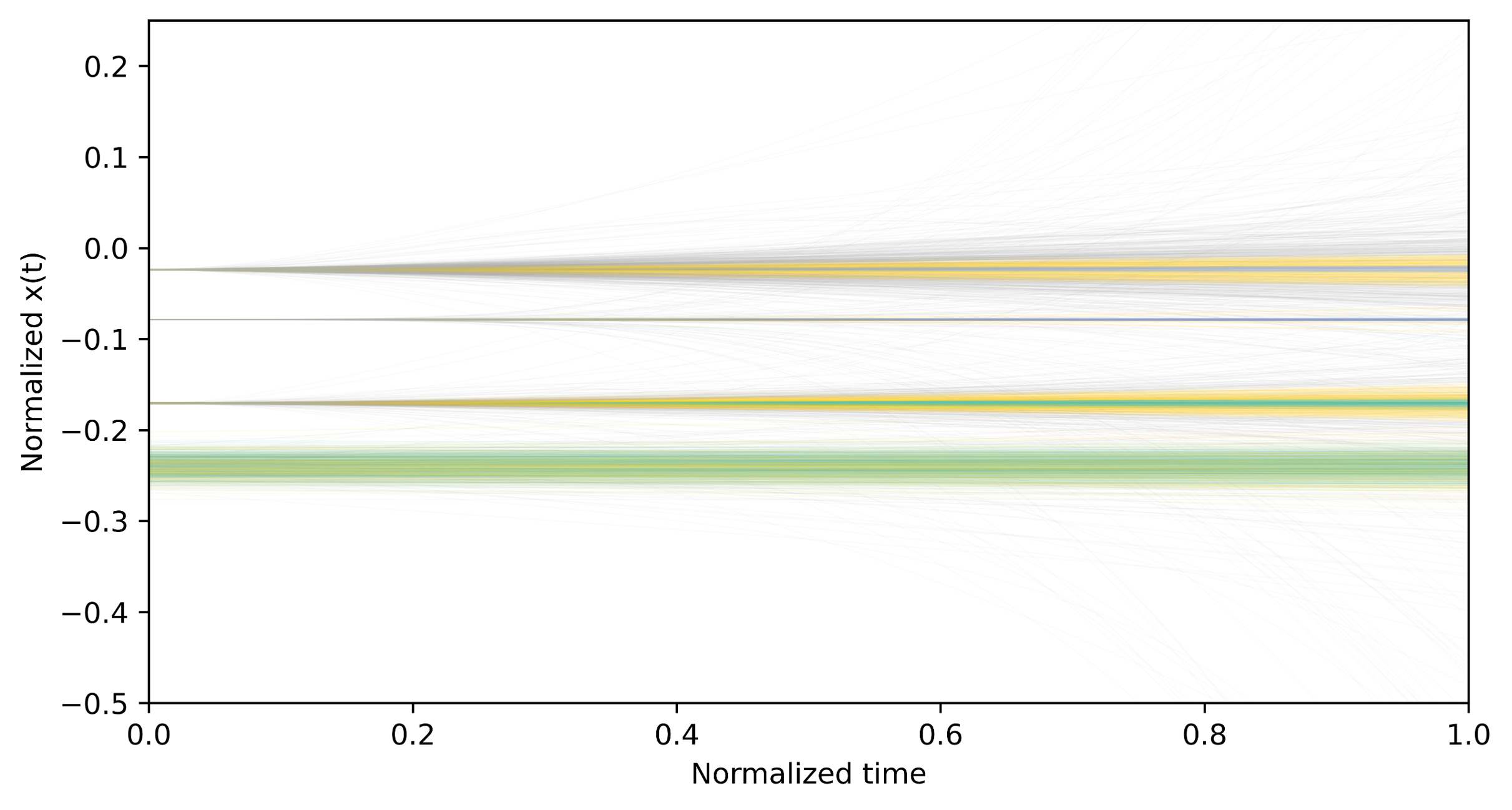}
        \caption{SNO (ACC = 61.3\%)}
        \label{fig:fdc_scp_4}
    \end{subfigure}

    \vspace{0.5em} 

    \begin{subfigure}[b]{0.495\linewidth}
        \centering
        \includegraphics[width=\linewidth]{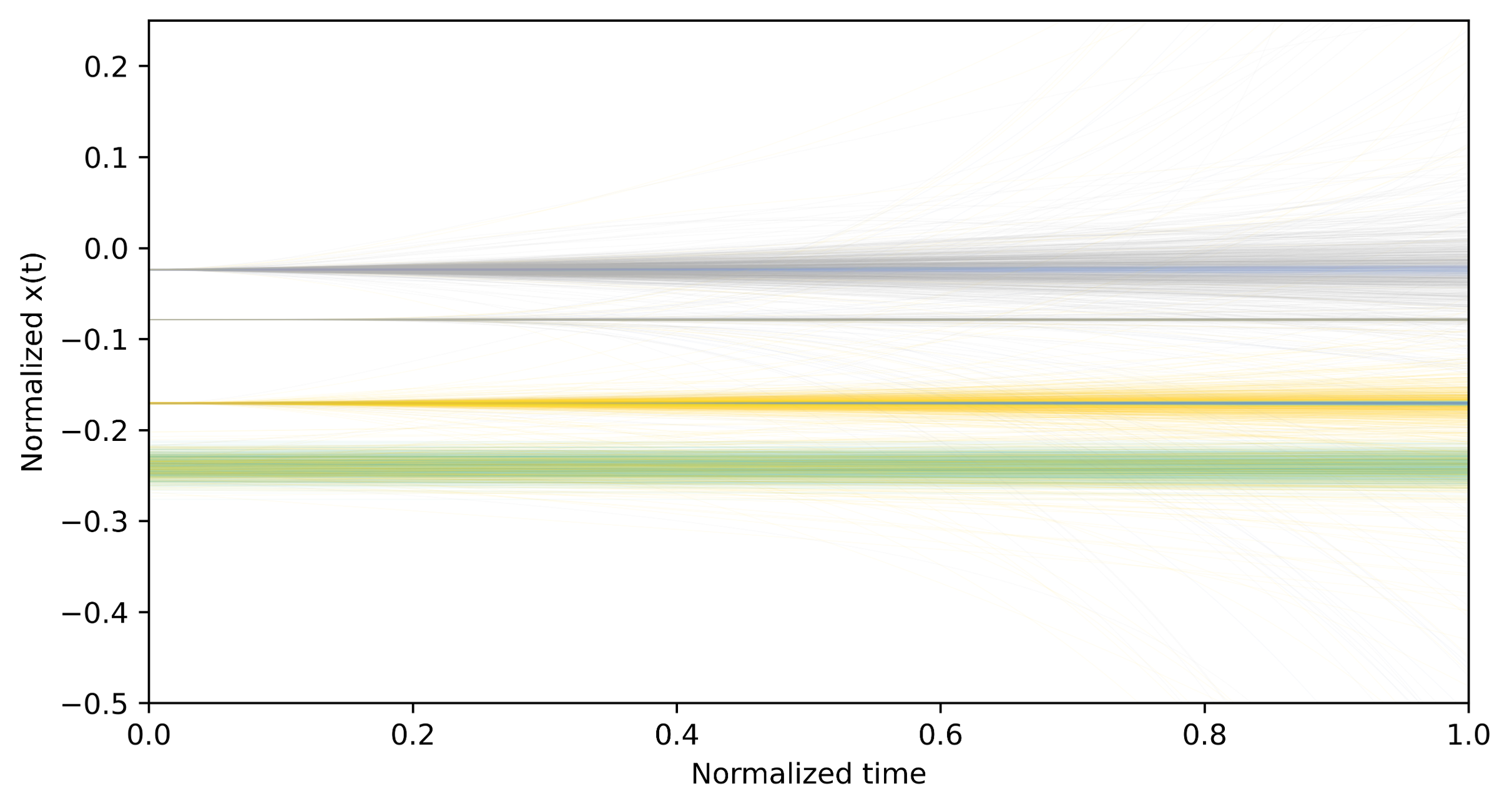}
        \caption{SNO + spectrogram (ACC = 65.2\%)}
        \label{fig:fdc_scp_spec_4}
    \end{subfigure}
    \hfill
    \begin{subfigure}[b]{0.495\linewidth}
        \centering
        \includegraphics[width=\linewidth]{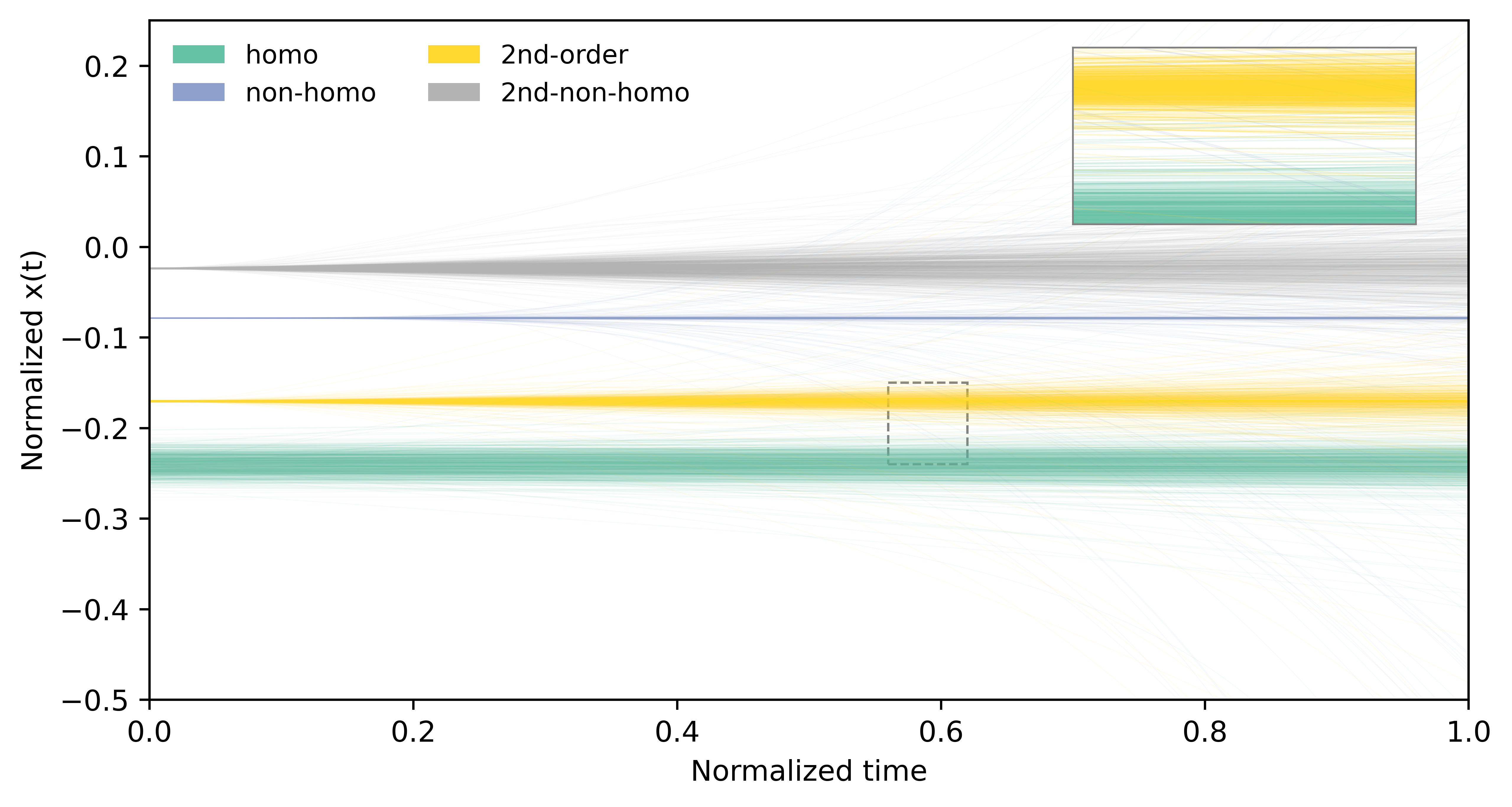}
        \caption{Ground-truth clusters}
        \label{fig:fdc_true_4}
    \end{subfigure}

    \caption{Qualitative comparison on the high-variability ODE-4 test set. 
    (a) K-means (ACC=49.1\%) struggles to handle the intra-class variability introduced by randomized neural vector fields, resulting in fragmented clusters. 
    (b-c) In contrast, the SNO (ACC=61.3\%) and SNO+Spectrogram (ACC=65.2\%) recovers coherent dynamical groupings even in this stochastic regime. 
    (d) Ground-truth clusters (magnified view) highlight the high-frequency fluctuations and complex overlaps that characterize these randomized neural ODEs.}  
    \label{fig:fdc_overlay_ode4}
\end{figure}

\subsection{Ablation Studies: Verifying Theoretical Conditions}

We finally validate the necessity of each term in our objective function $L_{\mathrm{clu}}$. 
Importantly, these ablation studies serve as an empirical verification of the theoretical conditions required for Universal Clustering (Theorem~\ref{thrm:MainKuratowski}).

\begin{table}[!htbp]
\centering
\footnotesize
\setlength{\tabcolsep}{4pt}
\begin{tabular}{c c c ccc ccc}
\toprule
\multirow{2}{*}{$L_{\text{e}}$} & \multirow{2}{*}{$L_{\text{con}}$} & \multirow{2}{*}{$H(Y)$} 
& \multicolumn{3}{c}{ODE-6} & \multicolumn{3}{c}{ODE-4} \\ 
\cmidrule(r){4-6} \cmidrule(r){7-9}
 &  &  & ACC & ARI & NMI & ACC & ARI & NMI \\
\midrule
\ding{51} &       &       & 0.333 & 0.248 & 0.427 & - & - & - \\ 
\ding{51} & \ding{51} &       & 0.333 & 0.262 & 0.435 & - & - & - \\ 
\ding{51} & \ding{51} & \ding{51} & \textbf{0.935} & \textbf{0.871} & \textbf{0.913} & \textbf{0.636} & \textbf{0.366} & \textbf{0.400} \\ 
         & \ding{51} &       & 0.333 & 0.143 & 0.402 & - & - & - \\ 
         & \ding{51} & \ding{51} & \underline{0.883} & \underline{0.802} & \underline{0.846} & \underline{0.618} & \underline{0.353} & \underline{0.370} \\ 
         &           & \ding{51} & - & - & - & 0.579 & 0.292 & 0.340 \\ 
\ding{51} &           & \ding{51} & 0.846 & 0.719 & 0.783 & 0.613 & 0.332 & 0.365 \\ 
\bottomrule
\end{tabular}
\caption{Ablation study of loss components on the training sets of ODE-6 and ODE-4. A ``--'' indicates model collapse, where all samples are assigned to a single cluster. The combination of all three terms is required to satisfy the consistency, sharpness, and non-degeneracy conditions of the theoretical framework. Reported metrics reflect the best performance observed during optimization.}
\label{tab:ode_loss_ablation}
\end{table}

The results in Table~\ref{tab:ode_loss_ablation} reveal:
\begin{itemize}
    \item \emph{Role of entropy ($H(Y)$) $\rightarrow$ avoiding degenerate collapse:}
Without $H(Y)$, the model collapses to assigning all samples to a single cluster. This confirms that entropy regularization is essential to prevent trivial partitions and to maintain a non-degenerate clustering in practice.
    
    \item \emph{Role of Confidence ($L_{\text{con}}$) $\rightarrow$ Lemma~\ref{lem:classification_open__generalapproximation___hardclassificatoin} (Sharpness):}
    Including $L_{\text{con}}$ improves the separation quality (ARI). This aligns with Lemma~\ref{lem:classification_open__generalapproximation___hardclassificatoin},  which requires the soft-assignment map to converge toward a sharp indicator function in the Sierpiński space to define valid clusters.
    
    \item \emph{Synergy:} 
    The best performance is achieved only when Consistency ($L_e$), Confidence ($L_{\text{con}}$), and Balance ($H(Y)$) are jointly optimized. This validates that approximating the true clustering partition requires satisfying topological stability, indicator convergence, and non-degeneracy simultaneously.
\end{itemize}

\paragraph{Sensitivity Analysis}
As shown in Fig.~\ref{alpha}, our method is robust to the choice of the regularization weight $\alpha$. For the structured ODE-6, a moderate $\alpha \approx 2$ is optimal, while for the high-variability ODE-4, performance is stable across a wide range, suggesting that the entropy constraint is a fundamental geometric requirement rather than a sensitive hyperparameter. Although a slightly larger $\alpha$ can be optimal for ODE-6, we fix $\alpha=1$ across datasets to ensure consistency and avoid dataset-specific tuning.

\begin{figure}[!htbp]
    \centering
    \includegraphics[width=0.6\linewidth]{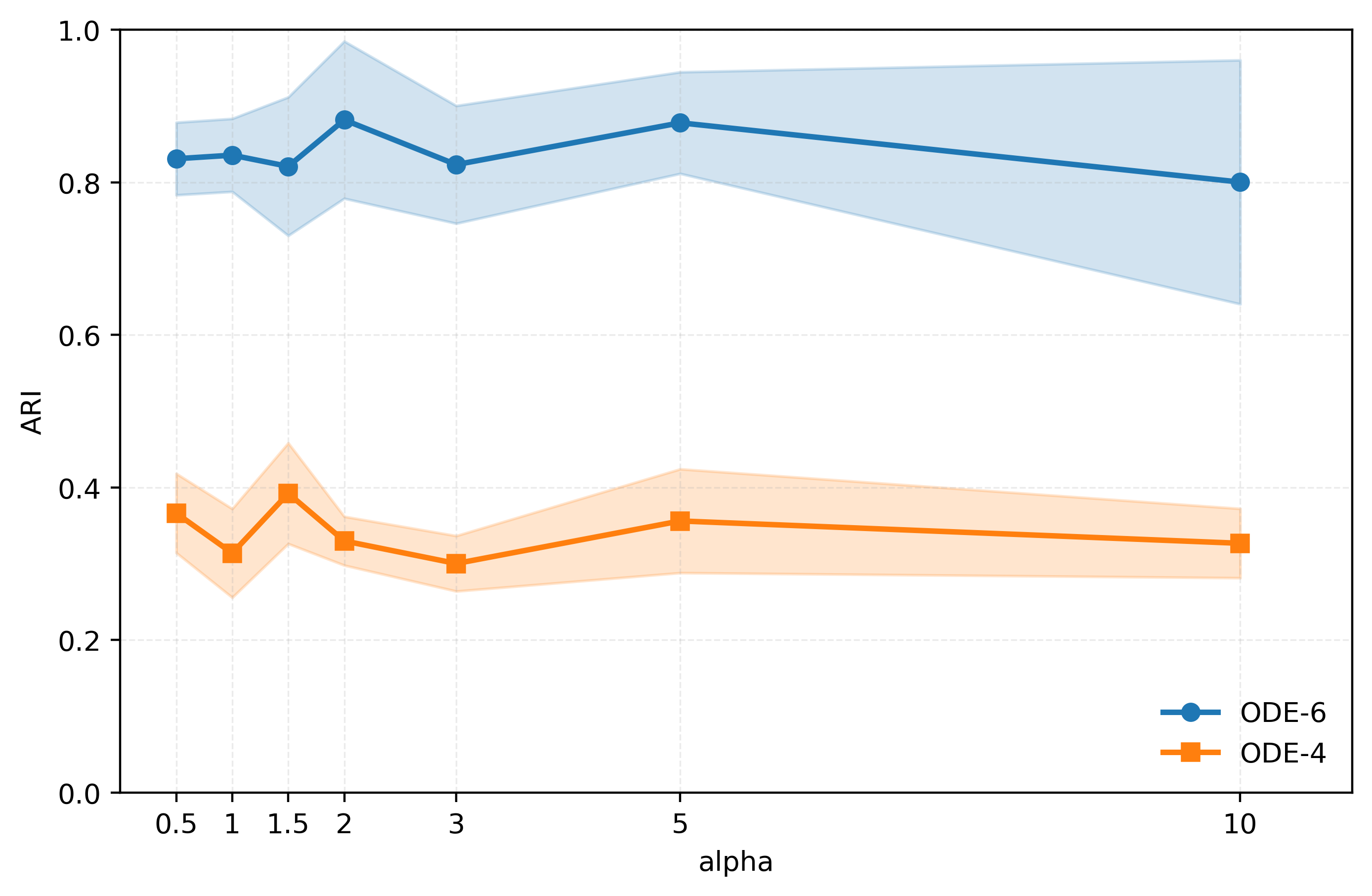}
    \caption{Comparison of different loss weights \( \alpha \). The solid line shows the mean ARI across five runs, and the shaded region indicates the standard deviation.}
    \label{alpha}
\end{figure}

%% file: Conclusion.tex
\section{Related Work}
\label{s:Intro__ss:RelatedWork}
We discuss the three broad research directions related to our results.

\subsection{Functional Data Clustering}
Functional Data Clustering (FDC) \cite{jacques2014fdcsurvey, zhang2022reviewclusteringmethodsfunctional} provides a practical framework for grouping curves or shapes exhibiting similar temporal behavior, and has been widely used for trajectories arising in physical, biological, and chemical dynamical systems \cite{ramsay2005fda, wang2016functional}. Classical FDC pipelines typically combine (i) a finite-dimensional representation---often via handcrafted bases such as B-splines or wavelets \cite{abraham2003unsupervised, giacofci2013wavelet}---with (ii) a standard clustering routine (e.g., $K$-means) applied in the resulting coefficient space. Many approaches also incorporate functional registration to align curves \cite{marron2015functional, lu2017bayesian, 9626620}, which can be computationally costly and sensitive to modeling choices. While effective in practice, these methods ultimately cluster coefficient vectors rather than directly reasoning about cluster geometry in the underlying function space; consequently, they offer limited control over (or interpretability of) the induced cluster regions at the functional level.

\subsection{Deep Clustering}
Recent deep clustering methods address scalability and representation learning by leveraging strong pre-trained encoders, including CLIP \cite{radford2021learning} and DINO \cite{caron2021dino}, for time-series representation \cite{wimmer2023leveragingvisionlanguagemodelsgranular, chen2025visiontsvisualmaskedautoencoders}, imaging-based analysis \cite{caron2018deepcluster, 10.1145/3689036, chu2023image}, and latent embedding pipelines \cite{singh2025shape}. Despite their empirical success, these approaches generally treat trajectories as generic signals to be embedded, and typically do not explicitly incorporate the latent geometric or dynamical structure that can govern similarity between paths.

\subsection{Operator Learning - Regression Theory}
A related literature studies neural operators for learning maps between function spaces, often with the goal of approximating solution operators of differential equations \cite{lu2021deeponet, kovachki2021universal, li2020fourier, benitez2024out}. Existing theory largely focuses on pointwise accuracy or operator-norm approximation guarantees \cite{kovachki2023neuraloperator, kratsios2024mixture}. These guarantees, however, do not directly address clustering, where the primary object is the induced partition of the function space (i.e., the geometry and stability of decision regions), rather than the approximation of function values. In particular, strong operator-norm control alone does not preclude substantial distortions of latent decision boundaries, and therefore does not by itself ensure convergence of the learned cluster regions to the true cluster sets.

\section{Conclusion}
In this work, we studied \emph{universal clustering} for functional data in infinite-dimensional Hilbert spaces. Existing clustering pipelines are predominantly designed for finite-dimensional representations and typically provide limited theoretical guarantees on the recovery of cluster regions in function space, which restricts applicability to infinite-dimensional functional settings. To fill this gap, we established a theoretical foundation for clustering with neural operators. We proved that sufficiently expressive sampling-based neural operators induce cluster regions that converge to the ideal partition in the upper Kuratowski sense, providing a shape-aware and conservative notion of convergence suited to infinite-dimensional settings. We further demonstrated the practical relevance of the framework through an application to clustering ODE-generated trajectories. Using a lightweight operator-learning pipeline on discretized ODE trajectories, our experiments exhibited behavior consistent with our theory, showing improved cluster separability and robustness. These results effectively bridge abstract convergence theory with practical dynamical system clustering, establishing that neural operators can discover functional clusters.

%% file: Proofs.tex

\section{Proofs}
\label{s:Proofs}

\subsection{Proof of Small Auxiliary Results}
\label{s:Proofs__ss:Auxiliary}
\begin{proof}[{Proof of Proposition~\ref{prop:classifier_representation__prelimVersion}}]
If $K=1$ there is nothing to show, thus assume that $K>1$.  For every $k,i=1,\dots,K$ consider the map $\rho_{k:i}:\mathcal{K}\to \|f-f_i\|-\|f-f_k\|$.  Since $\mathcal{H}$ is a topological vector space then addition and scalar multiplication by $-1$ are both continuous and the norm $\|\cdot\|$ is continuous since $\mathcal{H}$ is a Banach space; whence each $\rho_{k:i}$ is continuous.  Consequently, the sets $C_{k:i}=\rho^{-1}([0,\infty))$ are closed. Since finite intersection of closed sets is again closed then $C_k=\bigcap_{i=1}^K\,C_{k:i}$ is closed.
\end{proof}

This section is devoted to the proof of Theorem~\ref{thrm:MainKuratowski}.
\subsection{Proof of Theorem~\ref{thrm:MainKuratowski}}
\label{s:Theory__ss:Proof}
The proof of Theorem~\ref{thrm:MainKuratowski} will be undergone in three major steps.
Before beginning the proof, it will be convenient to introduce another peculiar unique $2$-point topological space: 
the Sierpiński space $\mathbb{S}$ whose point-set is $\{0,1\}$ and whose topology is induced by the unique \textit{linear} order $\lesssim$ on $\{0,1\}$ with $0$ as a minimal element.  I.e.\ its open sets are $\{0\}$, as well as the usual sets $\{0,1\}$ and $\emptyset$; however, unlike the relative topology induced by restriction of the norm topology, $\{1\}$ is \textit{not open} in $\mathbb{S}$.  

However, in the remainder of this paper, it is more convenient to use the ``mirror'' topology of $\mathbb{S}$ (which we still denote as $\mathbb{S}$ for simplicity) which is obtained by swapping the role of $0$ and $1$, i.e. its open sets are $\{ \emptyset, \{1\},\{0,1\} \}$:
\[
    \mathbb{S} = \Big(
    \{0,1\}
    ,
    \big\{ \emptyset, \{1\},\{0,1\} \big\}
    \Big)
\]

As shown on~\citep[Chapter 7]{taylor2011foundations}, the indicator function of any open set in $\mathcal{K}$ defines a continuous map from $\mathcal{K}$ to $\mathbb{S}$.  As discussed, in~\citep[Section 3.2]{kratsios2020non} given any open subset $U\subseteq \mathcal{K}$ and any sequence $(U_n)_{n=1}^{\infty}$ of open subsets of $\mathcal{K}$ we have
$\lim\limits_{n\uparrow \infty}\, U_n=U$ (in the Scott topology) if and only if the indicator functions $I_{U_k}$ converge to the indicator function of $I_U$ in the compact-open topology (which is the natural generalization of the topology of uniform convergence on compact sets to the non-metrizable setting such as $C(\mathcal{K},\mathbb{S})$; see~\cite[page 285]{munkres} for a definition).

\subsubsection{Step 1 - Clusters Representation}
\label{s:Theory__ss:Proof___sss:ClusteRep}
Our first lemma shows that every (non-empty) finite collection of open sets can be parameterized by a continuous map into a hypercube of dimension equal to the number of sets.

\begin{proposition}[Representation of Pure Clusters]
	\label{prop:classifier_representation}
	Let $\mathcal{K}$ be a non-empty subset of a Hilbert space $\mathcal{H}$,
	$K\in \mathbb{N}_+$, and $f_1,\dots,f_K\in \mathcal{K}$ be distinct.  For every $k=1,\dots,K$ define the class 
    $
		\mathcal{C}_k\eqdef \mathcal{K}\setminus C_k
    $
	where $C_k$ is defined as in~\eqref{eq:the_clusters__boundaried}.
	Then, for every $\gamma \in (0,1)$ there exists a map $c:\mathcal{K}\to [0,1]^K$ which is continuous in the compact-open topology on $C(\mathcal{K},[0,1]^K)$ such that
	\begin{equation}
		\label{eq:clusteringmap}
		\Big(
		I_{(\gamma,1]}
		\circ
		c_k
		\Big)
		=
		I_{C_k}
		\,\,
		\mbox{ for every } k=1,\dots,K.
	\end{equation}
\end{proposition}

We will prove Proposition~\ref{prop:classifier_representation} in two steps.

\begin{lemma}
	\label{lem:classification_open__existence}
	Let $\mathcal{K}$ be a non-empty subset of a Hilbert space $\mathcal{H}$,
	let $K\in \mathbb{N}_+$, and let $\mathcal{C}_1,\dots,\mathcal{C}_K\subseteq \mathcal{K}$ be non-empty and open.
	Then, for every $\gamma \in (0,1)$ there exists a continuous map $c:\mathcal{K}\to [0,1]^K$ in the compact-open topology on $C(\mathcal{K},[0,1]^K)$ such that
	\begin{equation}
		\label{eq:classifiers_open_sets}
		\Big(
		I_{(\gamma,1]}
		\circ
		c_k
		\Big)
		=
		I_{\mathcal{C}_k}
		\mbox{ for every } k=1,\dots,K.
	\end{equation}
\end{lemma}

\begin{proof}[{Proof of Lemma~\ref{lem:classification_open__existence}}]
	Since $\mathcal{H}$ is a Hilbert space, it is metrizable; hence the relative topology on $\mathcal{K}$ is also metrizable.
	Fix $\gamma\in (0,1)$ and for each $k=1,\dots,K$ define the map $c_k:\mathcal{K}\to [0,1]$ for each $f\in \mathcal{K}$ by
	\begin{equation}
		\label{eq:ck_distance_def}
		c_k(f)
		\eqdef
		\gamma
		+
		(1-\gamma)
		\,
		\frac{d_{\mathcal{K}}\big(f,\mathcal{K}\setminus \mathcal{C}_k\big)}
		{1+d_{\mathcal{K}}\big(f,\mathcal{K}\setminus \mathcal{C}_k\big)}
	,
	\end{equation}
	where $d_{\mathcal{K}}$ denotes the metric induced on $\mathcal{K}$ by the norm of $\mathcal{H}$.
	Each map $f\mapsto d_{\mathcal{K}}(f,\mathcal{K}\setminus \mathcal{C}_k)$ is continuous, hence so is each $c_k$.
	Since $\mathcal{C}_k$ is open, we have that $f\in \mathcal{C}_k$ if and only if $d_{\mathcal{K}}\big(f,\mathcal{K}\setminus \mathcal{C}_k\big)>0$ which, in turn, holds if and only if $c_k(f)>\gamma$; establishing~\eqref{eq:classifiers_open_sets}.
	We conclude by defining $c:\mathcal{K}\to [0,1]^K$ by $c(f)\eqdef(c_k(f))_{k=1}^K$.
\end{proof}

Next, we show that the $K$-means problem always induces $K$ distinct open sets (clusters) whenever the problem is non-degenerate; i.e.\ whenever there are distinct means.

\begin{lemma}
	\label{lem:Kclusters_are_open}
	Let $\mathcal{K}$ be a non-empty subset of a Hilbert space $\mathcal{H}$,
	$K\in \mathbb{N}_+$,
	and $f_1,\dots,f_K\in \mathcal{K}$ be distinct.
	For every $k=1,\dots,K$ the set $C_k$, defined in~\eqref{eq:the_clusters__boundaried},
	is closed and non-empty.
\end{lemma}

\begin{proof}[{Proof of Lemma~\ref{lem:Kclusters_are_open}}]
	If $K=1$ there is nothing to show; suppose therefore that $K>1$.
	By the triangle inequality, the map $f\mapsto \|f\|$ is $1$-Lipschitz.  Furthermore, the $(K-1)$-fold minimum map $[0,\infty)^{K-1}\ni x\mapsto \min_{i=1,\dots,K-1} x_i$ is $1$-Lipschitz since $\min_i x_i\le \|x\|_{\infty}\le \|x\|_2$.
	Consequently, for every $k=1,\dots,K$ the map $\Psi_k:\mathcal{H}\to \mathbb{R}$ sending any $f\in \mathcal{H}$ to
	\begin{equation}
		\label{eq:preimage}
		\Psi_k(f)
		\eqdef
		\min_{i=1,\dots,K;\,i\neq k}\, \|f-f_i\|
		-
		\|f-f_k\|
	\end{equation}
	is continuous.
	By definition, cf.~\eqref{eq:the_clusters__boundaried}, we have
	\begin{equation}
		\label{eq:Ck_preimage}
		C_k
		=
		\big(\Psi_k^{-1}([0,\infty))\big)\cap \mathcal{K}.
	\end{equation}
	Since $[0,\infty)$ is closed and $\Psi_k$ is continuous, $\Psi_k^{-1}([0,\infty))$ is closed in $\mathcal{H}$, hence $C_k$ is closed in the relative topology on $\mathcal{K}$.
	Finally, $f_k\in C_k$ by definition, so $C_k$ is non-empty.
\end{proof}

Upon combining Lemmata~\ref{lem:Kclusters_are_open} and~\ref{lem:classification_open__existence} we directly obtain Proposition~\ref{prop:classifier_representation}.

\subsubsection{Step 2 - Universal Classification}
\label{s:Theory__ss:Proof___sss:UnivCluster}
Next, we show that the map $c$ in Proposition~\ref{prop:classifier_representation}, for any arbitrary threshold $\gamma\in (0,1)$, can be approximated.  We first derive a general approximation lemma valid for very broad classes of models which we then specialize to our neural operators.

We again argue in two steps; with the first step showing the universality of a general model class for the soft classification (taking values in $[0,1]^K$ not necessarily in $\{0,1\}^K$).

\begin{lemma}
	\label{lem:classification_open__generalapproximation}
	Let $\mathcal{K}$ be a non-empty, locally-compact, subset of a Hilbert space $\mathcal{H}$, and let $K\in \mathbb{N}_+$.
	If $\mathcal{F}\subseteq C(\mathcal{K},\mathbb{R}^K)$ is dense in the topology of uniform convergence on compact subsets and $\rho:\mathbb{R}^K\to (0,1)^K$ is a homeomorphism, then the set
	\begin{equation}
		\label{eq:general_classifier}
		\mathcal{F}_{\text{soft-class}}
		\eqdef
		\Big\{
		\rho\circ \hat{f}:\,\hat{f}\in \mathcal{F}
		\Big\}
	\end{equation}
	is dense in the compact-open topology of $C(\mathcal{K},[0,1]^K)$.
\end{lemma}

\begin{proof}[{Proof of Lemma~\ref{lem:classification_open__generalapproximation}}]
	Since $\rho$ is a homeomorphism onto $(0,1)^K$, post-composition by $\rho$ defines a homeomorphism
	\begin{equation}
		\label{eq:postcomp_homeo}
		C(\mathcal{K},\mathbb{R}^K)\ni f \mapsto \rho\circ f \in C(\mathcal{K},(0,1)^K)
	\end{equation}
	for the compact-open topology. Hence the density of $\mathcal{F}$ in $C(\mathcal{K},\mathbb{R}^K)$ implies the density of $\mathcal{F}_{\text{soft-class}}$ in $C(\mathcal{K},(0,1)^K)$ for the same topology.
	
	Finally, since $(0,1)^K$ is dense in $[0,1]^K$ (in the relative topology), standard density-transfer arguments (e.g.~\citep[Lemma B.4]{kratsios2020non} when $\mathcal{K}$ is locally-compact, or the direct metric argument on compacta) yield that $C(\mathcal{K},(0,1)^K)$ is dense in $C(\mathcal{K},[0,1]^K)$ in the compact-open topology. The conclusion follows by transitivity of density.
\end{proof}

Next, we upgrade our universal soft-classification result to a hard/strict classification result.

\begin{lemma}
	\label{lem:classification_open__generalapproximation___hardclassificatoin}
	Let $\mathcal{K}$ be a non-empty, locally-compact, subset of a Hilbert space $\mathcal{H}$,
	$K\in \mathbb{N}_+$, and $f_1,\dots,f_K\in \mathcal{K}$ be distinct.
	If $\mathcal{F}\subseteq C(\mathcal{K},\mathbb{R}^K)$ is dense in the topology of uniform convergence on compact subsets and $\rho:\mathbb{R}^K\to (0,1)^K$ is a homeomorphism, then for any $\gamma \in (0,1)$ there exists a sequence $(\hat{f}_n)_{n=1}^{\infty}$ in $\mathcal{F}$ such that
	\begin{equation}
		\label{eq:setvaluedapproximation}
		\lim\limits_{n\uparrow \infty}
		\,
		I_{(\gamma,1]}
		\bullet
		\big(
		\rho \circ \hat{f}_n
		\big)
		=
		I_{\prod_{k=1}^K\, C_k}
	\end{equation}
	where the limit in~\eqref{eq:setvaluedapproximation} is taken in the compact-open topology of $C(\mathcal{K},\{0,1\}^K)$ where $\{0,1\}^K$ is equipped with the $K$-fold product of the Sierpiński topology on $\{0,1\}$ and $\bullet$ denotes componentwise composition.
\end{lemma}

\begin{proof}[{Proof of Lemma~\ref{lem:classification_open__generalapproximation___hardclassificatoin}}]
	If $K=1$ there is nothing to show; thus assume that $K>1$.
	Let $c:\mathcal{K}\to [0,1]^K$ be the continuous map from Proposition~\ref{prop:classifier_representation} satisfying~\eqref{eq:clusteringmap}.
	
	Equip $\{0,1\}$ with the Sierpiński topology. Since $(\gamma,1]$ is open in $[0,1]$ (relative topology), the indicator map $I_{(\gamma,1]}:[0,1]\to \{0,1\}$ is continuous, hence its $K$-fold product $\prod_{k=1}^K\, I_{(\gamma,1]}:[0,1]^K\to \{0,1\}^K$ is continuous for the product topology.
	
	By Lemma~\ref{lem:classification_open__generalapproximation}, we may choose $\hat{f}_n\in \mathcal{F}$ such that $\rho\circ \hat{f}_n \to c$ in the compact-open topology of $C(\mathcal{K},[0,1]^K)$. Continuity of post-composition in the compact-open topology then yields
	\begin{equation}
		\label{eq:hard_pass_to_limit}
		I_{(\gamma,1]} \bullet (\rho\circ \hat{f}_n)
		\to
		I_{(\gamma,1]} \bullet c
	\end{equation}
	in the compact-open topology of $C(\mathcal{K},\{0,1\}^K)$.
	Finally, by~\eqref{eq:clusteringmap} we have $I_{(\gamma,1]} \bullet c = 
    (I_{\mathcal{C}_k})_{k=1}^K
    $, which is exactly~\eqref{eq:setvaluedapproximation}.
\end{proof}

It only remains to specialize $\mathcal{F}$ to our computationally amenable class of neural operators designed to align well with sampling-based computational pipelines. Thus, it remains to establish the universal approximation property of our class of sampling-based neural operators.

\subsubsection{Step 3 - Universality of Sampling-Based Neural Operator}
\label{s:Theory__ss:Proof___sss:Universal_SBNO}
Our strategy will be to deploy~\citep[Theorem 1]{CNO}, as defined in~\citep[Equation (6)]{CNO}.
To this end, we may take the maps $\varphi_{n^{out}}$ to be the identity map on $\mathbb{R}^K$ and $\psi_{n^{in}}$ to be the identity map on $\mathbb{R}^m$, where $m=n^{in}$. Additionally, we may take $B=\mathbb{R}^K$ and $I_{B:n^{out}}$ to be the identity map also.
Thus, we only need to show that there are points $\{x_n\}_{n=1}^{\infty}$ in our RKHS $\mathcal{H}$ over $\mathcal{X}$ such that for every $m\in \mathbb{N}_+$
the map
\begin{equation}
	\label{eq:project}
	P_{\mathcal{H}:m}(f)
	\eqdef
	\Big(
	\langle f,\kappa(\cdot,x_n)\rangle_{\mathcal{H}}
	\Big)_{n=1}^m
\end{equation}
is as in~\citep[Equation (16)]{CNO}; that is, we only need to show that $\big(\kappa(\cdot,x_n)\big)_{n=1}^{\infty}$ defines a Schauder basis of $\mathcal{H}$.

\begin{proposition}
	\label{prop:universality}
	Fix $K\in \mathbb{N}_+$,
	let $\mathcal{H}$ be an 
    separable RKHS over a non-empty set $\mathcal{X}$ with kernel $\kappa:\mathcal{X}^2\to \mathbb{R}$ for which there exists a complete interpolating sampling set $\{x_{\lambda}\}_{\lambda\in \Lambda} \subseteq \mathcal{X}$, and let $\mathcal{K}\subseteq \mathcal{H}$ be non-empty and locally-compact.
	Then, the set $\mathcal{F}_{SNO}$ is dense in $C(\mathcal{K},\mathbb{R}^K)$ for the compact-open topology.
\end{proposition}

\begin{proof}[{Proof of Proposition~\ref{prop:universality}}]
	If $\mathcal{H}$ is finite-dimensional there is nothing to show; thus we assume $\mathcal{H}$ is infinite-dimensional.
	Since $\mathcal{H}$ admits a (countable) complete interpolating sampling sequence $\{x_{\lambda}\}_{\lambda\in \Lambda}$, then as discussed in~\citep[Lemma 2.3]{tsikalas2023interpolating}
	the set $\{\kappa(\cdot,x_{\lambda})\}_{\lambda\in \Lambda}\subset \mathcal{H}$ is a Riesz basis of $\mathcal{H}$.
	
	Hence there exists a bounded bijective linear operator $T:\mathcal{H}\to \mathcal{H}$ and an orthonormal basis $(e_{\lambda})_{\lambda\in\Lambda}$ of $\mathcal{H}$ such that
	\begin{equation}
		\label{eq:unitary_def}
		\kappa(\cdot,x_{\lambda}) = T e_{\lambda}
		\qquad\mbox{for every } \lambda\in \Lambda.
	\end{equation}
	Let $(\psi_{\lambda})_{\lambda\in\Lambda}$ be the bi-orthogonal family defined by
	\begin{equation}
		\label{eq:dual_basis_def}
		\psi_{\lambda}
		\eqdef
		(T^{-1})^{\ast} e_{\lambda},
	\end{equation}
	so that $\langle \kappa(\cdot,x_{\lambda}),\psi_{\mu}\rangle_{\mathcal{H}}=\delta_{\lambda\mu}$.
	Then every $f\in\mathcal{H}$ has the (unique) expansion
	\begin{equation}
		\label{eq:riesz_expansion}
		f
		=
		\sum_{\lambda\in\Lambda}
		\langle f,\psi_{\lambda}\rangle_{\mathcal{H}}
		\,
		\kappa(\cdot,x_{\lambda}),
	\end{equation}
	with convergence in $\mathcal{H}$ (equivalently, the partial sums converge in norm). In particular, $\{\kappa(\cdot,x_{\lambda})\}_{\lambda\in\Lambda}$ is a Schauder basis of $\mathcal{H}$.
	
	We have thus verified the projection/basis hypothesis needed to apply~\citep[Theorem 1]{CNO} with the coordinate maps given by~\eqref{eq:project}. Therefore, for every (possibly non-linear) Lipschitz map $g:\mathcal{K}\to \mathbb{R}^K$, every compact $K_0\subseteq \mathcal{K}$, and every $\varepsilon>0$, there exists some $f\in \mathcal{F}_{SNO}$ satisfying
	\begin{equation}
		\label{eq:univ_on_compacts}
		\sup_{x\in K_0}\, \|f(x)-g(x)\|<\varepsilon.
	\end{equation}
	Finally, since Lipschitz maps $\mathcal{K}\to\mathbb{R}^K$ are dense in $C(\mathcal{K},\mathbb{R}^K)$ for uniform convergence on compact sets (hence for the compact-open topology), we conclude that $\mathcal{F}_{SNO}$ is dense in $C(\mathcal{K},\mathbb{R}^K)$ for the compact-open topology.
\end{proof}

We thus arrive at the following intermediate result formulated in the language of the compact open topology on $C(\mathcal{K},\{0,1\}^K)$ where $\{0,1\}$ is equipped with the $K$-fold product of $\{0,1\}$ equipped with the Sierpiński topology.

\begin{proposition}[Universal Clustering]
	\label{prop:Main}
	Let $K\in \mathbb{N}_+$, and let $\mathcal{K}$ be a non-empty, locally-compact, subset of an RKHS $\mathcal{H}$ over a set $X$ which is CIS.
	Fix an increasing surjective ``logit feature'' map $\sigma\in C(\mathbb{R},(0,1))$, and a cluster threshold $0<\gamma<1$.
	
	For any \textit{distinct} (true) cluster centers $f_1,\dots,f_K\in \mathcal{K}$, define $\mathcal{C}_k\eqdef \mathcal{K}\setminus C_k$; where $C_k$ are defined as in~\eqref{eq:the_clusters__boundaried}.
	Then there exists a sequence of sampling-based neural operators $\{\hat{f}_n\}_{n=1}^{\infty}\subseteq \mathcal{F}_{SNO}$ such that
	\begin{equation}
	\label{eq:thmMain_limit}
		\lim_{n\uparrow \infty}
		\,
		\Big(
		I_{(\gamma,1]}\circ \sigma
		\Big)
		\bullet
		\hat{f}_n
		=
		I_{\prod_{k=1}^K\, \mathcal{C}_k}
	\end{equation}
	where the limit in~\eqref{eq:thmMain_limit} is taken in the compact-open topology of $C(\mathcal{K},\{0,1\}^K)$, where $\{0,1\}^K$ is equipped with the $K$-fold product of the Sierpiński topology on $\{0,1\}$ and $\bullet$ denotes componentwise composition.
\end{proposition}
\begin{proof}[{Proof of Proposition~\ref{prop:universality}}]
Direct consequence of Proposition~\ref{prop:universality}, and Lemma~\ref{lem:classification_open__generalapproximation___hardclassificatoin}.
\end{proof}
In order to deduce Theorem~\ref{thrm:MainKuratowski}, it remains to relate the convergence of the indicator functions of the complements of the relevant sets to the original classes $C_1,\dots,C_K$ as defined in~\eqref{eq:thmMain_limit} in a sensible topology on closed sets; namely the upper Kuratowski topology.  This is the last and final step of our proof.

\subsection{Step 4 - Conversion of Result Back to Original Classes $C_k$}
The final step will rely on the following general result.
Before formulating the result, we recall the usual definition of the upper Kuratowski convergence: for any $x\in \mathcal{H}$ we write $U_x$ for an arbitrary open neighborhood thereof.  
The Kuratowski upper limit of a sequence $(C_n)_{n=1}^{\infty}$ of subsets of $\mathcal{H}$ is the subset of $\mathcal{H}$ defined by
\[
    \lim\limits_{n\uparrow \infty}\, C_n
\eqdef 
    \big\{
    x: \forall\, U_x \in \mathcal{N}(x)\,\, \forall\, N\in\mathbb{N}\,\, \exists\, n\ge N
    \ \text{such that}\ 
    C_n \cap U_x \neq \emptyset 
    \big\}
,
\] 
where $\mathcal{N}(x)$ is the set of all neighborhood of $x$. Now, given a subset $C$ of $\mathcal{H}$, we say that the sequence $(C_n)_{n=1}^{\infty}$ upper Kuratowski convergences to $C$ provided that $\lim\limits_{n\uparrow \infty}\, C_n=C$.

\begin{lemma}
\label{lem:setvaluedconvergenceequivalences}
Let $\mathcal{K}\subseteq \mathcal{H}$ be locally-compact.  
Let $(U_n)_{n=1}^{\infty}$ be a sequence of open subsets of $X$ and let $U\subseteq X$ also be open.  Let $F \eqdef X\setminus U$ and for every $n\in \mathbb{N}_+$ define $
F_n \eqdef X\setminus U_n$.  

\noindent The following are equivalent:
\begin{enumerate}
\item[(i)] $ (I_{U_n})_{n=1}^{\infty}$ converges to $I_U$ in the compact-open topology on $C(X,\mathbb S)$.
\item[(ii)] $(F_n)_{n=1}^{\infty}$ converges to $F$ in the upper (outer) Kuratowksi sense
\item[(iii)] For every sequence $(f_n)_{n=1}^{\infty}$ in $\mathcal{K}$ converging to some $f\in \mathcal{K}$ we have
\[
        \delta_{C}(f)
    \le 
        \liminf_{n \to \infty} \delta_{C_n}(f_n)
.
\]
\end{enumerate}
\end{lemma}
\begin{proof}[{Proof of Lemma~\ref{lem:setvaluedconvergenceequivalences}}]
Any map $f:X\to\mathbb S$ is continuous if and only if
$f^{-1}(\{1\})$ is open; cf.~\citep[page 2871]{dolecki1995upper}. 
Since $\mathcal{H}$ is a Hilbert space then $\mathcal{K}$ is a Hausdorff space when equipped with the subspace topology.  Whence $\mathcal{K}$ is a locally-compact Hausdorff space.  Consequently, as $\mathcal{K}$ is locally-compact then, \citep[Theorem 1.1]{dolecki1995upper} implies the conclusion.  

Now, applying~\citep[Proposition 7.4 (f)]{rockafellar1998variational} we have that (ii) is equivalent to the upper epigraphical convergence of the convex-analytic indicator functions $(\delta_{F_n})_{n=1}^{\infty}$ to $\delta_F$; where for any subset $A\subseteq \mathcal{K}$ we have $\delta_A(x)=0$ if $x\not\in A$ and $\infty$ otherwise.  Now since $\mathcal{K}$ is a metric space then we may apply~\citep[Lemma 5.1.13 Equation (5.1.9)]{borwein2005techniques} to obtain the equivalence of (ii) and (iii).
\end{proof}

Now, applying Lemma~\ref{lem:setvaluedconvergenceequivalences} (i) and (iii) to the conclusion of Proposition~\ref{prop:Main} yields Theorem~\ref{thrm:MainKuratowski}.

%% file: Appendix.tex
\section{Algorithm}
We provide the pseudo code for our proposed concrete sampling-based neural operator for ODE clustering:

\begin{algorithm}[ht]
\SetAlgoLined
\LinesNumbered
\caption{Sampling-Based Operator Learning for ODE Clustering}\label{alg:CAC}
\KwIn{$\mathcal{X}=\{x_i\}_{i=1}^{N}$, frozen encoder $\phi(\cdot)$, clusters $K$, batch size $B$, loss weight $\alpha$}
\KwOut{Soft assignments $y_i = \operatorname{softmax}(g(h_i))$}

\textbf{Preprocessing:} Normalize trajectories and render each $x_i$ into image/spectrogram\;
Initialize cluster head $g(\cdot)$\;

\For{each epoch}{
  \For{each mini-batch $\mathbf{X}_B = \{x_i\}_{i=1}^{B}$}{
    \tcp{Batch-wise Pair Construction}
    $\mathbf{X}^a = T^a(\mathbf{X}_B)$, $\mathbf{X}^b = T^b(\mathbf{X}_B)$\;
    
    \tcp{Forward Pass (Frozen Encoder + Trainable Head)}
    $\mathbf{H}^a = \phi(\mathbf{X}^a)$, $\mathbf{H}^b = \phi(\mathbf{X}^b)$\;
    $\mathbf{Z}^a = g(\mathbf{H}^a)$, $\mathbf{Z}^b = g(\mathbf{H}^b)$\;
    $\mathbf{Y}^a = \operatorname{softmax}(\mathbf{Z}^a)$, $\mathbf{Y}^b = \operatorname{softmax}(\mathbf{Z}^b)$\;
    
    \tcp{Objective \& Optimization}
    $L_{\mathrm{clu}} = L_e(\mathbf{Y}^a, \mathbf{Y}^b) + L_{\mathrm{con}}(\mathbf{Y}^a, \mathbf{Y}^b) - \alpha H(\mathbf{Y})$\;
    Update parameters of $g(\cdot)$ by minimizing $L_{\mathrm{clu}}$\;
  }
}
\end{algorithm}